\documentclass[final]{colt2020} 

\usepackage{times}
\usepackage{booktabs}       
\usepackage{amsfonts}       
\usepackage{nicefrac}       
\usepackage{enumitem}

\usepackage{amsmath,amssymb,mathrsfs}

\usepackage{footnote}
\usepackage{algorithm}

\allowdisplaybreaks

\makeatletter
\def\set@curr@file#1{\def\@curr@file{#1}} 
\makeatother

\usepackage{graphicx}
\usepackage{mathtools}
\usepackage{footnote}
\usepackage{float}
\usepackage{xspace}
\usepackage{multirow}
\usepackage{wrapfig}
\usepackage{framed}
\usepackage{bbm}
\usepackage{footnote}
\usepackage{nicefrac}
\usepackage{makecell}
\usepackage[T1]{fontenc}
\makesavenoteenv{tabular}
\makesavenoteenv{table}

\definecolor{Green}{rgb}{0.13, 0.65, 0.3}
\usepackage[]{color-edits}

\addauthor{HL}{red}
\addauthor{hl}{Green}
\addauthor{CL}{orange}
\addauthor{MZ}{cyan}

\newcommand{\BiGraph}{Directed Complete Bipartite Graphs}
\newcommand{\bigraph}{directed complete bipartite graphs\xspace}

\newcommand{\expthreeG}{\textsc{Exp3.G}\xspace}

\newcommand{\calA}{{\mathcal{A}}}

\newcommand{\calD}{{\mathcal{D}}}

\newcommand{\Reg}{\text{\rm Reg}}

\newcommand{\one}{\boldsymbol{1}}

\newcommand{\KL}{\text{\rm KL}}

\newcommand{\isstar}{i_S^\star}

\DeclareMathOperator*{\diag}{diag}
\DeclareMathOperator*{\argmin}{argmin}

\newcommand{\field}[1]{\mathbb{#1}}

\newcommand{\fR}{\field{R}}

\newcommand{\E}{\field{E}}

\newcommand{\inner}[1]{ \left\langle {#1} \right\rangle }

\newcommand{\norm}[1]{\left\|{#1}\right\|}

\newcommand{\wh}{\widehat}

\newcommand{\order}{\ensuremath{\mathcal{O}}}
\newcommand{\otil}{\ensuremath{\tilde{\mathcal{O}}}}

\newcommand{\const}{64}
\newcommand{\constK}{\const K}

\newcommand{\Nin}{N^{\text{in}}}

\newcommand{\wklydn}{d}
\newcommand{\wklyds}{\mathcal{D}}
\newcommand{\brgmd}{D}
\newcommand{\LC}{L_{\wklyds}}
\newcommand{\slpn}{s}
\newcommand{\woslpn}{\bar{s}}
\newcommand{\initp}{$p_1$ is such that $p_{1,i} = \frac{1}{2\slpn}$ for $i \in S$ and $p_{1,i} = \frac{1}{2\woslpn}$ for $i \in \bar{S}$.}
\newcommand{\mtep}{m}

\usepackage{prettyref}
\newcommand{\pref}[1]{\prettyref{#1}}

\newcommand{\savehyperref}[2]{\texorpdfstring{\hyperref[#1]{#2}}{#2}}
\newrefformat{eq}{\savehyperref{#1}{Eq.~\textup{(\ref*{#1})}}}
\newrefformat{eqn}{\savehyperref{#1}{Equation~\ref*{#1}}}
\newrefformat{lem}{\savehyperref{#1}{Lemma~\ref*{#1}}}
\newrefformat{def}{\savehyperref{#1}{Definition~\ref*{#1}}}
\newrefformat{line}{\savehyperref{#1}{Line~\ref*{#1}}}
\newrefformat{thm}{\savehyperref{#1}{Theorem~\ref*{#1}}}
\newrefformat{corr}{\savehyperref{#1}{Corollary~\ref*{#1}}}
\newrefformat{cor}{\savehyperref{#1}{Corollary~\ref*{#1}}}
\newrefformat{sec}{\savehyperref{#1}{Section~\ref*{#1}}}
\newrefformat{app}{\savehyperref{#1}{Appendix~\ref*{#1}}}
\newrefformat{assum}{\savehyperref{#1}{Assumption~\ref*{#1}}}
\newrefformat{ex}{\savehyperref{#1}{Example~\ref*{#1}}}
\newrefformat{fig}{\savehyperref{#1}{Figure~\ref*{#1}}}
\newrefformat{alg}{\savehyperref{#1}{Algorithm~\ref*{#1}}}
\newrefformat{rem}{\savehyperref{#1}{Remark~\ref*{#1}}}
\newrefformat{conj}{\savehyperref{#1}{Conjecture~\ref*{#1}}}
\newrefformat{prop}{\savehyperref{#1}{Proposition~\ref*{#1}}}
\newrefformat{proto}{\savehyperref{#1}{Protocol~\ref*{#1}}}
\newrefformat{prob}{\savehyperref{#1}{Problem~\ref*{#1}}}
\newrefformat{claim}{\savehyperref{#1}{Claim~\ref*{#1}}}
\newrefformat{que}{\savehyperref{#1}{Question~\ref*{#1}}}
\newrefformat{op}{\savehyperref{#1}{Open Problem~\ref*{#1}}}
\newrefformat{fn}{\savehyperref{#1}{Footnote~\ref*{#1}}}
\newrefformat{tab}{\savehyperref{#1}{Table~\ref*{#1}}}
\newrefformat{fig}{\savehyperref{#1}{Figure~\ref*{#1}}}

\DeclareOldFontCommand{\rm}{\normalfont\rmfamily}{\mathrm}
\DeclareOldFontCommand{\it}{\normalfont\itshape}{\mathit}
\title[A Closer Look at Small-loss Bounds for Bandits with Graph Feedback]{A Closer Look at Small-loss Bounds for Bandits with Graph Feedback}

\coltauthor{%
 \Name{Chung-Wei Lee} \Email{leechung@usc.edu}\\
 \Name{Haipeng Luo} \Email{haipengl@usc.edu} \\
 \Name{Mengxiao Zhang}
 \Email{mengxiao.zhang@usc.edu}\\
 \addr University of Southern California
}

\begin{document}
\SetAlgoVlined
\DontPrintSemicolon
\maketitle


\begin{abstract}
We study {\it small-loss} bounds for adversarial multi-armed bandits with graph feedback, that is, adaptive regret bounds that depend on the loss of the best arm or related quantities, instead of the total number of rounds.
We derive the first small-loss bound for general strongly observable graphs, resolving an open problem of~\citet{DBLP:journals/corr/abs-1711-03639}.
Specifically, we develop an algorithm with regret $\mathcal{\tilde{O}}(\sqrt{\kappa L_\star})$ where $\kappa$ is the clique partition number and $L_\star$ is the loss of the best arm, and for the special case of self-aware graphs where every arm has a self-loop, we improve the regret to $\mathcal{\tilde{O}}(\min\{\sqrt{\alpha T}, \sqrt{\kappa L_\star}\})$ where $\alpha \leq \kappa$ is the independence number.
Our results significantly improve and extend those by~\citet{DBLP:journals/corr/abs-1711-03639} who only consider self-aware undirected graphs. 


Furthermore, we also take the first attempt at deriving small-loss bounds for weakly observable graphs.
We first prove that no typical small-loss bounds are achievable in this case, and then propose algorithms with alternative small-loss bounds in terms of the loss of some specific subset of arms.
A surprising side result is that $\otil(\sqrt{T})$ regret is achievable even for weakly observable graphs as long as the best arm has a self-loop.

Our algorithms are based on the Online Mirror Descent framework but require a suite of novel techniques that might be of independent interest.
Moreover, all our algorithms can be made parameter-free without the knowledge of the environment.

\end{abstract}

\begin{keywords}
multi-armed bandits, feedback graph, small-loss bounds.
\end{keywords}

\section{Introduction}

Adversarial multi-armed bandits with graph feedback is an online learning model that generalizes the classic expert problem~\citep{freund1997decision} as well as the standard multi-armed bandits problem~\citep{auer2002nonstochastic}.
In this model, the learner needs to choose one of the $K$ arms at each round, while simultaneously the adversary decides the loss of each arm. 
Afterwards, the learner receives feedback based on a graph with the $K$ arms as nodes.
Specifically, the learner observes the loss of every arm to which the chosen arm is connected.
Clearly, the full-information expert problem corresponds to having a complete feedback graph, while the standard multi-armed bandits problem corresponds to having a feedback graph with only self-loops.

\citet{pmlr-v40-Alon15} provided a full characterization of the minimax regret for this problem.
Specifically, it was shown that the minimax regret for strongly observable graphs and that for weakly observable graphs are $\tilde{\Theta}(\sqrt{\alpha T})$ and $\tilde{\Theta}({\wklydn}^{\nicefrac{1}{3}}T^{\nicefrac{2}{3}})$ respectively, where $T$ is the total number of rounds and $\alpha$ and $\wklydn$ are the independence number and {weak domination number} of the feedback graph respectively (see \pref{sec:setup} for definitions).

However, it is well-known that more adaptive data-dependent regret bounds are achievable for a wide range of online learning problems.
Among them, perhaps the most common one is the so-called {\it first-order} or {\it small-loss} bounds, which replaces the dependence on $T$ by the total loss of the best arm $L_\star \leq T$.
Such bounds are usually never worse than the worst-case bounds, but could be potentially much smaller if a relatively good arm exists.
Achieving small-loss bounds for bandits with graph feedback has been surprisingly challenging.
\citet{DBLP:journals/corr/abs-1711-03639} took the first attempt and their algorithms achieve regret $\mathcal{\tilde{O}}(\alpha^{\nicefrac{1}{3}}{L_\star}^{\nicefrac{2}{3}})$ or $\tilde{\mathcal{O}}(\sqrt{\kappa L_\star})$ ($\kappa$ is the clique partition number), but {\it only for self-aware undirected graphs} (self-aware means that every node has a self-loop).
It was left as a major open problem whether better and more general small-loss bounds are achievable.

\renewcommand{\arraystretch}{1.2}
\begin{table}
\centering
\caption{\small Main results and comparisons with prior work. 
$T$ is the number of rounds, $L_\star \leq T$ is the total loss of the best arm,
$\alpha$, $\kappa$, and $\wklydn$ are the independence, clique partition, and {weak domination number} respectively.
For our results for weakly observable graphs, $\gamma$ can be any value in $[\nicefrac{1}{3}, \nicefrac{1}{2}]$, $i^\star$ is the best arm, $S$ is the set of nodes with a self-loop, 
$L_{i_S^\star}$ is the loss of the best arm in $S$, 
$L_\wklyds$ is the average loss of nodes in a weakly dominating set,
and dependence on other parameters is omitted. 
See \pref{sec:setup} for detailed definitions.
All our algorithms have parameter-free versions.
}
\label{tab:table1}
\resizebox{\textwidth}{!}{%
\begin{tabular}{|cc|c|c|c|c|}
\hline
\multicolumn{2}{|c|}{\multirow{3}{*}{Graph Type}} & \multicolumn{3}{c|}{Regret} \\
\cline{3-5} 
& & \makecell{Minimax \\ \scriptsize\citep{pmlr-v40-Alon15}} & \small\citep{DBLP:journals/corr/abs-1711-03639} & \multicolumn{1}{c|}{\textbf{Our work}} \\
\hline
\multirow{3}{*}{\makecell{Strongly \\ Observable}} & \multicolumn{1}{|c|}{\small General} & \multirow{3}{*}{$\tilde{\Theta}(\sqrt{\alpha T})$} &N/A & \multicolumn{1}{c|}{$\tilde{\mathcal{O}}(\sqrt{(\kappa+1) L_\star})$} \\
\cline{2-2}
\cline{4-5}
& \multicolumn{1}{|c|}{\makecell{\small Special case: \\ \small self-aware}} & & \makecell{\vspace{-10pt}\ \\ $\mathcal{\tilde{O}}(\alpha^{\nicefrac{1}{3}}{L_\star}^{\nicefrac{2}{3}})$, $\tilde{\mathcal{O}}(\sqrt{\kappa L_\star})$ \\ \small(undirected graphs only)} & \multicolumn{1}{c|}{$\mathcal{\tilde{O}}(\min\{\sqrt{\kappa L_\star}, \sqrt{\alpha T}\})$} \\
\hline
\multirow{4}{*}{\makecell{Weakly \\ Observable}} &  \multicolumn{1}{|c|}{\small General} & $\tilde{\Theta}({\wklydn}^{\nicefrac{1}{3}}T^{\nicefrac{2}{3}})$ & \multirow{4}{*}{N/A} & 
\makecell{
(no $o(L_\star)$ 
 bounds achievable) \\
$
\begin{cases}
    \tilde{\mathcal{O}}(L_{\wklyds}^{1-\gamma}),&\text{if $i^\star\in S$},\\
    \tilde{\mathcal{O}}(L_{\wklyds}^{\nicefrac{(1+\gamma)}{2}}),&\text{else}.
\end{cases}
$ 
}\\
\cline{2-3}
\cline{5-5}
& \multicolumn{1}{|c|}{\makecell{\small Special case: \\ \small bipartite}} &  $\tilde{\Theta}(T^{\nicefrac{2}{3}})$ & & 
$
\begin{cases}
    \tilde{\mathcal{O}}(\sqrt{L_\star}),&\text{if $i^\star\in S$},\\
    \tilde{\mathcal{O}}(L_{i_S^\star}^{
    \nicefrac{2}{3}}),&\text{else}.
\end{cases}
$ \\
\hline
\end{tabular}
}
\end{table}

Our work makes a significant step towards a full understanding of small-loss bounds for bandits with a fixed directed feedback graph.
Specifically, our contributions are (see also \pref{tab:table1}):
\begin{itemize}
    \item (\pref{sec: kappaLstar}) For general strongly observable graphs, we develop an algorithm with regret $\mathcal{\tilde{O}}(\sqrt{(\kappa +1) L_\star})$.
    This is the first small-loss bound for the general case, extending the results of~\citep{DBLP:journals/corr/abs-1711-03639} that only hold for self-aware undirected graphs and resolving an open problem therein.
    \item (\pref{sec: minAlphaKappa}) For the special case of self-aware (directed) graphs, we develop an algorithm with regret $\mathcal{\tilde{O}}(\min\{\sqrt{\alpha T},\sqrt{\kappa L_\star}\})$, again strictly improving~\citep{DBLP:journals/corr/abs-1711-03639} by providing an extra robustness guarantee (note that $\alpha \leq \kappa$ always holds).
    \item (\pref{sec:weakly}) For weakly observable graphs (where small-loss bounds have not been studied before at all), we prove that {\it no algorithm can achieve typical small-loss bounds} (such as $o(L_\star)$). Despite this negative result, we develop an algorithm with regret $\tilde{\mathcal{O}}(L_{\wklyds}^{\nicefrac{2}{3}})$ where $L_{\wklyds}$ is the average loss of a weakly dominating set (and dependence on other parameters is omitted for simplicity). More generally, we also achieve different trade-offs between the case when the best arm has a self-loop and the case without, such as $\tilde{\mathcal{O}}(\sqrt{L_{\wklyds}})$ versus $\tilde{\mathcal{O}}(L_{\wklyds}^{\nicefrac{3}{4}})$.
We further consider a special case with a complete bipartite graph, and show that our algorithm achieves $\tilde{\mathcal{O}}(\sqrt{L_\star})$ when the best arm has a self-loop and $\tilde{\mathcal{O}}(L_{i_S^\star}^{\nicefrac{2}{3}})$ otherwise, where $L_{i_S^\star}$ is the loss of the best arm with a self-loop.
A surprising implication of our result is that $\tilde{\mathcal{O}}(\sqrt{T})$ regret is possible even for weakly observable graphs as long as the best arm has a self-loop.
    \item (Appendix) We provide parameter-free versions of all our algorithms using sophisticated doubling tricks, which we emphasize is highly non-trivial for bandit settings, especially because some of our algorithms consist of a layer structure {combining} different subroutines.
\end{itemize}

Our algorithms are based on the well-known Online Mirror Descent framework, but importantly with a suite of different techniques including hybrid regularizers, unconstrained loss shifting trick, increasing learning rates, combining algorithms with partial information, adding correction terms to loss estimators, and their combination in an innovative way.
We defer further discussion on the novelty of each component and comparisons with prior work to the description of each algorithm.

\paragraph{Related work.}
The bandits with graph feedback model was first proposed by~\citep{mannor2011bandits}. 
Later, \citet{pmlr-v40-Alon15,alon2017nonstochastic} gave a full characterization of the minimax regret for this problem. 
There are many follow-ups that consider different variants and extensions of this problem, such as~\citep{kocak2016online, feng2018online, rangi2018online, arora2019bandits}.

Small-loss bounds have been widely studied in the online learning literature.
For the full-information expert problem, the classic Hedge algorithm~\citep{freund1997decision} achieves $\otil(\sqrt{L_\star})$ regret already.
For the standard multi-armed bandits problem and its variant semi-bandits, there are also several different approaches to achieve $\otil(\sqrt{L_\star})$ regret~\citep{allenberg2006hannan, neu2015first, foster2016learning, pmlr-v75-wei18a, bubeck2020first}.
Even for the challenging contextual bandits setting (which is in fact a special case of learning with time-varying feedback graphs), 
it was shown by~\citet{allen2018make} that $\otil(\sqrt{L_\star})$ regret is also achievable.

The work most related to ours is~\citep{DBLP:journals/corr/abs-1711-03639}.
As mentioned, we significantly extend their results to more general graphs, including graphs with directed edges, graphs without self-loops, and even weakly observable graphs, and we also improve their bound for self-aware graphs.
Our algorithms are also based on very different ideas compared to theirs which are mainly built on the recursive arm freezing technique.
We point out that, however, they also studied high probability bounds and time-varying graphs for some cases, which is not the focus of this work.

\section{Problem Setup and Notations}\label{sec:setup}

Throughout the paper, we denote $\{1,\dots,m\}$ by $[m]$ for some positive integer $m$. Before the game starts, the adversary decides a sequence of $T$ loss vectors $\ell_1,\dots,\ell_T\in [0,1]^K$ for some integers $K \geq 2$ and $T\geq2K$, and a directed feedback graph $G=([K],E)$ for $E\subseteq [K]\times[K]$ which is fixed and known. Each node in the graph represents one of the $K$ arms, and in this paper we use the terms ``arm'' and ``node'' interchangeably. At each round $t\in [T]$, the learner selects an arm $i_t\in[K]$ and incurs loss $\ell_{t,i_t}$. At the end of this round, the learner receives feedback according to $G$. Specifically, the learner observes the loss of arm $i$ for all $i$ such that $i_t\in \Nin (i)$, where $\Nin(i) \triangleq \{j:(j,i)\in E\}$ is the set of nodes that can observe $i$. 
The \emph{regret} with respect to an arm $i$ is defined as 
$
\Reg_i\triangleq\mathbb{E}\left[\sum^T_{t=1}\ell_{t,i_t}-\sum_{t=1}^{T}\ell_{t,i}\right],
$
which is the expected difference between the learner's total loss and that of arm $i$
(the expectation is with respect to the learner's internal randomness). 
We denote the best arm by $i^\star\triangleq \argmin_{i\in [K]}\sum_{t=1}^{T}\ell_{t,i}$ and define $\Reg\triangleq\Reg_{i^\star}$. 
\citet{pmlr-v40-Alon15} show that the minimax regret (in terms of $T$) for strongly observable graphs and weakly observable graphs (definitions to follow) are $\Reg = \tilde{\Theta}(\sqrt{T})$ and $\Reg = \tilde{\Theta}(T^{\nicefrac{2}{3}})$ respectively.


Our goal is to obtain the so-called \emph{small-loss} regret bounds that could potentially be much smaller than the minimax bounds. Specifically, for an arm $i$, we denote its total loss by $L_{i}\triangleq\sum^T_{t=1}\ell_{t,i} \leq T$, and we use the shorthand $L_\star\triangleq L_{i^\star}$. 
Our goal is to replace the dependence of $T$ by $L_\star$ when bounding $\Reg$, or more generally, to replace $T$ by $L_i$ when bounding $\Reg_i$ for each arm $i$.
Below, we introduce some graph-related notions and other notations necessary for discussions.

\paragraph{Observability.} 
A node $i$ is \emph{observable} if $\Nin(i) \neq \emptyset$.
An observable node is \emph{strongly observable} if either $i\in \Nin(i)$ or $\Nin(i)=[K]\backslash\{i\}$, and \emph{weakly observable} otherwise. 
Similarly, a graph is observable if all of its nodes are observable.
An observable graph is strongly observable if all nodes are strongly observable, and weakly observable otherwise.

\paragraph{Special cases.} 
We denote by $S\triangleq \{i\in[K]:i\in \Nin(i)\}$ the subset of nodes with a self-loop, and by $\bar{S}\triangleq[K]\backslash S$ the subset of nodes without a self-loop. We further use $\slpn$ and $\woslpn$ to denote $|S|$ and $|\bar{S}|$.
A graph is \emph{self-aware} if $S=[K]$, that is, every node has a self-loop, which is a special case of strongly observable graphs. 
We also consider a special case of weakly observable graphs with $(i, j) \in E$ for every $i \in S$ and every $j \in \bar{S}$, and call it a \emph{directed complete bipartite graph}.\footnote{Note that unlike the traditional definition of bipartite graphs, here we allow additional edges other than those from $S$ to $\bar{S}$, as adding more edges only makes the problem easier.} 

\paragraph{Independence sets and cliques.} 
We remind the reader the standard concepts of independence sets and cliques.
An {independent set} $\mathcal{I}$ is a subset of nodes such that for any two distinct nodes $i,j\in\mathcal{I}$, we have $(i,j)\notin E$. 
The {independence number} of a graph is the cardinality of its largest independent set. 
A clique $\mathcal{C}$ is a subset of nodes such that for any two distinct nodes $i,j\in\mathcal{C}$, we have $(i,j) \in E$. 
A clique partition $\{\mathcal{C}_1,\dots,\mathcal{C}_m\}$ of a graph is a partition of its nodes such that each $\mathcal{C}_k$ in this collection is a clique. The clique partition number of a graph is the cardinality of its smallest clique partition. 
As in previous works, our bounds for strongly observable graphs depend on the independence number $\alpha$ of $G$, or the clique partition number $\kappa$ of the subgraph $G_S$ obtained by restricting $G$ to only the nodes in $S$.
Note that $\alpha \leq \kappa+1$ always holds.

Some of our algorithms rely on having a clique partition of $G_S$,
which we assume is given, even though it is in general NP-hard to find~\citep{karp1972reducibility}. 
We emphasize that, however, our algorithms work with any clique partition and the bounds hold with $\kappa$ replaced by the size of this partition.

\paragraph{Weakly dominating sets.} 
Following~\citep{pmlr-v40-Alon15}, for a weakly observable graph, we define a \emph{weakly dominating set} $\wklyds$ to be a set of nodes such that all weakly observable nodes can be observed by at least one node in $\wklyds$.
Our bounds for weakly observable graphs depend on the {\emph{weak domination number}} $\wklydn$ of graph $G$, which is the cardinality of its smallest weakly dominating set. 
Similarly, we assume that a weakly dominating set of size $d$ is given, but our algorithms work using any weakly dominating set and our bounds hold with $d$ replaced by the size of this set.

\paragraph{Other notations.}
For a differentiable convex function $\psi$ defined on a convex set $\Omega$, the associated Bregman divergence is defined as $\brgmd_\psi(x,y)=\psi(x)-\psi(y)-\inner{\nabla\psi(y),x-y}$ for any two points $x,y\in \Omega$. 
For a positive definite matrix $M \in \fR^{K\times K}$, we define norm $\norm{z}_M \triangleq \sqrt{z^\top M z}$ for any vector $z \in \fR^K$.
For two matrices $M_1$ and $M_2$, $M_1 \preceq M_2$ means that $M_2 - M_1$ is positive semi-definite, and for two vectors $v_1$ and $v_2$, $v_1 \preceq v_2$ means that $v_1$ is coordinate-wise less than or equal to $v_2$.
We denote the $(K-1)$-dimensional simplex by $\Delta(K)$, 
the diagonal matrix with $v_1, \ldots, v_K$ on the diagonal by $\diag\{v_1, \ldots, v_K\}$,
and the all-zero and all-one vectors in $\mathbb{R}^K$ by $\mathbf{0}$ and $\mathbf{1}$ respectively. 
The notation $\otil(\cdot)$ hides logarithmic dependence on $K$ and $T$.


\section{Strongly Observable Graphs}
\label{sec:strongly}

In this section, we focus on strongly observable graphs. 
We first show how to achieve $\Reg = \mathcal{\tilde{O}}(\sqrt{(\kappa+1) L_\star})$ in  general, and then discuss how to further improve it to $\mathcal{\tilde{O}}(\min\{\sqrt{\alpha T}, \sqrt{\kappa L_\star}\})$ for the special case of self-aware graphs (\pref{sec: minAlphaKappa}).

There are three key components/ideas to achieve $\mathcal{\tilde{O}}(\sqrt{(\kappa+1) L_\star})$ regret. 
Specifically, starting from the \expthreeG algorithm of~\citep{pmlr-v40-Alon15}, which is an instance of Online Mirror Descent (OMD) with the entropy regularizer, natural loss estimators for graph feedback, and an additional $\Theta(1/\sqrt{T})$ amount of uniform exploration, we make the following three modifications:
\begin{itemize}
\item (\pref{sec: remove uniform exploration}) Reduce the amount of uniform exploration to $\Theta(1/T)$ and add a constant amount of log-barrier to the regularizer. We show that this modification maintains the same $\mathcal{\tilde{O}}(\sqrt{\alpha T})$ regret as \expthreeG, but importantly, the smaller amount of uniform exploration is the key for further obtaining small-loss bounds.
\item (\pref{sec: betaLstar}) Replace the entropy regularizer with the log-barrier regularizer for nodes in $S$. This leads to a small-loss bound of order $\mathcal{\tilde{O}}(\sqrt{({\slpn}+1)L_\star})$. 
\item (\pref{sec: kappaLstar}) Create a clique partition for nodes in $S$, run a Hedge variant within each clique, and run the algorithm from \pref{sec: betaLstar} treating each clique as a meta-node,
which finally improves the regret to $\mathcal{\tilde{O}}(\sqrt{(\kappa+1) L_\star})$.
This part relies on highly nontrivial extensions of techniques from~\citep{DBLP:journals/corr/AgarwalLNS16} on combining algorithms in the bandit setting.
\end{itemize}



\subsection{Reducing the Amount of Uniform Exploration}\label{sec: remove uniform exploration}

We start by describing the \expthreeG algorithm of~\citep{pmlr-v40-Alon15}.
It maintains a distribution $p_t \in \Delta(K)$ for each time $t$, and samples $i_t$ according to $p_t$.
Then a standard importance-weighted loss estimator $\hat{\ell}_t$ is constructed such that 
\begin{equation}\label{eq:estimator}
\hat{\ell}_{t,i} = \frac{\ell_{t,i}}{W_{t,i}}\mathbbm{1}\left\{i_t\in \Nin (i)\right\} \quad\text{where}\;  W_{t,i}\triangleq\sum_{j\in \Nin (i)}p_{t,j}.
\end{equation}
It is clear that $\E[\hat{\ell}_{t,i}] = \ell_{t,i}$, that is, the estimator is unbiased.
With such a loss estimator, the distribution is updated according to the classic OMD algorithm:
\begin{equation}\label{eq:OMD}
p_{t+1}= \argmin_{p\in \Omega}\inner{p, \hat{\ell}_{t}}+ {\brgmd}_{\psi}\left(p, p_{t}\right),
\quad\text{with}\;\; p_1 = \argmin_{p\in \Omega} \psi(p).
\end{equation}
For \expthreeG, $\psi(p) = \frac{1}{\eta}\sum_{i \in [K]} p_i \ln p_i$ is the standard Shannon entropy regularizer with learning rate $\eta \leq 1/2$, and $\Omega = \left\{p\in \Delta(K):p_i\ge\frac{2\eta}{K},\forall i\in [K]\right\}$ is the clipped simplex and enforces $\mathcal{O}(\eta)$ amount of uniform exploration.\footnote{%
In the original \expthreeG algorithm, $\Omega$ is the exact simplex $\Delta(K)$ and uniform exploration is enforced by sampling $i_t$ according to $p_t$ with probability $1-2\eta$ and according to a uniform distribution with probability $2\eta$. Nevertheless, our slight modification essentially serves the same purpose and makes subsequent discussions easier.
}

Standard OMD analysis shows that the instantaneous regret of \expthreeG against any $u\in \Delta(K)$  at time $t$ is bounded as 
$\langle p_t - u, \hat{\ell}_t \rangle \leq {\brgmd}_{\psi}(u,p_t)-{\brgmd}_{\psi}(u,p_{t+1}) +   \|\hat{\ell}_t\|_{\nabla^{-2}\psi(p_t)}^2$.
However, the last term $\|\hat{\ell}_t\|_{\nabla^{-2}\psi(p_t)}^2$ (often called the local-norm term) could be prohibitively large for general strongly observable graphs.
The analysis of \expthreeG overcomes this issue via a key {\it loss shifting} trick. 
Specifically, it is shown that the following more general bound holds
\begin{equation}\label{eq:loss_shifting}
\inner{p_t - u, \hat{\ell}_t} \leq {\brgmd}_{\psi}(u,p_t)-{\brgmd}_{\psi}(u,p_{t+1}) +  \norm{\hat{\ell}_t - z \cdot \one}_{\nabla^{-2}\psi(p_t)}^2
\end{equation}
for any $z \leq 1/\eta$, and in particular,
with $z = \sum_{i\in \bar{S}} p_{t,i} \hat{\ell}_{t,i}$, the local-norm term $\|\hat{\ell}_t - z \cdot \one\|_{\nabla^{-2}\psi(p_t)}^2$ is bounded by $\mathcal{\tilde{O}}(\eta \alpha)$ in expectation.
This choice of $z$ is indeed not larger than $1/\eta$ due to the form of $\hat{\ell}_t$ and {\it importantly the $\mathcal{O}(\eta)$ amount of uniform exploration}.
The rest of the analysis is straightforward and shows that $\Reg = \otil(\frac{1}{\eta} + \eta \alpha T)$, which is $\otil(\sqrt{\alpha T})$ by picking $\eta = 1/\sqrt{\alpha T}$.

To obtain small-loss bounds, one clear obstacle in \expthreeG is the uniform exploration, which contributes to $\order(\eta T) = \order(\sqrt{T})$ regret already by the above optimal choice of $\eta$.
Our first step is thus to get rid of this large amount uniform exploration,
and we take an approach that is completely different from~\citep{DBLP:journals/corr/abs-1711-03639}.
Specifically, noting that the key reason to have uniform exploration is the constraint $z \leq 1/\eta$ in the loss shifting trick \pref{eq:loss_shifting},
our key idea is to {\it remove this constraint completely}, which turns out to be possible if the regularizer contains a constant amount of log-barrier (similar to~\citep{pmlr-v83-bubeck18a, DBLP:journals/corr/abs-1905-12950}), as shown in the following lemma.

\begin{lemma}[Unconstrained Loss Shifting]
\label{lem:unifyOMDregret}
Let $p_{t+1}=\argmin_{p\in \Omega}\; \langle p, \hat{\ell}_t \rangle+{\brgmd}_{\psi}(p, p_t)$, for $\Omega \subseteq \Delta(K)$ and $\psi: \Omega \rightarrow \fR$.
Suppose (a) $0\le \hat{\ell}_{t,i}\le\max\left\{\frac{1}{p_{t,i}},\frac{1}{1-p_{t,i}}\right\},\;\forall i\in[K]$,
(b) $\nabla^{-2}\psi(p) \preceq 4 \nabla^{-2}\psi(q)$ holds when $p \preceq 2q$,  
(c) $\nabla^2\psi(p)\succeq \diag\left\{\nicefrac{\constK}{p_{1}^2},\dots,\nicefrac{\constK}{p_{K}^2}\right\},\; \forall p\in \Omega$.
Then we have
    \begin{align}
        \inner{p_t-u, \hat{\ell}_t}\le {\brgmd}_{\psi}(u,p_t)-{\brgmd}_{\psi}(u,p_{t+1})+8 \min_{z\in \fR}\|\hat{\ell}_t-z\cdot \mathbf{1}\|_{\nabla^{-2}\psi(p_t)}^2. \label{eq:unconstrained_shifting}
    \end{align}
\end{lemma}

Condition (a) is clearly satisfied for strongly observable graphs if $\hat{\ell}_t$ is defined by \pref{eq:estimator} since $W_{t,i} \geq p_{t,i}$ for $i\in S$ and $W_{t,i}=1-p_{t,i}$ for $i \notin S$.
Condition (b) is trivially satisfied for all common regularizers for the simplex such as Shannon entropy, Tsallis entropy, log-barrier, and any of their combinations.
Finally, to ensure Condition (c), one only needs to include a log-barrier component $c \sum_{i\in[K]}\ln\frac{1}{p_i}$ for $c \geq \constK$ in the regularizer, whose Hessian is exactly $\diag\{\nicefrac{c}{p_{1}^2},\dots,\nicefrac{c}{p_{K}^2}\}$.
This inspires us to propose the following hybrid regularizer
\begin{equation}\label{eq:hybrid_regularizer}
\psi(p)=\frac{1}{\eta}\sum_{i\in [K]}p_i\ln p_i+ c\sum_{i\in [K]}\ln \frac{1}{p_i},
\end{equation}
and we prove the following theorem.

\begin{theorem}\label{thm:alphaTEXP3}
The OMD update \pref{eq:OMD} with $\Omega = \left\{p\in \Delta(K):p_i\ge\frac{1}{T},\forall i\in [K]\right\}$, $\hat{\ell}_t$ defined in \pref{eq:estimator}, and $\psi$ defined in \pref{eq:hybrid_regularizer} for $c = \constK$ ensures 
$
\Reg\le \otil\big(\frac{1}{\eta}+ \eta \alpha T + K^2 \big)
$
for any strongly observable graph.
Choosing $\eta=1/\sqrt{\alpha T}$, we have $\Reg=\mathcal{\tilde{O}}\big(\sqrt{\alpha T}+K^2\big).$
\end{theorem}

Note that we still enforce a small $1/T$ amount of uniform exploration, which is important for a technical lemma~\citep[Lemma~5]{pmlr-v40-Alon15}, but this only contributes $\order(K)$ regret.
Also, adding the log-barrier leads to a small $\order(K^2)$ overhead in the Bregman divergence term $D_\psi(u, p_1)$, but only makes the local-norm term smaller and thus $\min_z \|\hat{\ell}_t-z\cdot \mathbf{1}\|_{\nabla^{-2}\psi(p_t)}^2$ is still of order $\otil(\eta\alpha)$ in expectation.
The proof of \pref{thm:alphaTEXP3} is now straightforward and is deferred to \pref{app:alphaT}.

\subsection{\texorpdfstring{$\mathcal{\tilde{O}}\big(\sqrt{({\slpn}+1) L_\star}\big)$}{} Regret Bound}\label{sec: betaLstar}

Having solved the uniform exploration issue, we now discuss our first attempt to obtain small-loss bounds for strongly observable graphs.
Inspired by the fact that for multi-armed bandits, that is, the case where $E$ contains all the self-loops but nothing else, OMD with the log-barrier regularizer achieves a small-loss bound~\citep{foster2016learning}, we propose to replace the entropy regularizer with the log-barrier for all nodes in $S$, while keeping the hybrid regularizer~\pref{eq:hybrid_regularizer} for nodes in $\bar{S}$:
\begin{equation}\label{eq:hybrid_regularizer2}
\psi(p)= \frac{1}{\eta}\sum_{i\in S}\ln \frac{1}{p_i}+\frac{1}{\eta}\sum_{i\in \bar{S}}p_i\ln p_i+c\sum_{i\in \bar{S}}\ln \frac{1}{p_i}.
\end{equation}
We note that it is important not to also use $1/\eta$ amount of log-barrier for nodes in $\bar{S}$, since this leads to an overhead of $K/\eta$ for the Bregman divergence term and thus at best gives $\otil(\sqrt{KL_\star})$ regret.
We prove the following theorem for our proposed algorithm.

\begin{theorem}\label{thm:thmbetalstar}
OMD \pref{eq:OMD} with $\Omega = \left\{p\in \Delta(K):p_i\ge\frac{1}{T},\forall i\in [K]\right\}$, $\hat{\ell}_t$ defined in \pref{eq:estimator}, and $\psi$ defined in \pref{eq:hybrid_regularizer2} for $c = \constK$ and $\eta \leq \frac{1}{64K}$ ensures 
$
\Reg=\mathcal{O}\left(\frac{{\slpn}\ln T+\ln K}{\eta}+\eta L_\star+K^2\ln T\right)
$
for any strongly observable graph.
Choosing $\eta=\min\left\{\sqrt{\frac{{\slpn}+1}{L_\star}},\frac{1}{\constK}\right\}$ gives $ \mathcal{\tilde{O}}\left(\sqrt{({\slpn}+1) L_\star}+K^2\right)$.\footnote{%
In fact, the $s$ dependence can be improved to the number of nodes that are not observed by every other nodes
(by using log-barrier only for these nodes).
However, we simplify the presentation with a looser bound since this improvement does not help improve the final main result in \pref{sec: kappaLstar}.
}
\end{theorem}

While the algorithmic idea is straightforward, the main challenge in the analysis is to deal with the nodes in $\bar{S}$.
Specifically, for the particular choices of $\eta$ and $c$, we know that the conditions of \pref{lem:unifyOMDregret} hold, and
the key is thus again to bound the local-norm term $\min_{z\in \fR}\|\hat{\ell}_t-z\cdot \mathbf{1}\|_{\nabla^{-2}\psi(p_t)}^2$.
Simply taking $z = 0$ or the previous choice $z = \sum_{i\in \bar{S}} p_{t,i} \hat{\ell}_{t,i}$ from~\citep{pmlr-v40-Alon15} does not give the desired bound.
Instead, we propose a novel shift: $z = \hat{\ell}_{t, i_0}$ for some $i_0 \in \bar{S}$ with $p_{t,i_0} \geq 1/2$, or $z=0$ if no such $i_0$ exists. 
Direct calculations then show $\|\hat{\ell}_t-z\cdot \mathbf{1}\|_{\nabla^{-2}\psi(p_t)}^2 = \order(\eta \langle p_t, \hat{\ell}_t \rangle)$.
In expectation, the local-norm term is thus related to the loss of the algorithm $\langle p_t, \ell_t \rangle$, and it is well-known that this is enough for obtaining small-loss bounds.
For the complete proof, see \pref{app:betaLstar}.

\subsection{\texorpdfstring{$\mathcal{\tilde{O}}\big(\sqrt{(\kappa+1) L_\star}\big)$}{} Regret Bound}\label{sec: kappaLstar}

Finally, we discuss how to further improve our bound to $\otil(\sqrt{(\kappa+1) L_\star})$.
Let $\mathcal{C}_1,\dots,\mathcal{C}_\kappa$ be a clique partition of $G_S$ (recall that $G_S$ is $G$ restricted to $S$).
Essentially, we compress each clique as one meta-node, and together with nodes from $\bar{S}$, this creates a meta-graph with $\beta \triangleq \kappa+\woslpn$ nodes, which can be seen as a ``low-resolution'' version of $G$.
We index these meta-nodes as $1, \ldots, \kappa$, 
and without loss of generality we assume that the original indices of nodes in $\bar{S}$ are $\kappa+1, \ldots, \kappa+\woslpn$,
so that $[\beta]$ is the set of nodes for this meta-graph, and $[\kappa]$ is the set of nodes with a self-loop (taking the same role as $S$ in the original graph).
If we were actually dealing with a problem with this meta-graph, running the algorithm from \pref{thm:thmbetalstar} would thus give $\otil(\sqrt{(\kappa+1) L_\star})$ regret.

To solve the original problem, however, we need to specify what to do when a meta-node is selected.
Note that within a meta-node, we are essentially facing the classic expert problem with full-information~\citep{freund1997decision} since nodes are all connected with each other.
A natural idea is thus to run an expert algorithm with a small-loss bound within each clique, and when a clique is selected, follow the suggestion of the corresponding expert algorithm. 
We choose to use an adaptive version of Hedge~\citep{freund1997decision} as the expert algorithm, with details deferred to \pref{alg:AdaHedge} in \pref{app:kappaLstar}.
Importantly, the regret of Hedge has only logarithmic dependence on the number of nodes and thus does not defeat the purpose of obtaining  $\otil(\sqrt{(\kappa+1) L_\star})$ regret.

\pref{fig:alg3} illustrates the main idea of our algorithm, and \pref{alg:kappalstaralg} shows the complete pseudocode.
We use $i$ to index a node in the original graph and $j$ to index a node in the meta-graph. 
Note that nodes from $\bar{S}$ appear in both graphs so they are indexed by either $i$ or $j$, depending on the context.
The $\kappa$ instances of Hedge are denoted by $\mathcal{A}_1, \ldots, \mathcal{A}_\kappa$, where $\mathcal{A}_j$ only operates over nodes in $\mathcal{C}_j$.
For notational convenience, however, we require $\mathcal{A}_j$ to output at time $t$ a distribution $\widetilde{p}_{t}^{(j)} \in \Delta(K)$ over all nodes with the constraint $\widetilde{p}_{t, i}^{(j)} = 0$ for all $i \notin \mathcal{C}_j$ (\pref{line:receive_A_j}), and we also feed the estimated losses of all nodes to $\mathcal{A}_j$ (\pref{line:feed_A_j}), even though it ignores those outside $\mathcal{C}_j$.

The algorithm maintains a distribution $p_t \in \Delta(\beta)$ for the meta-graph, updated using the algorithm from \pref{thm:thmbetalstar} (\pref{line:OMD}).
The only difference is that we use a time-varying regularizer defined in \pref{eq:hybrid_regularizer3}, where the learning rate $\eta_{t,j}$ for meta-node $j \in [\kappa]$ is time-varying and also different for different $j$ (all starting from $\eta_{1,j}=\eta$; more explanation to follow).
At the beginning of time $t$, the algorithm first samples $j_t \in p_t$.
If $j_t$ happens to be a node in $\bar{S}$, we play $i_t = j_t$;
otherwise, we sample $i_t$ from the distribution received from $\mathcal{A}_{j_t}$.
See \pref{line:sampling1} and
\pref{line:sampling2}.

\begin{figure}[t]
\centering
\includegraphics[width=0.8\linewidth]{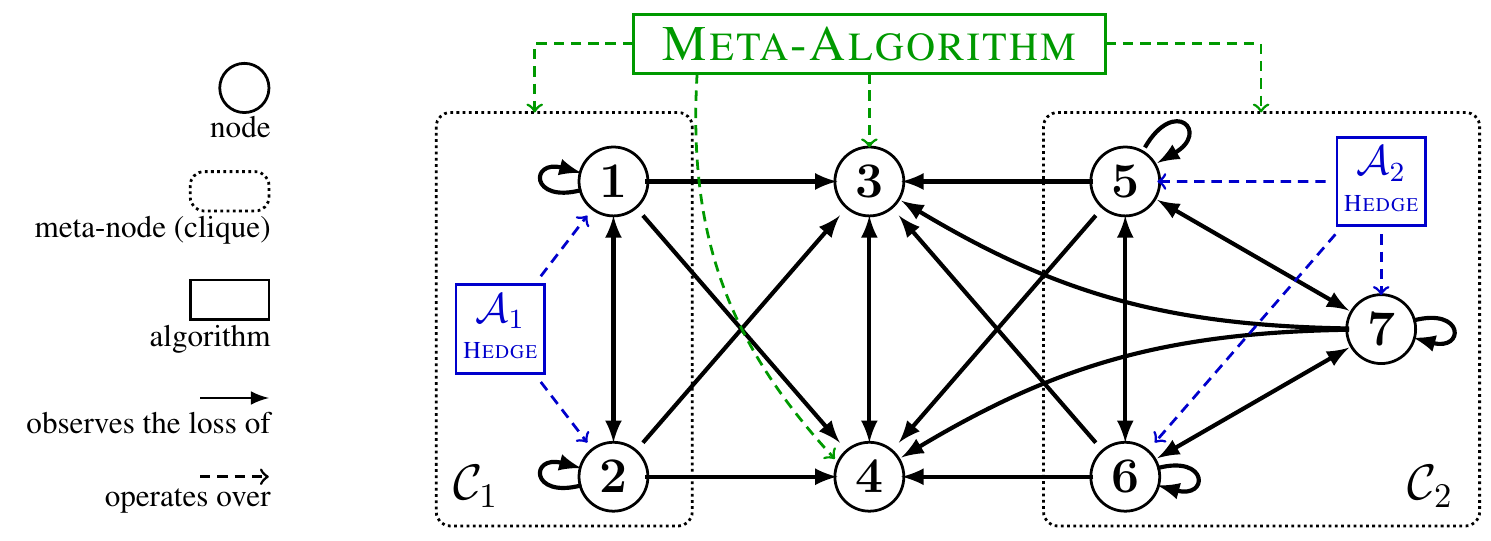}\caption{ An illustration of \pref{alg:kappalstaralg} for a graph with 7 nodes. Here, we have $S = \{1,2,5,6,7\}$, $\bar{S} = \{3, 4\}$, and $\kappa = 2$ with $\mathcal{C}_1 =\{1,2\}$ and $\mathcal{C}_2=\{5,6,7\}$ being a minimum clique partition of $G_S$. The meta-algorithm operates over nodes 3 and 4, and also the two cliques, each with a Hedge instance running inside.
}\label{fig:alg3}
\end{figure}

\setcounter{AlgoLine}{0}
\begin{algorithm}[t]\caption{Algorithm for General Strongly Observable Graphs}\label{alg:kappalstaralg}
\textbf{Input:} Feedback graph $G$ and a clique partition $\{\mathcal{C}_1,\dots,\mathcal{C}_\kappa\}$ of $G_S$, parameters $\eta$, $c$.

\textbf{Define:} $\beta= \kappa +\woslpn$ and
$\Omega =\left\{p\in \Delta(\beta):p_j \ge \frac{1}{T}, \forall j\in[\beta]\right\}$.

\textbf{Initialize:} $p_1=\argmin_{p\in \Omega}\psi_1(p)$ (see \pref{eq:hybrid_regularizer3}), $\eta_{1, j} = \eta, \rho_{1, j} = 2\kappa, \forall j\in [\kappa]$.

\textbf{Initialize:} an instance $\mathcal{A}_j$ of adaptive Hedge (\pref{alg:AdaHedge}) with nodes in $\mathcal{C}_j$, $\forall j\in [\kappa]$.

\For{$t=1,2,\dots, T$}{
\nl For each $j\in [\kappa]$, receive $\widetilde{p}_{t}^{(j)} \in \{p\in \Delta(K) : p_{t,i} = 0, \;\forall i \notin \mathcal{C}_j\}$ from $\mathcal{A}_j$.
\label{line:receive_A_j}

\nl Sample $j_t \sim p_t$.
\label{line:sampling1}

\nl \lIf{$j_t\in [\kappa]$}{draw $i_t \sim \widetilde{p}_{t}^{(j_t)}$; \textbf{else} let $i_t=j_t$.}
\label{line:sampling2}

\nl Pull arm $i_t$ and receive feedback $\ell_{t,i}$ for all $i$ such that $i_t\in \Nin(i)$.

\nl Construct estimator $\widetilde{\ell}_t \in \fR^K$ such that
$
\widetilde{\ell}_{t,i}=\begin{cases}
\frac{\ell_{t,i}}{p_{t,j_t}}\mathbbm{1}\{j_t \in [\kappa], i \in \mathcal{C}_{j_t}\},&\text{for $i\in S$},\\
\frac{\ell_{t,i}}{1-p_{t,i}}\mathbbm{1}\{i\ne i_t\},&\text{for $i\in \bar{S}$}.
\end{cases}
$
\label{line:estimator}

\nl For each $j\in [\kappa]$, feed $\widetilde{\ell}_t$ to $\mathcal{A}_j$.
\label{line:feed_A_j}

\nl Construct estimator $\hat{\ell}_{t} \in \fR^\beta$ for meta-nodes such that
$
\hat{\ell}_{t,j}=\begin{cases}
\inner{\widetilde{p}_{t}^{(j)}, \widetilde{\ell}_{t}}, &\mbox{for $j\in [\kappa]$},\\
\widetilde{\ell}_{t,j},&\mbox{for $j\in \bar{S}$}.
\end{cases}
$
\label{line:meta_estimator}

\nl Compute $p_{t+1}=\argmin_{p\in\Omega}\left\{ \inner{p,\hat{\ell}_{t}}+{\brgmd}_{\psi_{t}}(p,p_{t})\right\}$ where
\begin{equation}\label{eq:hybrid_regularizer3}
\psi_t(p)=\sum_{j\in [\kappa]}\frac{1}{
\eta_{t,j}}\ln \frac{1}{p_j}+\frac{1}{\eta}\sum_{j\in \bar{S}}p_j\ln p_j+c\sum_{j\in \bar{S}}\ln \frac{1}{p_j}.
\end{equation}
\label{line:OMD}
\nl \For{$j \in [\kappa]$}{
\lIf {$\frac{1}{p_{t+1,j}}>\rho_{{t},j}$}{set $\rho_{t+1,j}=\frac{2}{p_{t+1,j}},\eta_{t+1,j}=\eta_{t,j}e^{\frac{1}{\ln T}}$.}
\lElse{set $\rho_{t+1,j}=\rho_{{t},j},\eta_{t+1,j}=\eta_{t,j}$.}}
}
\label{line:increasing_eta}
\end{algorithm}

After playing arm $i_t$ and receiving loss feedback, we construct loss estimator $\widetilde{\ell}_t \in \fR^K$ for nodes in $G$ (\pref{line:estimator}) and estimator $\hat{\ell}_t \in \fR^\beta$ for nodes in the meta-graph (\pref{line:meta_estimator}).
The estimator for nodes in $\bar{S}$, for either $G$ or the meta-graph, is exactly the same as the standard one described in \pref{eq:estimator}.
The estimator for a node $i \in S$ also essentially follows \pref{eq:estimator}, except that we ignore all edges that point to $i$ but are not from those nodes in the same clique, so the probability of observing $i$ is $p_{t, j}$ for $j$ being the index of the clique to which $i$ belongs.\footnote{%
One could also follow exactly \pref{eq:estimator} to construct more complicated estimators, but it does not lead to a better bound.
}
Finally, the estimator for a meta-node $j \in [\kappa]$ is simply $\langle \widetilde{p}_{t}^{(j)}, \widetilde{\ell}_{t}\rangle$, which is the estimated loss of the corresponding Hedge $\mathcal{A}_j$.

While the idea of combining algorithms in such a two-level hierarchy is natural and straightforward, doing it in a partial-information setting is {\it notoriously challenging} and requires extra techniques, as explained in detail in~\citep{DBLP:journals/corr/AgarwalLNS16}.
In a word, the difficulty is that the Hedge algorithms do not always receive feedback and thus do not yield the usual regret bound as one would get for a full-information problem.
To address this issue, we apply the increasing learning rate technique from~\citep{DBLP:journals/corr/AgarwalLNS16}. 
Specifically, we maintain a threshold $\rho_{t,j}$ for each time $t$ and each meta-node $j$ (starting from $\rho_{1,j} = 2\kappa$).
Every time after the OMD update, if $p_{t+1,j}$ becomes too small and $1/p_{t+1,j}$ exceeds the threshold $\rho_{t,j}$, we increase the learning rate for $j$ by a factor of $e^{\frac{1}{\ln T}}$ and set the new threshold to be $\rho_{t+1,j}=2/p_{t+1,j}$ (\pref{line:increasing_eta}).
The high-level idea behind this technique is that when the probability of picking a clique is small and thus the corresponding Hedge receives little feedback, increasing the corresponding learning rate allows its faster recovery, should a node in the clique become the best node later. 
For more intuition, we refer the reader to~\citep{DBLP:journals/corr/AgarwalLNS16}. 
We are now ready to show the main theorem of this section.
\begin{theorem}\label{thm:kappaLstarThm}
    \pref{alg:kappalstaralg} with $c={64\beta}$ and $\eta\le \eta_{\text{max}}\triangleq \min\left\{\frac{1}{64\beta}, \frac{1}{1000(\ln T)\ln^2(KT)}\right\}$ guarantees:
    \begin{align*}
        \Reg\le\otil\left(\frac{\kappa\ln T+\ln K}{\eta}+\eta L_\star+\beta^2\ln T\right)
    \end{align*}
    for any strongly observable graph.
     Choosing $\eta=\min\left\{\eta_{\text{max}}, \sqrt{\frac{\kappa+1}{L_\star}}\right\}$ gives $\mathcal{\tilde{O}}\left(\sqrt{(\kappa +1)L_\star}+\beta^2\right).$
\end{theorem}

Our algorithm strictly improves over the \textsc{GREEN-IX-Graph} algorithm of~\citep{DBLP:journals/corr/abs-1711-03639} (in terms of expected regret), which achieves the same bound but only for undirected self-aware graphs. 
The proof of \pref{thm:kappaLstarThm} is deferred to \pref{app:kappaLstar}.
Moreover, in \pref{app:kappaLstarDoublingTrick} we also provide an adaptive version of our algorithm with a sophisticated doubling trick, which achieves the same bound but without the need of knowing $L_\star$ for learning rate tuning. We remark that doing doubling trick in such a two-level structure and with partial information is highly non-trivial.

\subsection{\texorpdfstring{$\mathcal{\tilde{O}}\big(\min\{\sqrt{\alpha T}, \sqrt{\kappa L_\star}\}\big)$}{} Regret Bound for Self-Aware Graphs}\label{sec: minAlphaKappa}

Although our bound $\otil(\sqrt{(\kappa+1) L_\star})$ could be much smaller than the worst-case bound $\otil(\sqrt{\alpha T})$, it is not always better since $\alpha \leq \kappa +1$.
To remedy this drawback, we propose another algorithm with $\mathcal{\tilde{O}}(\min\{\sqrt{\alpha T}, \sqrt{\kappa L^\star}\})$ regret for the special case of self-aware graphs.
The high-level idea is to first run an algorithm with $\otil(\sqrt{\kappa L_\star})$ regret
while keeping an estimate of $L_\star$,
and when we are confident that $\otil(\sqrt{\alpha T})$ is the better bound, switch to any algorithm with $\otil(\sqrt{\alpha T})$ regret.
%
%
For the first part, using \pref{alg:kappalstaralg} would create some technical issues and we are unable to analyze it unfortunately (otherwise we could have dealt with general strongly observable graphs). 
Instead, we introduce a new algorithm using a similar clipping technique of~\citep{DBLP:journals/corr/abs-1711-03639}.

We emphasize that the key challenge here is that the algorithm has to be adaptive in the sense that it does not need the knowledge of $L_\star$ --- otherwise, simply comparing the two bounds and running the corresponding algorithm with the better bound {solves} the problem already.
We again design a sophisticated doubling trick to resolve this issue.
All details are included in \pref{app:minAlphaKappa}.

\section{Weakly Observable Graphs}
\label{sec:weakly}

In this section, we consider small-loss bounds for weakly observable graphs, which have not been studied before.
Recall that the minimax regret bound in this case is $\tilde{\Theta}({\wklydn}^{\nicefrac{1}{3}}T^{\nicefrac{2}{3}})$ where $d$ is the {weak domination number}.
The most natural small-loss bound one would hope for is therefore $\tilde{\Theta}({\wklydn}^{\nicefrac{1}{3}}L_\star^{\nicefrac{2}{3}})$.
However, it turns out that this is not achievable, and in fact, no typical small-loss bounds are achievable for {\it any} weakly observable graph, as we prove in the following theorem.

\begin{theorem}\label{thm:lowerbound}
    For any weakly observable graph and any algorithm $\calA$ (without the knowledge of $L_\star$), if $\calA$ guarantees $\otil(1)$ regret when $L_\star = 0$ (ignoring dependence on $K$), then there exists a sequence of loss vectors such that the regret of $\calA$ is $\Omega(T^{1-\epsilon})$ for any $\epsilon \in (0, 1/3)$.
\end{theorem}

Note that this precludes small-loss bounds such as $\otil(L_\star^{\nicefrac{2}{3}})$, or even $\otil(\min\{L_\star^{a}, T^{\nicefrac{2}{3}}\})$ for any $a>0$.
The proof crucially relies on the fact that for a weakly observable graph, one can always find a pair of nodes $u$ and $v$ such that neither of them can observe $u$.
See \pref{app:LowerBound} for details.

Despite this negative result, nevertheless, we provide alternative first-order regret bounds in terms of the loss of some specific subset of nodes.
Specifically, ignoring dependence on other parameters, for the special case of directed complete bipartite graphs, we obtain regret $\mathcal{\tilde{O}}(\sqrt{L_\star})$ when $i^\star \in S$ and $\mathcal{\tilde{O}}(L_{i_S^\star}^{\nicefrac{2}{3}})$ otherwise, where $i_S^\star=\argmin_{i\in S}\sum_{t=1}^T\ell_{t,i}$ is the best node with a self-loop.
Moreover, for general weakly observable graphs, we achieve regret $\mathcal{\tilde{O}}(\sqrt{L_{\wklyds}})$ when $i^\star\in S$ and $\mathcal{\tilde{O}}(L_{\wklyds}^{\nicefrac{3}{4}})$ otherwise (or other different trade-offs between the two cases),
where $\wklyds$ is a weakly dominating set and $L_{\wklyds}\triangleq\frac{1}{|\calD|}\sum_{i\in \calD}\sum^T_{t=1}\ell_{t,i}$ is the average total loss of nodes in $\wklyds$.

Our algorithm is summarized in \pref{alg:adahome}.
The key algorithmic idea is inspired by the work of~\citep{pmlr-v75-wei18a}. 
They show that a variant of OMD with certain correction terms added to the loss estimators leads to a regret bound on $\Reg_i$ where the typical local-norm term $\|\hat{\ell}_t\|_{\nabla^{-2}\psi(p_t)}^2$ is replaced by a term that is {\it only in terms of the information of arm $i$}.
For our problem this is the key to achieve different orders of regret for different arms.
More specifically, the algorithm performs a standard OMD update, except that $\hat{\ell}_t$ is replaced by $\hat{\ell}_t + a_t$ for some correction terms $a_t$ (\pref{line:OMD_variant}).

\setcounter{AlgoLine}{0}
\begin{algorithm}[t]
\caption{Algorithm for Weakly Observable Graphs}\label{alg:adahome}
\textbf{Input:} Feedback graph $G$, decision set $\Omega$, parameter $\eta\le \frac{1}{5}$ and $\bar{\eta}$. 

\textbf{Define:} hybrid regularizer $\psi(p)=\frac{1}{\eta}\sum_{i\in S}\ln \frac{1}{p_i}+\frac{1}{\bar{\eta}} \sum_{i\in \bar{S}}p_i\ln p_i$.

\textbf{Initialize:} $p_1$ is such that $p_{1,i} = \frac{1}{2\slpn}$ for $i \in S$ and $p_{1,i} = \frac{1}{2\woslpn}$ for $i \in \bar{S}$.

\For{$t=1,2,\dots, T$}{
\nl
Play arm $i_t\sim {p_t}$ and receive feedback $\ell_{t,i}$ for all $i$ such that $i_t\in \Nin(i)$.\label{line:alg2bg}

\nl
Construct estimator $\hat{\ell}_t$ such that $\hat{\ell}_{t,i}=\frac{\ell_{t,i}}{W_{t,i}}\mathbbm{1}\{i_t\in \Nin(i)\}$ where $W_{t,i}=\sum_{j\in \Nin (i)}p_{t,j}$. 

\nl
Construct correction term $a_t$ such that 
$a_{t,i} =\begin{cases}
2\eta p_{t,i}\hat{\ell}_{t,i}^2, &\mbox{for $i\in S$},\\
2\bar{\eta}\hat{\ell}_{t,i}^2,  &\mbox{for $i\in \bar{S}$}.
\end{cases}
$
\label{line:correction}

\nl
Compute $p_{t+1}=\argmin_{p\in\Omega}\big\{\langle p,\hat{\ell}_{t}+a_{t}\rangle+{\brgmd}_{\psi}(p,p_{t})\big\}$.
\label{line:OMD_variant}
}
\end{algorithm}

Before specifying our correction term, we first describe the regularizer. 
Similar to the algorithms for strongly observable graphs, we again use a hybrid regularizer 
$\psi(p)=\frac{1}{\eta}\sum_{i\in S}\ln \frac{1}{p_i}+\frac{1}{\bar{\eta}} \sum_{i\in \bar{S}}p_i\ln p_i$, that is, log-barrier for nodes in $S$ and entropy for nodes in $\bar{S}$.
Note that we do not enforce a small amount of log-barrier for every node as in \pref{sec:strongly} (reasons to follow).
Also note that the learning rate for nodes in $S$ and $\bar{S}$ are different ($\eta$ and $\bar{\eta}$ respectively), which is also crucial for getting different orders of regret for different nodes.
In light of this specific choice of regularizer, our correction term $a_t$ is defined as in \pref{line:correction}, because $\eta p_{t,i}\hat{\ell}_{t,i}^2$ is the typical correction term for log-barrier~\citep{pmlr-v75-wei18a}, and on the other hand $\bar{\eta}\hat{\ell}_{t,i}^2$ is the typical one for entropy~\citep{steinhardt2014adaptivity}.
Mixing these two correction terms is novel as far as we know.

The estimator $\hat{\ell}_t$ is constructed exactly by \pref{eq:estimator}, and it remains to specify the decision set $\Omega \subseteq \Delta(K)$, which is different for different cases and will be discussed separately.
In both cases, the decision set is such that some relatively large amount of uniform exploration is enforced over a subset of nodes, which is also the reason why the small amount of log-barrier is not needed anymore.
Similar to the analysis of~\citep{pmlr-v75-wei18a}, we prove the following lemma (see \pref{app:LowerBound}).

\begin{lemma}\label{lem:lemadahome}
\pref{alg:adahome} ensures
$
\inner{p_t-u,\hat{\ell}_t}\le {\brgmd}_{\psi}(u,p_t)-{\brgmd}_{\psi}(u,p_{t+1})+\inner{u,a_t}
$ for all $u\in \Omega$,
as long as $\Omega$ is a subset of $\{p \in \Delta(K): \sum_{j\in \Nin (i)}p_{t,j} \geq 5\bar{\eta}, \;\forall i \in\bar{S}\}$.
\end{lemma}

Naturally, to compete with node $i$, we let $u$ almost concentrate on $i$, in which case $\inner{u,a_t}$ is roughly $a_{t,i}$ (only in terms of $i$; key difference compared to \pref{eq:unconstrained_shifting}).
To understand why this is useful, consider the case when $i \in S$ so $a_{t,i}$ is $\eta p_{t,i}\hat{\ell}_{t,i}^2$. 
By the construction of the loss estimator, the latter is bounded by $\eta \ell_{t,i}$ in expectation, which is the key of getting $\otil(\sqrt{T})$ regret in this case.

\paragraph{\BiGraph.}
For the special case of directed complete bipartite graphs, we take $\Omega = \left\{p\in \Delta(K):\sum_{i\in S}p_i\ge \sqrt{\bar{\eta}}\right\}$, which ensures that every node in $\bar{S}$ is observed with probability at least $\sqrt{\bar{\eta}}$ (by the definition of directed complete bipartite graphs).
This, however, unavoidably introduces dependence on $L_{i_S^\star}$ when bounding $\Reg_i$ for $i \in \bar{S}$, as shown below.

\begin{theorem}\label{thm:adahomethm}
For any directed complete bipartite graph, \pref{alg:adahome} with $\eta\le\frac{1}{5}$, $\bar{\eta}\le\frac{1}{25}$ and $\Omega = \left\{p\in \Delta(K):\sum_{i\in S}p_i\ge \sqrt{\bar{\eta}}\right\}$ guarantees:
    \begin{align*}
        \Reg_i\le \begin{cases} \displaystyle\frac{{\slpn}\ln T}{\eta}+2\eta L_{i} +2\slpn, &\mbox{for } i\in S,\\ 
\displaystyle\frac{{2\slpn}\ln T}{\eta}+\frac{2\ln K}{\bar{\eta}}+2\sqrt{\bar{\eta}}L_{\isstar}+2\sqrt{\bar{\eta}}L_i + 2\slpn, & \mbox{for } i\in \bar{S}, \end{cases}
    \end{align*}
    where $i^\star_S=\argmin_{i\in S}\sum_{t=1}^T\ell_{t,i}$. 
    Choosing $\eta=\min\left\{\sqrt{\frac{{\slpn}}{L_{\isstar}}}, \frac{1}{5}\right\}$ and $\bar{\eta}=\min\left\{L_{\isstar}^{-\nicefrac{2}{3}}, \frac{1}{25}\right\}$, we have:
$\Reg_i=\mathcal{\tilde{O}}\left(\sqrt{{\slpn}L_{i}}+{\slpn}\right)$ for $i\in S$ and $\Reg_i = \mathcal{\tilde{O}}\left(L_{\isstar}^{\nicefrac{2}{3}}+\sqrt{{\slpn}L_{\isstar}}+{\slpn}\right)$ for $i \in \bar{S}$.
\end{theorem}

Note that even though~\citet{pmlr-v40-Alon15} show that the worst-case regret of any weakly observable graph is $\Omega(T^{\nicefrac{2}{3}})$,
our result indicates that for directed complete bipartite graphs, one can in fact achieve much better regret of order $\otil(\sqrt{T})$ when the best node has a self-loop, while still maintaining the worst-case regret $\otil(T^{\nicefrac{2}{3}})$.
Moreover, in the former case, we can even achieve a typical small-loss bound, while in the latter case, the regret could be better than $\otil(T^{\nicefrac{2}{3}})$ as long as the best node in $S$ has sublinear total loss.
In \pref{app:adaptivebipartite}, we also provide an adaptive version of the algorithm without the need of knowing $L_i$ or $L_{\isstar}$ to tune learning rates (while maintaining the same bound), which requires a nontrivial combination of a clipping technique and doubling trick.

\paragraph{General Case.}
For general weakly observable graphs, following similar ideas of forcing the algorithm to observe nodes in $\bar{S}$ with a large enough probability, we take $\Omega = \left\{p\in \Delta(K):p_i\ge \delta,\forall i\in \wklyds \right\}$ where $\wklyds$ is a minimum weakly dominating set with size $d$ and $\delta$ is some parameter.
By definition, this ensures that each node in $\bar{S}$ is observed with probability at least $\delta$.
However, this also introduces dependence on $L_\wklyds$ for $\Reg_i$, even when $i \in S$, as shown in the following theorem.

\begin{theorem}\label{thm:adahome-general-thm}
For any weakly observable graph, \pref{alg:adahome} with $\frac{1}{T}\le\delta\le\min\left\{\frac{1}{125}, \frac{1}{4\slpn}, \frac{1}{4\wklydn}\right\}$, $\eta\le\frac{1}{25}$, $\bar{\eta}\le \delta^{\frac{4}{3}}$, and $\Omega = \left\{p\in \Delta(K):p_i\ge \delta,\forall i\in \wklyds \right\}$ guarantees:
\begin{align*}
\Reg_i\le \begin{cases} \displaystyle\frac{2\slpn\ln T}{\eta}+2\eta L_i + 2\delta\wklydn\LC + 2\slpn, &\mbox{for } i\in S,\\ 
\displaystyle\frac{\slpn\ln T}{\eta}+\frac{\ln(2 \woslpn)}{\bar{\eta}}+\frac{2\bar{\eta}L_i}{\delta} + 2\delta\wklydn\LC + 2\slpn, & \mbox{for } i\in \bar{S}. \end{cases}
\end{align*}
For any $\gamma\in [\frac{1}{3},\frac{1}{2}]$, setting $\delta=\min\left\{\frac{1}{125},\frac{1}{4\slpn},\frac{1}{4\wklydn}, \LC^{-\gamma}\right\}$, $\eta=\min\left\{\sqrt{\frac{1}{\LC}}, \frac{1}{25}\right\}$, and $\bar{\eta}=\min\left\{\sqrt{\frac{\delta}{\LC}},\delta^{\frac{4}{3}}\right\}$ gives 
$\Reg= \mathcal{\tilde{O}}\left(\LC^{1-\gamma}\right)$ if $i^\star\in S$ and 
$\mathcal{\tilde{O}}\left(\LC^{(1+\gamma)/2}\right)$ otherwise (ignoring dependence on $s$ and $d$).
\end{theorem}

Note that unlike the special case of directed complete bipartite graphs, we face a trade-off here when setting the parameters, due to the extra parameter $\delta$ that appears in both cases ($i \in S$ or $i\in \bar{S}$).
For example, when picking $\gamma=\frac{1}{3}$, we achieve $\Reg= \mathcal{\tilde{O}}\big(\LC^{\nicefrac{2}{3}}\big)$ always, better than the worst-case bound as long as $\LC$ is sublinear.
On the other hand, picking $\gamma=\frac{1}{2}$, we achieve $\Reg = \mathcal{\tilde{O}}\left(\sqrt{\LC}\right)$ when $i^\star \in S$ and $\mathcal{\tilde{O}}\big(\LC^{\nicefrac{3}{4}}\big)$ otherwise.
Once again, we provide an adaptive version in \pref{app:DoublingTrickGeneralCase}.




\acks{The authors are supported by NSF Award IIS-1755781 and thank 
S{\'e}bastien Bubeck, Akshay Krishnamurthy, Thodoris Lykouris, and Chen-Yu Wei for helpful discussions.}

\bibliography{ref}
\newpage

\appendix

\section*{Conclusions and Open Problems}
\label{conclusion}

In this work, we provide various new results on small-loss bounds for bandits with a directed feedback graph (either strongly observable or weakly observable), making a significant step towards a full understanding of this problem.

One clear open question is whether one can achieve $\otil(\sqrt{\alpha L_\star})$ regret for strongly observable graphs, which would be a strict improvement over the minimax regret $\otil(\sqrt{\alpha T})$.
Note again that our bound $\otil(\sqrt{(\kappa+1) L_\star})$ is weaker since $\alpha \leq \kappa+1$ always holds.
Achieving this ideal bound appears to require new ideas.
Even for the special case of self-aware graphs, the problem remains challenging and the closest result is the bound $\mathcal{\tilde{O}}(\alpha^{\nicefrac{1}{3}}{L_\star}^{\nicefrac{2}{3}})$ by~\citet{DBLP:journals/corr/abs-1711-03639}.
Another future direction is to generalize our results to time-varying feedback graphs, which also appears to require new ideas.

\section{Proofs for \pref{sec: remove uniform exploration}}
\label{app:alphaT}

In this section, we prove \pref{lem:unifyOMDregret} and \pref{thm:alphaTEXP3}.
To prove \pref{lem:unifyOMDregret}, we first show the following auxiliary lemma, which states that the OMD update enjoys multiplicative stability under certain conditions.

\begin{lemma}\label{lem:GeneralStability}
    Let $p_{t+1}=\argmin_{p\in \Omega}\left<p, \hat{\ell}_t\right>+D_{\psi}(p, p_t)$ for $\Omega \subseteq \Delta(K)$ and $\psi: \Omega \rightarrow \fR$ such that $\nabla^2\psi(p)\succeq \diag\left\{\frac{C_1}{p_{i}^2},\dots,\frac{C_1}{p_{K}^2}\right\}$ for some $C_1\ge 9$ and  $\nabla^{-2}\psi(p) \preceq 4\nabla^{-2}\psi(q)$ as long as $p \preceq 2q$. 
    If there exists $z\in \mathbb{R}$ such that $\|\hat{\ell}_t-z\cdot \one\|_{\nabla^{-2}\psi(p_t)}\le \frac{1}{8}$, then we have $\frac{1}{2}p_{t}\preceq p_{t+1}\preceq2p_{t}$.
\end{lemma}

\begin{proof}
The proof follows similar ideas of recent work such as~\citep{pmlr-v75-wei18a} or~\citep{bubeck2019improved}.
     Let $F_t(p)\triangleq\left<p, \hat{\ell}_t-z\cdot \one\right>+D_{\psi}(p,p_t)$, for $z$ satisfying the condition $\|\hat{\ell}_t-z\cdot \one\|_{\nabla^{-2}\psi(p_t)}\le \frac{1}{8}$. As we only shift each entry of the loss estimator by a constant, according to the algorithm, we have $p_{t+1}=\argmin_{p\in \Omega}F_t(p)$. We first show that $F_t(p')\ge F_t(p_t)$ for any $p'\in \Omega$ such that $\|p'-p_t\|_{\nabla^2\psi(p_t)}=1$. We start by applying Taylor expansion:
     \begin{align*}
         F_t(p') &= F_t(p_t)+\nabla F_t(p_t)^\top(p'-p_t)+\frac{1}{2}(p'-p_t)^\top\nabla^2 F_t(\xi)(p'-p_t) \\
         &= F_t(p_t)+\left(\hat{\ell}_t-z\cdot\one\right)^\top(p'-p_t)+\frac{1}{2}\|p'-p_t\|_{\nabla^2\psi(\xi)}^2 \\
         &\ge F_t(p_t)-\|\hat{\ell}_t-z\cdot\one\|_{\nabla^{-2}\psi(p_t)}\|p'-p_t\|_{\nabla^2\psi(p_t)}+\frac{1}{2}\|p'-p_t\|_{\nabla^2\psi(\xi)}^2 \\
         &= F_t(p_t)-\|\hat{\ell}_t-z\cdot\one\|_{\nabla^{-2}\psi(p_t)}+\frac{1}{2}\|p'-p_t\|_{\nabla^2\psi(\xi)}^2,
     \end{align*}
where the inequality is by H{\"o}lder's inequality and $\xi$ is some point on the line segment between $p_t$ and $p'$. By the condition $\nabla^2\psi(p)\succeq \diag\left\{\frac{9}{p_{i}^2},\dots,\frac{9}{p_{K}^2}\right\}$, we have $1=\|p'-p_t\|^2_{\nabla^2\psi(p_t)}\ge9\sum_{i\in [K]}\frac{(p_i'-p_{t,i})^2}{p_{t,i}^2}$, which implies $\frac{|p_i'-p_{t,i}|}{p_{t,i}}\le\frac{1}{3}$ for all $i\in [K]$. 
Therefore, we have $ \xi\preceq\frac{4}{3}p_{t}\preceq 2 p_{t}$, which leads to $\nabla^2\psi(\xi)\succeq\frac{1}{4}\nabla^2\psi(p_t)$ according to the assumption. Plugging it into the previous inequality, we have 
\begin{align*}
    F_t(p')-F_t(p_t)&\ge -\|\hat{\ell}_t-z\cdot\one\|_{\nabla^{-2}\psi(p_t)}+\frac{1}{2}\|p'-p_t\|_{\nabla^2\psi(\xi)}^2 \ge -\frac{1}{8}+\frac{1}{8}= 0.
\end{align*}
Therefore, according to the optimality of $p_{t+1}$ and the convexity of $F_t$, we have $\|p_{t+1}-p_t\|_{\nabla^2\psi(p_t)}\le 1$. Following the previous analysis, we further have:
\begin{align*}
    1\ge\|p_{t+1}-p_t\|_{\nabla^2\psi(p_t)}^2\ge 9\sum_{i\in [K]}\frac{(p_{t+1,i}-p_{t,i})^2}{p_{t,i}^2}\ge 9\frac{\left(p_{t+1,j}-p_{t,j}\right)^2}{p_{t,j}^2},\;\forall j\in[K].
\end{align*}
So we conclude $p_{t+1,i}\in \left[\frac{2}{3}p_{t,i},\frac{4}{3}p_{t,i}\right]\subseteq[\frac{1}{2}p_{t,i},2p_{t,i}]$ for all $i\in [K]$, finishing the proof.
\end{proof}


The next lemma further shows that the condition $\exists z, \|\hat{\ell}_t-z\cdot \one\|_{\nabla^{-2}\psi(p_t)}\le \frac{1}{8}$ is easily satisfied as long as $0\le \hat{\ell}_{t,i}\le\max\left\{\frac{1}{p_{t,i}},\frac{1}{1-p_{t,i}}\right\}$ for all $i$.

\begin{lemma}\label{lem:SpecialStability}
    If  $0\le \hat{\ell}_{t,i}\le\max\left\{\frac{1}{p_{t,i}},\frac{1}{1-p_{t,i}}\right\}$ for all $i\in[K]$, under the same conditions of \pref{lem:GeneralStability} with $C_1=\constK$, there exists $z\in \mathbb{R}$ such that $\|\hat{\ell}_t-z\cdot \one\|_{\nabla^{-2}\psi(p_t)}\le \frac{1}{8}$.
\end{lemma}

\begin{proof}
    If $p_{t,i}\le\frac{1}{2}$ for all $i\in[K]$, then we have $p_{t,i}\hat{\ell}_{t,i}\le \max\left\{1,\frac{p_{t,i}}{1-p_{t,i}}\right\}\le 1$ for all $i\in [K]$. In this case, $z=0$ satisfies:
\begin{align*}
    \|\hat{\ell}_{t} - z\cdot \one\|^2_{\nabla^{-2}\psi(p_{t})} &\le \sum_{i\in [K]} \frac{p_{t,i}^2\hat{\ell}_{t,i}^2}{C_1}\le \frac{K}{C_1} = \frac{1}{64}.
\end{align*}
On the other hand, if there is one node $i_{t,0}$ such that $p_{t,i_{t,0}}>\frac{1}{2}$, then $p_{t,i}\le\frac{1}{2}$ and $p_{t,i}\hat{\ell}_{t,i} \leq 1$ must be true for all $i \ne i_{t,0}$.
In this case picking $z=\hat{\ell}_{t,i_{t,0}}$ gives the following bound on $\|\hat{\ell}_{t}-z\cdot\one\|_{\nabla^{-2}\psi(p_{t})}$:
\begin{align*}
\left\|\hat{\ell}_{t}-\hat{\ell}_{t,i_{t,0}}\one\right\|_{\nabla^{-2}\psi(p_{t})}^2 &\le \frac{1}{C_1}\sum_{i\ne i_{t,0}} {p_{t,i}^2(\hat{\ell}_{t,i}-\hat{\ell}_{t,i_{t,0}})^2}\\
&\le \frac{1}{C_1}\sum_{i\ne i_{t,0}} \left({p_{t,i}^2\hat{\ell}_{t,i}^2}+ {p_{t,i}^2\hat{\ell}_{t,i_{t,0}}^2}\right)\\
&\le \frac{(K-1)}{C_1}+\frac{(1-p_{t,i_{t,0}})^2\hat{\ell}_{t,i_{t,0}}^2}{C_1} \tag{$\sum_{i\ne i_{t,0}} p_{t,i}^2 \leq (1-p_{t,i_{t,0}})^2$} \\
&\le\frac{1}{64} \tag{$0 \leq \hat{\ell}_{t,i_{t,0}}\le\frac{1}{1-p_{t,i_{t,0}}}$}.
\end{align*}
Combining the two cases finishes the proof.
\end{proof}

Now we are ready to prove \pref{lem:unifyOMDregret}.

\begin{proof}{\textbf{of \pref{lem:unifyOMDregret}}.} 
For any time step $t$ and any $z\in\mathbb{R}$, we first follow standard Online Mirror Descent analysis and show
\begin{equation}\label{eq:standard_OMD_bound}
\left<p_t-u, \hat{\ell}_t\right>\le D_{\psi}(u,p_t)-D_{\psi}(u,p_{t+1})+2\|\hat{\ell}_t-z\cdot\one\|_{\nabla^{-2}\psi(\xi)}^2.
\end{equation}
for some $\xi$ on the line segment of $p_t$ and $p_{t+1}$. 
Define $F_t(p) \triangleq \left<p, \hat{\ell}_t-z\cdot \one \right>+D_{\psi}\left(p,p_t\right)$.
As we only shift each entry of the loss estimator by a constant, according to the algorithm, we have $p_{t+1} = \argmin_{p\in \Omega} F_t(p)$ and thus by Taylor expansion,
it holds for some $\xi$ on the line segment of $p_t$ and $p_{t+1}$ that
\begin{align}
    F_t(p_t) - F_t(p_{t+1}) &= \nabla F_t(p_{t+1})^\top\left(p_t-p_{t+1}\right)+\frac{1}{2}\left(p_t-p_{t+1}\right)^\top\nabla^2F_t(\xi)\left(p_t-p_{t+1}\right) \nonumber\\
    &\ge \frac{1}{2}\|p_t-p_{t+1}\|^2_{\nabla^2\psi(\xi)}.\label{eq:lemm1.2}
\end{align}
On the other hand, by the non-negativity of Bregman divergence and H{\"o}lder's inequality, we have
\begin{align}
    F_t(p_t) - F_t(p_{t+1})&=\left<p_t-p_{t+1},\hat{\ell}_t-z\cdot \one\right>-D_{\psi}(p_{t+1},p_t)\nonumber \\
    &\le \left<p_t-p_{t+1},\hat{\ell}_t-z\cdot \one\right>\nonumber \\
    &\le \|p_t-p_{t+1}\|_{\nabla^2\psi(\xi)}\cdot\|\hat{\ell}_t-z\cdot \one\|_{\nabla^{-2}\psi(\xi)}.\label{eq:lemm1.1}
\end{align}
Combining \pref{eq:lemm1.2} and \pref{eq:lemm1.1}, we have $\|p_t-p_{t+1}\|_{\nabla^2\psi(\xi)}\le 2 \|\hat{\ell}_t-z\cdot \one\|_{\nabla^{-2}\psi(\xi)}$. 
Furthermore, standard analysis of Online Mirror Descent (see e.g. \citep[Lemma 6]{pmlr-v75-wei18a}) shows
\[ \left<p_t-u, \hat{\ell}_t\right>=\left<p_t-u, \hat{\ell}_t-z\cdot \one\right> \le D_{\psi}(u,p_t)-D_{\psi}(u,p_{t+1})+\left<p_t-p_{t+1},\hat{\ell}_{t}-z\cdot \one\right>,\]
which proves \pref{eq:standard_OMD_bound} after applying H{\"o}lder's inequality again and the previous conclusion $\|p_t-p_{t+1}\|_{\nabla^2\psi(\xi)}\le 2 \|\hat{\ell}_t-z\cdot \one\|_{\nabla^{-2}\psi(\xi)}$.

Finally, according to \pref{lem:SpecialStability}, we know that the conditions of \pref{lem:GeneralStability} hold, and thus multiplicative stability $\frac{1}{2}p_t \preceq p_{t+1} \preceq 2p_t$ holds, implying $\xi\preceq 2p_t$. 
By Condition (b) of the lemma statement, we have $\nabla^{-2}\psi(\xi)\preceq 4\nabla^{-2}\psi(p_t)$, which shows $\|\hat{\ell}_t-z\cdot\one\|_{\nabla^{-2}\psi(\xi)}^2 \leq 4\|\hat{\ell}_t-z\cdot\one\|_{\nabla^{-2}\psi(p_t)}^2$ and completes the proof as $z$ is arbitrary.
\end{proof}
Finally, we prove \pref{thm:alphaTEXP3}.
\begin{proof}{\textbf{of \pref{thm:alphaTEXP3}}.}
    \noindent According to the choice of $c$ and the construction of loss estimators, the conditions of \pref{lem:unifyOMDregret} hold and we have
    \begin{align*}
         \left<p_t-u, \hat{\ell}_t\right>\le D_{\psi}(u, p_t)-D_{\psi}(u, p_{t+1})+8\min_{z\in \fR}\|\hat{\ell}_t-z\cdot \one\|_{\nabla^{-2}\psi(p_t)}^2.
     \end{align*}
To bound the local-norm term $\min_{z\in \fR}\|\hat{\ell}_t-z\cdot \one\|_{\nabla^{-2}\psi(p_t)}^2$, one could follow the analysis of~\citep{pmlr-v40-Alon15}.
However, to be consistent with other proofs in this work, we provide a different analysis based on a novel loss shift (that is critical for all other proofs).
Specifically, we consider two cases. First, if $p_{t,i}<\frac{1}{2}$ holds for all $i\in\bar{S}$, then we relax the local-norm term by taking $z=0$:
     \begin{align*}
         \min_{z\in \fR}\|\hat{\ell}_t-z\cdot \one\|_{\nabla^{-2}\psi(p_t)}^2 &\leq \|\hat{\ell}_t\|_{\nabla^{-2}\psi(p_t)}^2
         =\sum_{i\in [K]}\frac{1}{\nicefrac{c}{p_{t,i}^2}+\nicefrac{1}{\eta p_{t,i}}}\hat{\ell}_{t,i}^2\\
         &\le \sum_{i\in [K]}\eta p_{t,i}\hat{\ell}_{t,i}^2
         \le \sum_{i\in S}\eta p_{t,i}\hat{\ell}_{t,i}^2+2\sum_{i\in \bar{S}}\eta p_{t,i}\hat{\ell}_{t,i},
     \end{align*}
     where the last step is because $\hat{\ell}_{t,i}\le \frac{1}{1-p_{t,i}}\leq 2$ for $i\in \bar{S}$. On the other hand, if there exists $i_{t,0}\in \bar{S}$ such that $p_{t,i_{t,0}}\ge\frac{1}{2}$, then we take $z=\hat{\ell}_{t,i_{t,0}}$ and arrive at:
     \begin{align*}
         \min_{z\in \fR}\|\hat{\ell}_t-z\cdot \one\|_{\nabla^{-2}\psi(p_t)}^2 &\le \|\hat{\ell}_t-\hat{\ell}_{t,i_{t,0}}\cdot\one\|_{\nabla^{-2}\psi(p_t)}^2 \\
         &\le \sum_{i\ne i_{t,0}}\eta p_{t,i}\left(\hat{\ell}_{t,i}-\hat{\ell}_{t,i_{t,0}}\right)^2\\
         &\le \sum_{i\ne i_{t,0}}\eta p_{t,i}\left(\hat{\ell}_{t,i}^2+\hat{\ell}_{t,i_{t,0}}^2\right) \\
     &= \sum_{i\in S}\eta p_{t,i}\hat{\ell}_{t,i}^2+\sum_{i\in \bar{S},i\ne i_{t,0}}\eta p_{t,i}\hat{\ell}_{t,i}^2 + \sum_{i\in[K], i\ne i_{t,0}}\eta p_{t,i}\hat{\ell}_{t,i_{t,0}}^2 \\
      &\le \sum_{i\in S}\eta p_{t,i}\hat{\ell}_{t,i}^2+2\sum_{i\in \bar{S},i\ne i_{t,0}}\eta p_{t,i}\hat{\ell}_{t,i} + \sum_{i\in[K], i\ne i_{t,0}} \eta p_{t,i}\hat{\ell}_{t,i_{t,0}}^2 \\
     &= \sum_{i\in S}\eta p_{t,i}\hat{\ell}_{t,i}^2+2\sum_{i\in \bar{S},i\ne i_{t,0}}\eta p_{t,i}\hat{\ell}_{t,i} + \eta\left(1-p_{t,i_{t,0}}\right)\hat{\ell}_{t,i_{t,0}}^2 \\
     &\le \sum_{i\in S}\eta p_{t,i}\hat{\ell}_{t,i}^2+2\sum_{i\in \bar{S}}\eta p_{t,i}\hat{\ell}_{t,i},
     \end{align*}
     where the second to last inequality is because $\hat{\ell}_{t,i} \le \frac{1}{1-p_{t,i}} \le 2$ for $i\in \bar{S} \setminus \{i_{t,0}\}$ and the final inequality is because $(1-p_{t,i_{t,0}})\hat{\ell}_{t,i_{t,0}}\le \frac{1-p_{t,i_{t,0}}}{1-p_{t,i_{t.0}}}=1 \leq 2p_{t,i_{t,0}}$.
    Therefore,  combining the two cases we have shown:
 \begin{align*}
     \left<p_t-u, \hat{\ell}_t\right>\le D_{\psi}(u,p_t)-D_{\psi}(u,p_{t+1})+{8\eta}\sum_{i\in S}p_{t,i}\hat{\ell}_{t,i}^2+16\eta\sum_{i\in \bar{S}}p_{t,i}\hat{\ell}_{t,i}.
 \end{align*}
 Summing over $t$ and telescoping, we further have:
  \begin{align*}
     \sum_{t=1}^T\left<p_t-u, \hat{\ell}_t\right>&\le D_{\psi}(u,p_1)+{8\eta}\sum_{t=1}^T\sum_{i\in S}p_{t,i}\hat{\ell}_{t,i}^2+16\eta\sum_{t=1}^T\sum_{i\in \bar{S}}p_{t,i}\hat{\ell}_{t,i}.\\
     &\le D_\psi(u,p_1)+{8\eta}\sum_{t=1}^T\sum_{i\in S} \frac{p_{t,i}}{W_{t,i}}\hat{\ell}_{t,i}+16\eta\sum_{t=1}^T\sum_{i\in \bar{S}}p_{t,i}\hat{\ell}_{t,i}.
 \end{align*}
We choose $u = \left(1-\frac{K}{T}\right)e_{i^\star}+\frac{1}{T}\cdot\one$.
By the optimality of $p_1$, we bound the Bregman divergence term as:
\begin{align*}
D_{\psi}(u, p_1) &\leq \psi(u) - \psi(p_1) \\
&\leq \frac{1}{\eta}\sum_{i\in [K]}p_{1,i}\ln\frac{1}{p_{1,i}} + c\sum_{i\in[K]}\ln\frac{1}{u_i}
\leq \frac{\ln K}{\eta} + cK\ln T.
\end{align*}
Comparing $u$ and $e_{i^\star}$, we bound $\sum_{t=1}^T\left<p_t-e_{i^\star}, \hat{\ell}_t\right>$ by
 \begin{align*}
 \frac{\ln K}{\eta}+cK\ln T+{8\eta}\sum_{t=1}^T\sum_{i\in S}\frac{p_{t,i}}{W_{t,i}}\hat{\ell}_{t,i}+16\eta\sum_{t=1}^T\sum_{i\in \bar{S}}p_{t,i}\hat{\ell}_{t,i}+\frac{1}{T}\sum_{t=1}^T\sum_{i\in [K]}\hat{\ell}_{t,i}.
 \end{align*}
 Taking expectation over both sides, we arrive at:
 \begin{align*}
     \Reg & \le \frac{\ln K}{\eta}+cK\ln T+\E\left[{8\eta}\sum_{t=1}^T\sum_{i\in S}\frac{p_{t,i}}{W_{t,i}}\ell_{t,i}+16\eta\sum_{t=1}^T\sum_{i\in \bar{S}}p_{t,i}\ell_{t,i}\right]+\frac{1}{T}\sum_{t=1}^T\sum_{i\in [K]}\ell_{t,i} \\
     &\le \frac{\ln K}{\eta}+cK\ln T+32\eta\alpha T\ln\left(\frac{4KT}{\alpha}\right)+16\eta T+K = \otil\left(\frac{1}{\eta}+ \eta \alpha T + K^2 \right), 
 \end{align*}
where the last inequality uses the fact $\ell_{t,i}\le 1$ and also a graph-theoretic lemma~\citep[Lemma~5]{pmlr-v40-Alon15} which asserts $\sum_{i\in S}\frac{p_{t,i}}{W_{t,i}} \leq 4\alpha\ln\left(\frac{4KT}{\alpha}\right)$. 
This finishes the proof. 
\end{proof}

\section{Proofs for \pref{sec: betaLstar}}
\label{app:betaLstar}

We prove \pref{thm:thmbetalstar} in this section.
\begin{proof}{\textbf{of \pref{thm:thmbetalstar}}.}
     Similar to the proof of \pref{thm:alphaTEXP3}, the conditions of \pref{lem:unifyOMDregret} hold and we have
    \begin{align*}
         \left<p_t-u, \hat{\ell}_t\right>\le D_{\psi}(u, p_t)-D_{\psi}(u, p_{t+1})+8\min_{z\in \fR}\|\hat{\ell}_t-z\cdot \one\|_{\nabla^{-2}\psi(p_t)}^2.
     \end{align*}
     Once again, we bound the local-norm term by considering two cases separately. \\
     (i). If $p_{t,i}<\frac{1}{2}$ holds for all $i\in \bar{S}$, then choosing $z=0$ we have:
     \begin{align*}
         &\min_{z\in \fR}\|\hat{\ell}_t-z\cdot \one\|_{\nabla^{-2}\psi(p_t)}^2 \le \|\hat{\ell}_t\|_{\nabla^{-2}\psi(p_t)}^2\\
         &=\sum_{i\in S}\eta  p_{t,i}^2\hat{\ell}_{t,i}^2+\sum_{i\in \bar{S}}\frac{1}{\nicefrac{c}{p_{t,i}^2}+\nicefrac{1}{\eta p_{t,i}}}\hat{\ell}_{t,i}^2\\
         &\le \sum_{i\in S}\eta  p_{t,i}^2\hat{\ell}_{t,i}^2+\sum_{i\in \bar{S}}\eta  p_{t,i}\hat{\ell}_{t,i}^2\\
         &\le \sum_{i\in S}\eta  p_{t,i}\hat{\ell}_{t,i}+2\sum_{i\in \bar{S}}\eta  p_{t,i}\hat{\ell}_{t,i}\\
         &\le 2\eta\inner{p_t, \hat{\ell}_t}.
     \end{align*}
     The third inequality is because $p_{t,i}\hat{\ell}_{t,i}\le 1$ for $i\in S$ and $\hat{\ell}_{t,i}\le\frac{1}{1-p_{t,i}}\le 2$ for $i\in \bar{S}$.\\
     (ii). If there exists $\exists i_{t,0}\in \bar{S}$ such that $p_{t,i_{t,0}}\ge\frac{1}{2}$, the choosing $z=\hat{\ell}_{t,i_{t,0}}$ we have:
     \begin{align*}
         \min_{z\in \fR}\|\hat{\ell}_t-z\cdot \one\|_{\nabla^{-2}\psi(p_t)}^2  & \le \|\hat{\ell}_t-\hat{\ell}_{t,i_{t,0}}\cdot\one\|_{\nabla^{-2}\psi(p_t)}^2 \\
         &= \sum_{i\in S}\eta p_{t,i}^2\left(\hat{\ell}_{t,i}-\hat{\ell}_{t,i_{t,0}}\right)^2+\sum_{i\in \bar{S}, i \ne i_{t,0}}\frac{1}{\nicefrac{c}{p_{t,i}^2}+\nicefrac{1}{\eta p_{t,i}}}\left(\hat{\ell}_{t,i}-\hat{\ell}_{t,i_{t,0}}\right)^2 \\
         &\le \sum_{i\in S}\eta p_{t,i}^2\left(\hat{\ell}_{t,i}^2+\hat{\ell}_{t,i_{t,0}}^2\right)+\sum_{i\in \bar{S}, i \ne i_{t,0}}\eta p_{t,i}\left(\hat{\ell}_{t,i}^2+\hat{\ell}_{t,i_{t,0}}^2\right)\\
     &= \sum_{i\in S}\eta p_{t,i}^2\hat{\ell}_{t,i}^2+ \sum_{i\in \bar{S},i\ne i_{t,0}}\eta p_{t,i}\hat{\ell}_{t,i}^2 +  \sum_{i\ne i_{t,0}}\eta p_{t,i}\hat{\ell}_{t,i_{t,0}}^2\\
      &= \sum_{i\in S}\eta p_{t,i}^2\hat{\ell}_{t,i}^2+ \sum_{i\in \bar{S},i\ne i_{t,0}}\eta p_{t,i}\hat{\ell}_{t,i}^2 + \eta \left(1-p_{t,i_{t,0}}\right)\hat{\ell}_{t,i_{t,0}}^2 \\
     &\le \sum_{i\in S}\eta p_{t,i}\hat{\ell}_{t,i}+2 \sum_{i\in \bar{S},i\ne i_{t,0}}\eta p_{t,i}\hat{\ell}_{t,i}+2\eta p_{t,i_{t,0}}\hat{\ell}_{t,i_{t,0}}\\
     &\le 2\eta\inner{p_t, \hat{\ell}_t}.
     \end{align*}
     The second to last inequality is because $p_{t,i}\hat{\ell}_{t,i}\le 1$ for $i\in S$, $\hat{\ell}_{t,i}\le\frac{1}{1-p_{t,i}}\le 2$ for $i\in \bar{S}\backslash\{i_{t,0}\}$, and $(1-p_{t,i_{t,0}})\hat{\ell}_{t,i_{t,0}}\le 1 \le 2p_{t,i_{t,0}}$.
     
     Combining the two cases, we have
 \begin{align*}
     \inner{p_t-u, \hat{\ell}_t}\le D_{\psi}(u,p_t)-D_{\psi}(u,p_{t+1})+16\eta\inner{p_t,\hat{\ell}_t},
 \end{align*}
 and summing over $t$ and telescoping, we further have 
  \begin{align*}
     \sum_{t=1}^T\inner{p_t-u, \hat{\ell}_t}\le D_{\psi}(u,p_1)+16\eta\sum_{t=1}^T\inner{p_t,\hat{\ell}_t}.
 \end{align*}
 We choose $u = \left(1-\frac{K}{T}\right)e_{i^\star}+\frac{1}{T}\cdot\mathbf{1}$ and 
 calculate the Bregman divergence term as
 \begin{align*}
 D_{\psi}(u, p_1) &\leq \psi(u) - \psi(p_1) \tag{by optimality of $p_1$}\\
&\leq \frac{1}{\eta}\sum_{i\in S}\ln\frac{1}{u_{i}} + \frac{1}{\eta}\sum_{i\in \bar{S}}p_{1,i}\ln\frac{1}{p_{1,i}} + c\sum_{i\in\bar{S}}\ln\frac{1}{u_i} 
\leq \frac{\slpn\ln T}{\eta} + \frac{\ln K}{\eta} + cK\ln T.
 \end{align*}
Comparing the difference between $u$ and $e_{i^\star}$ and rearranging, we arrive at:
 \begin{align*}
     \sum_{t=1}^T\inner{p_t-e_{i^\star},\hat{\ell}_t}\le\frac{1}{1-16\eta}\left(\frac{\slpn\ln T+\ln K}{\eta}+cK\ln T +\frac{1}{T}\sum_{t=1}^T\sum_{i\in [K]}\hat{\ell}_{t,i}+16\eta\sum_{t=1}^T\hat{\ell}_{t,i^\star}\right).
 \end{align*}
Taking expectation on both sides shows
 \begin{align*}
     \Reg &\le\frac{1}{1-16\eta}\left(\frac{\slpn\ln T+\ln K}{\eta}+cK\ln T+K+16\eta L_\star\right) \\
     &= \mathcal{O}\left(\frac{{\slpn}\ln T+\ln K}{\eta}+\eta L_\star+K^2\ln T\right),
 \end{align*}
finishing the proof.
\end{proof}

\section{Omitted details for \pref{sec: kappaLstar}}
\label{app:kappaLstar}

In this section, we provide omitted details for \pref{sec: kappaLstar}, including the adaptive Hedge subroutine used in \pref{alg:kappalstaralg} and its regret bound (\pref{app:adaptivehedge}), the proof of \pref{thm:kappaLstarThm} (\pref{app:kappaLstarThm}), and an adaptive version of \pref{alg:kappalstaralg} and its analysis (\pref{app:kappaLstarDoublingTrick}).

\subsection{Hedge with Adaptive Learning Rates}
\label{app:adaptivehedge}

We first provide details of the Hedge variant used in \pref{alg:kappalstaralg}. \pref{alg:AdaHedge} shows the complete pseudocode. 
Note that as described in \pref{sec: kappaLstar}, each Hedge instance only operates over a subset of arms, denoted by $C$ as an input of the algorithm. 
At each time $t$, the algorithm proposes a distribution $\widetilde{p}_t$, and then receives a loss vector $\widetilde{\ell}_{t}\in \mathbb{R}_+^K$.

Vanilla Hedge is simply OMD with the entropy regularizer over the simplex.
Our variant makes the following two modifications.
First, the decision set $\Omega$ is restricted to a subset of simplex so that zero probability is assigned to arms outside $C$ and at least $\frac{1}{|C|T}$ probability is assigned to each arm in $C$ for exploration purpose.
Second, we apply an adaptive time-varying learning rate as specified in \pref{line:alg3.2}.
This adaptive learning rate schedule ensures an adaptive regret bound (which is important for our analysis), as shown in the following lemma.


\setcounter{AlgoLine}{0}
\begin{algorithm}[t]\caption{Hedge with Adaptive Learning Rates}\label{alg:AdaHedge}
\textbf{Input:} The number of arms $K$, the set of active arms $C\subseteq [K]$.

\textbf{Define:} $\Omega=\left\{p\in \Delta(K):p_i\ge \frac{1}{|C|T}, \forall i\in C, \;\; \text{and}\;\; p_{i}=0, \forall i\notin C\right\}$

\textbf{Initialize:} $\widetilde{p}_1$ is the uniform distribution over $C$.

\For{$t=1,2,\dots, T$}{
\nl Propose distribution $\widetilde{p}_t$.

\nl Receive feedback $\widetilde{\ell}_{t}\in \mathbb{R}_+^K$. \label{line:alg3.1}

\nl Compute $\widetilde{p}_{t+1}=\argmin_{p\in \Omega}\left\{\inner{p, \widetilde{\ell}_t}+D_{\psi_t}\left(p, \widetilde{p}_t\right)\right\}$, where 
\[
\psi_t(p)=\frac{1}{
    	\eta_t}\sum_{i\in [K]}p_i\ln p_i, \quad\text{with}\; \eta_t=\sqrt{\frac{1}{1+\sum_{\tau=1}^{t}\sum_{i=1}^{K}\widetilde{p}_{\tau,i}\widetilde{\ell}_{\tau,i}^2}}.
\]
\label{line:alg3.2}
}
\end{algorithm}

\begin{lemma}\label{lem:HedgeLemma}
	\pref{alg:AdaHedge} ensures that for any $i\in C$, we have
	\begin{align}\label{eq:Sub}
	\sum_{t=1}^T\inner{\widetilde{p}_{t}-e_{i}, \widetilde{\ell}_{t}}\le 25\rho\ln^2 (KT) +10\ln (KT)\sqrt{\rho \sum_{t=1}^T\widetilde{\ell}_{t,i}},
	\end{align}
	where $\rho = \max\left\{1, \max_{t\in[T],i\in C}\widetilde{\ell}_{t,i}\right\}$.
\end{lemma}

\begin{proof}
Let $q_{t+1, i} = \widetilde{p}_{t,i} \exp(-\eta_t \widetilde{\ell}_{t,i})$.
One can verify $\widetilde{p}_{t+1} = \argmin_{p \in \Omega} D_{\psi_t}(p, q_{t+1})$ and also for any $u\in \Omega$,
\begin{align*}
\inner{\widetilde{p}_t-u, \widetilde{\ell}_t} &=  D_{\psi_{t}}(u, \widetilde{p}_t)-D_{\psi_{t}}(u, q_{t+1}) + D_{\psi_{t}}(\widetilde{p}_t, q_{t+1}) \\
&\leq D_{\psi_{t}}(u, \widetilde{p}_t)-D_{\psi_{t}}(u, \widetilde{p}_{t+1}) + D_{\psi_{t}}(\widetilde{p}_t, q_{t+1}),
\end{align*}
where the second step uses the generalized Pythagorean theorem.
On the other hand, using the fact $\exp(-x) \leq 1 - x + x^2$ for any $x \geq 0$, we also have
\begin{align*}
D_{\psi_{t}}(\widetilde{p}_t, q_{t+1}) &= \frac{1}{\eta_t}\sum_{i\in [K]} \left( \widetilde{p}_{t,i} \ln \frac{\widetilde{p}_{t,i}}{q_{t+1,i}} + q_{t+1,i} - \widetilde{p}_{t,i}\right) \\
&= \frac{1}{\eta_t}\sum_{i\in [K]} \widetilde{p}_{t,i}\left(\exp(-\eta_t \widetilde{\ell}_{t,i}) - 1 + \eta_t \widetilde{\ell}_{t,i} \right)
\leq \eta_t \sum_{i\in [K]} \widetilde{p}_{t,i} \widetilde{\ell}_{t,i}^2.
\end{align*}
Summing over $t$ we have shown
	\begin{align*}
	\sum_{t=1}^T\inner{\widetilde{p}_t-u, \widetilde{\ell}_t} &\le \sum_{t=1}^T\left(D_{\psi_{t}}(u, \widetilde{p}_t)-D_{\psi_{t}}(u, \widetilde{p}_{t+1})\right)+\sum_{t=1}^T\eta_t\sum_{i=1}^K\widetilde{p}_{t,i}\widetilde{\ell}_{t,i}^2 \\
	&\le \KL(u || \widetilde{p}_1)+\sum_{t=1}^{T-1}\left(\frac{1}{\eta_{t+1}}-\frac{1}{\eta_t}\right)\KL(u || \widetilde{p}_{t+1}) + \sum_{t=1}^T\eta_t\sum_{i=1}^K\widetilde{p}_{t,i}\widetilde{\ell}_{t,i}^2\\
	&\le \KL(u || \widetilde{p}_1)+\frac{\max_{p\in \Omega}\KL(u || p)}{\eta_T} + \sum_{t=1}^T\eta_t\sum_{i=1}^K\widetilde{p}_{t,i}\widetilde{\ell}_{t,i}^2 \\
	& \le \ln K +\frac{\ln(KT)}{\eta_T} + \sum_{t=1}^T\sum_{i=1}^K\frac{\widetilde{p}_{t,i}\widetilde{\ell}_{t,i}^2}{\sqrt{1+\sum_{\tau=1}^t\sum_{i=1}^K\widetilde{p}_{\tau,i}\widetilde{\ell}_{\tau,i}^2}} \\
	& \le \ln K +\frac{\ln(KT)}{\eta_T} + \int_{0}^{\sum_{t=1}^T\sum_{i=1}^K\widetilde{p}_{t,i}\widetilde{\ell}_{t,i}^2}\frac{1}{\sqrt{x+1}}dx \\
	& \le \ln K +\frac{\ln(KT)}{\eta_T} + 2\sqrt{1+\sum_{t=1}^T\sum_{i=1}^K\widetilde{p}_{t,i}\widetilde{\ell}_{t,i}^2} \\
	&=  \ln K + (\ln(KT)+2)\sqrt{1+\sum_{t=1}^T\sum_{i=1}^K\widetilde{p}_{t,i}\widetilde{\ell}_{t,i}^2}.
	\end{align*}
	Now choosing $u =\left(1-\frac{1}{T}\right)e_{i}+\frac{1}{|C|T}\cdot\one_C \in \Omega$ where $\one_C$ is the vector with one for coordinates in $C$ and zero otherwise, we have
	\begin{align*}
	\sum_{t=1}^T\inner{\widetilde{p}_t-e_i, \widetilde{\ell}_t} &\le  \ln K+(\ln (KT)+2)\sqrt{1+\sum_{t=1}^T\sum_{i=1}^K\widetilde{p}_{t,i}\widetilde{\ell}_{t,i}^2} + \frac{1}{|C|T}\sum_{t=1}^T\sum_{i\in C}\widetilde{\ell}_{t,i} \\
	&\le 4\ln (KT) + 3\ln (KT)\sqrt{\sum_{t=1}^T\sum_{i=1}^K\widetilde{p}_{t,i}\widetilde{\ell}_{t,i}^2}+\frac{1}{|C|T}\sum_{t=1}^T\sum_{i\in C}\widetilde{\ell}_{t,i} \\
	&\le 4\ln (KT) + 3\ln (KT) \sqrt{\rho\sum_{t=1}^T\inner{\widetilde{p}_t, \widetilde{\ell}_t}} +\rho.
	\end{align*}
	Let $\widetilde{L}_T\triangleq\sum_{t=1}^T\inner{\widetilde{p}_t, \widetilde{\ell}_t}$ and $\widetilde{L}_{T,i}\triangleq\sum_{t=1}^T\widetilde{\ell}_{t,i}$. By solving the quadratic inequality, we have
	\begin{align*}
	\sqrt{\widetilde{L}_T}&\le \frac{3\ln(KT) \sqrt{\rho}+\sqrt{9\ln^2 (KT)\rho+4\cdot( 4\ln (KT) + \rho+\widetilde{L}_{T,i}) }}{2}\\
	&\le 5\ln (KT)\sqrt{\rho}+\sqrt{\widetilde{L}_{T,i}}.
	\end{align*}
	Finally, squaring both sides proves
$
	\widetilde{L}_T -\widetilde{L}_{T,i}\le 25\ln^2(KT)\rho+ 10\ln (KT)\sqrt{\rho\widetilde{L}_{T,i}}.
$
\end{proof}

\subsection{Proofs of \pref{thm:kappaLstarThm}}\label{app:kappaLstarThm}

To prove \pref{thm:kappaLstarThm}, we combine the regret bounds of the meta-algorithm and the Hedge subroutine. For the former, we prove the following lemma, which combines the result of \pref{thm:thmbetalstar} and the effect of the increasing learning rate schedule proposed in~\citep{DBLP:journals/corr/AgarwalLNS16}, leading to an important negative regret term.

\begin{lemma}\label{lem:CorralLemma}
	\pref{alg:kappalstaralg} with $c=64\beta$ and $\eta \leq \frac{1}{64\beta}$ ensures that for any $j\in [\beta]$, 
	\begin{equation}\label{eq:Meta}
	\begin{split}
	\sum_{t=1}^T\inner{p_t-e_j,\hat{\ell}_{t}}&\le \mathcal{O}\left(\frac{\kappa \ln T+\ln K}{\eta}+\beta^2\ln T\right) + 80	\eta\sum_{t=1}^T\inner{p_t, \hat{\ell}_t}  \\
	&\quad\quad + \frac{1}{T}\sum_{t=1}^T\sum_{j\in[\beta]}\hat{\ell}_{t,j} -\frac{\rho_{T,j}}{20\eta\ln T}\mathbbm{1}\left\{j\in [\kappa]\right\}.
	\end{split}
	\end{equation}
\end{lemma}

\begin{proof}
We first show that according to our increasing learning rate schedule, the final learning rate is upper bounded by a constant times the original learning rate. Fix a node $j\in [\kappa]$. Let $n_j$ be such that $\eta_{T,j}=\sigma^{n_j}\eta_{1,j}$ with $\sigma = e^{\frac{1}{\ln T}}$, where we assume $n_j\ge 1$ (the case $n_j=0$ is trivial as one will see). Let $t_1,...,t_{n_j}$ be the rounds in which the learning rate update is executed for node $j$. Since $\frac{1}{p_{t_{n_j}+1,j}}>\rho_{t_{n_j},j}>2\rho_{t_{n_j-1},j}>...>2^{n_j-1}\rho_{1,j}=2^{n_j}\kappa$ and $\frac{1}{p_{t_{n_j}+1,1}}\le T$, we have $n_j\le \log_2 T$. Therefore, we have $\eta_{T,j}\le \sigma^{\log_2 T}\eta_{1,j}\le 5\eta_{1,j}=5\eta$.
	
Next, according to our choice of $c$ and $\eta$, the conditions of \pref{lem:unifyOMDregret} hold and we have
	\begin{align*}
	\inner{p_t-u, \hat{\ell}_t}\le D_{\psi_t}(u,p_t)-D_{\psi_t}(u,p_{t+1})+8\min_z\|\hat{\ell}_t-z\cdot \mathbf{1}\|_{\nabla^{-2}\psi_t(p_t)}^2.
	\end{align*}
	We consider the Bregman divergence terms and choose $u=\left(1-\frac{K}{T}\right)e_j+\frac{1}{T}\mathbf{1}$. If $j\in [\kappa]$, then with $h(y) = y - 1 - \ln y$ we have
	\begin{align*}
	\sum_{t=1}^T D_{\psi_t}(u, p_t)-D_{\psi_t}(u,p_{t+1}) &\le D_{\psi_1}(u, p_1)+\sum_{t=1}^{T-1}\left(D_{\psi_{t+1}}(u, p_{t+1})-D_{\psi_t}(u, p_{t+1})\right)\\
	&\le D_{\psi_1}(u,p_1)+\left(\frac{1}{\eta_{t_{n_j}+1,j}}-\frac{1}{\eta_{t_{n_j},j}}\right)h\left(\frac{u_j}{p_{t_{n_j}+1,j}}\right) \\
	&= D_{\psi_1}(u,p_1)+\frac{1-\sigma}{\sigma^{n_j}\eta} h\left(\frac{u_j}{p_{t_{n_j}+1,j}}\right) \\
	&\le \frac{\kappa \ln T + \ln K}{\eta} + c\beta \ln T -\frac{1}{5\eta\ln T}h\left(\frac{u_j}{p_{t_{n_j}+1,j}}\right) ,
	\end{align*}
	where we use the facts $1 - \sigma \leq -\frac{1}{\ln T}$ and $\sigma^{n_j}\leq 5$ as shown earlier, and also the exact same analysis of bounding $D_{\psi_1}(u,p_1)$ as in the proof of \pref{thm:thmbetalstar}.  
	Note that $\frac{u_j}{p_{t_{n_j}+1,j}}\ge\frac{1}{2p_{t_{n_j}+1,j}}\ge 2^{n_j-1}\kappa\ge 1$. Combining the facts that $h(y)$ is increasing when $y\ge 1$ and $\rho_{T,j}=\frac{2}{p_{t_{n_j}+1,j}} \leq 2T$, we have:
	\[h\left(\frac{u_j}{p_{t_{n_j}+1,j}}\right)\ge h\left(\frac{1}{2p_{t_{n_j}+1,j}}\right)=\frac{\rho_{T,j}}{4}-1-\ln\left(\frac{\rho_{T,j}}{4}\right)\ge \frac{\rho_{T,j}}{4}-2\ln T.\]
	We have thus shown when $j \in [\kappa]$,
	\begin{align*}
	\sum_{t=1}^T D_{\psi_t}(u, p_t)-D_{\psi_t}(u,p_{t+1})\le\mathcal{O}\left(\frac{\kappa\ln T+\ln K}{\eta}+c\beta\ln T\right)-\frac{\rho_{T,j}}{20\eta\ln T}
	\end{align*}
	On the other hand, if $j\in \bar{S}$, then we have by the monotonicity of learning rates
	\begin{align*}
	\sum_{t=1}^T D_{\psi_t}(u, p_t)-D_{\psi_t}(u,p_{t+1})
	&\le D_{\psi_1}(u, p_1)+\sum_{t=1}^{T-1}\left(D_{\psi_{t+1}}(u, p_{t+1})-D_{\psi_t}(u, p_{t+1})\right) \\
	&\le D_{\psi_1}(u, p_1) \le \frac{\kappa\ln T+\ln K}{\eta}+c\beta\ln T.
	\end{align*}
	It remains to deal with the local-norm term $\min_z\|\hat{\ell}_{t}-z\cdot \mathbf{1}\|_{\nabla^{-2}\psi_t(p_t)}^2$. Following the exact analysis in the proof of \pref{thm:thmbetalstar} and the fact $\eta_{t,j}\le 5\eta$ for all $t\in [T]$ and $j\in[\kappa]$, we have:
	\begin{align*}
	\min_z\|\hat{\ell}_{t}-z\cdot \mathbf{1}\|_{\nabla^{-2}\psi_t(p_t)}^2&\le 5\min_z\|\hat{\ell}_{t}-z\cdot \mathbf{1}\|_{\nabla^{-2}\psi_1(p_t)}^2\le 10\eta\inner{p_t, \hat{\ell}_t}.
	\end{align*}
	Combining the bounds for the Bregman divergence terms and the local-norm term, and accounting for the difference between $u$ and $e_j$ complete the proof.
\end{proof}

We are now ready to prove \pref{thm:kappaLstarThm}.

\begin{proof}{\textbf{of \pref{thm:kappaLstarThm}}.} 
The main idea of the proof is as follows. 
When $i^\star\in\bar{S}$, the regret is exactly $\mathbb{E}\left[\sum_{t=1}^T\inner{p_t-e_{i^\star},\hat{\ell}_t}\right]$. 
Therefore, \pref{lem:CorralLemma} already provides the small-loss bound guarantee by rearranging the terms and taking expectation on both sides. 
When $i^\star\in {C}_j$, according to our loss estimator construction, the regret is exactly the regret of the meta-algorithm plus the regret of $\mathcal{A}_j$, and we apply \pref{lem:CorralLemma} and \pref{lem:HedgeLemma} to bound each of these two parts and importantly use the negative term from \pref{eq:Meta} to cancel the corresponding terms in
 \pref{eq:Sub}. 
	
Formally, when $i^\star\in \bar{S}$, we apply \pref{lem:CorralLemma} with $j=i^\star$ and rearrange terms to arrive at
\begin{align*}
    \sum_{t=1}^T\inner{p_t-e_{i^\star}, \hat{\ell}_t}\le \mathcal{O}\left(\frac{\kappa\ln T+\ln K}{\eta}+c\beta\ln T+\eta\sum_{t=1}^T\hat{\ell}_{t,i^\star} + \frac{1}{T}\sum_{t=1}^T\sum_{j\in[\beta]}\hat{\ell}_{t,j}\right).
\end{align*}
Note that in this case $\mathbb{E}\left[\hat{\ell}_{t,i^\star}\right]=\mathbb{E}\left[\frac{\ell_{t,i^\star}}{1-p_{t,i^\star}}\mathbbm{1}\{i_t\ne i^\star\}\right]=\ell_{t,i^\star}$ for all $t\in[T]$. 
Thus, taking expectation shows
\begin{align*}
    \Reg= \mathcal{O}\left(\frac{\kappa\ln T+\ln K}{\eta}+\beta^2\ln T+\eta L_\star\right).
\end{align*}

On the other hand, when $i^\star\in \mathcal{C}_j$ for some $j\in [\kappa]$, we decompose the regret as 
\begin{align*}
  \Reg = \E\left[\sum_{t=1}^T\inner{p_t, \hat{\ell}_t}-\sum_{t=1}^T\inner{e_{i^\star}, \widetilde{\ell}_{t}}\right] = \E\left[\sum_{t=1}^T\inner{p_t-e_j, \hat{\ell}_t}\right]+\E\left[\sum_{t=1}^T\inner{\widetilde{p}_t^{(j)}-e_{i^\star}, \widetilde{\ell}_{t}}\right].
\end{align*}
Here, in the first equality, we use the facts
\begin{align*}
    \mathbb{E}\left[\inner{p_t,\hat{\ell}_t}\right] &= \mathbb{E}\left[\sum_{j\in [\kappa]}p_{t,j}\sum_{i\in \mathcal{C}_j}\widetilde{p}_{t,i}^{(j)}\frac{\ell_{t,i}}{p_{t,j}}\mathbbm{1}\{j=j_t\}+\sum_{i\in \bar{S}}p_{t,i}\frac{\ell_{t,i}}{1-p_{t,i}}\mathbbm{1}\{i\ne i_t\}\right]\\
    &= \mathbb{E}\left[\sum_{j\in  [\kappa]}p_{t,j}\sum_{i\in \mathcal{C}_j}\widetilde{p}_{t,i}^{(j)}\ell_{t,i}+\sum_{i\in \bar{S}}p_{t,i}\ell_{t,i}\right]\\
    &=\mathbb{E}\left[\ell_{t,i_t}\right]
\end{align*}
and $
    \mathbb{E}\left[\widetilde{\ell}_{t,i^\star}\right]=\mathbb{E}\left[\frac{\ell_{t,i^\star}}{p_{t,j}}\mathbbm{1}\left\{j=j_t\right\}\right]=\ell_{t,i^\star};
$
and the second equality is directly by the definition of $\hat{\ell}_{t,j}$ for $j\in [\kappa]$. 

For the first part of the decomposition, we apply \pref{lem:CorralLemma} directly;
for the second part, noting that the scale of $\widetilde{\ell}_{t}^{(j)}$ for all $t\in [T]$ is no more than $\rho_{T,j}$, according to \pref{lem:HedgeLemma} we have:
\begin{align*}
    \sum_{t=1}^T \inner{\widetilde{p}_{t}^{(j)}-e_{i^\star}, \widetilde{\ell}_{t}} \le 25\rho_{T,j}\ln^2 (KT) +10\ln (KT)\sqrt{\rho_{T,j}\sum_{t=1}^T\widetilde{\ell}_{t,i^\star}}.
\end{align*}
Combining the two gives
\begin{align}
    &\sum_{t=1}^T\inner{p_t, \hat{\ell}_t}-\sum_{t=1}^T\inner{e_{i^\star}, \widetilde{\ell}_{t}} \notag\\
    &\le \mathcal{O}\left(\frac{\kappa\ln T+\ln K}{\eta}+ c\beta\ln T\right) + 80\eta \sum_{t=1}^T\inner{p_t, \hat{\ell}_t} + \frac{1}{T}\sum_{t=1}^T\sum_{j\in[\beta]}\hat{\ell}_{t,j} \notag\\
    &\quad  -\frac{\rho_{T,j}}{40\eta\ln T}+25\rho_{T,j}\ln^2 (KT)-\frac{\rho_{T,j}}{40\eta\ln T}+10\ln (KT)\sqrt{\rho_{T,j}\sum_{t=1}^T\widetilde{\ell}_{t,i^\star}} \notag\\
    &\le \mathcal{O}\left(\frac{\kappa\ln T+\ln K}{\eta}+ c\beta\ln T\right) + 80\eta \sum_{t=1}^T\inner{p_t, \hat{\ell}_t}+\frac{1}{T}\sum_{t=1}^T\sum_{j\in[\beta]}\hat{\ell}_{t,j} \notag\\ 
    &\quad +1000\eta(\ln T)\ln^2 (KT)\sum_{t=1}^T \widetilde{\ell}_{t,i^\star}, \label{eq:key_step}
\end{align}
where the second inequality is by the fact $-ax+\sqrt{bx}\le \frac{b}{4a}$ for $a,b>0$ and also the condition $\eta\le\frac{1}{1000(\ln T)\ln^2 (KT)}$. By rearranging we have:
\begin{align}
    &\sum_{t=1}^T\inner{p_t, \hat{\ell}_t}-\sum_{t=1}^T\inner{e_{i^\star}, \widetilde{\ell}_{t}}\nonumber\\
    &\le\mathcal{O}\left(\frac{\kappa\ln T+\ln K}{\eta}+c\beta\ln T+\eta(\ln T)\ln^2(KT)\sum_{t=1}^T\widetilde{\ell}_{t,i^\star}\right)+\frac{1}{T}\sum_{t=1}^T\sum_{i\in[\beta]}\hat{\ell}_{t,i}. \label{eq:alg1regret}
\end{align}
Taking expectation on both sides finishes the proof.
\end{proof}

\subsection{Adaptive Version of \pref{alg:kappalstaralg}}\label{app:kappaLstarDoublingTrick}
In this section, we provide \pref{alg:DoublingTrickKappalstaralg}, an adaptive version of \pref{alg:kappalstaralg} with a doubling trick to remove the need of tuning the learning rate $\eta$ in terms of $L_\star$. The algorithm mostly follows \pref{alg:kappalstaralg}, starting from a relatively large value of $\eta$.
The key difference is that at the end of each round, we check if condition $\frac{\kappa+1}{\eta}\le \eta\sum_{\tau=T_{\lambda}+1}^{t}\inner{p_{t},\hat{\ell}_{t}}$ holds, where $T_\lambda+1$ is the time step of the most recent reset. 
If the condition holds, it implies that the current learning rate $\eta$ is not small enough, and we thus halve the learning rate, and at the same time reset the algorithm, which includes resetting the parameters $\eta_{t,j}$, $\rho_{t,j}$, and the distribution $p_t$, as well as resetting the Hedge instances. 

Below we prove that \pref{alg:DoublingTrickKappalstaralg} achieves the same regret bound as \pref{alg:kappalstaralg} without knowing $L_\star$.
The main difficulty of the doubling trick analysis is that $\inner{p_t, \hat{\ell}_t}$ is not well bounded when the graph is not self-aware, which is not the case in prior works such as \cite{pmlr-v75-wei18a}. We resolve this issue by again utilizing the negative regret term from the increasing learning rate schedule.

\setcounter{AlgoLine}{0}
\begin{algorithm}[t]\caption{Adaptive Version of \pref{alg:kappalstaralg}}\label{alg:DoublingTrickKappalstaralg}
	
	\textbf{Input:} Feedback graph $G$ and a clique partition $\{\mathcal{C}_1,\dots,\mathcal{C}_\kappa\}$ of $G_S$, parameter $\eta$ and $c$.
	
    	\textbf{Define:} $\beta= \kappa +\woslpn$ and
	$\Omega =\left\{p\in \Delta(\beta):p_j \ge \frac{1}{T}, \forall j\in[\beta]\right\}$.
	

	\For{$\lambda = 1,2,\dots$}{
	       $T_\lambda = t - 1$, $\eta_{t, j} = \eta$, $\rho_{t,j}=2\kappa,\forall j\in [\kappa]$, 
	       $p_t = \argmin_{p \in \Omega} \psi_t(p)$ ($\psi_t$ defined in \pref{eq:hybrid_regularizer3}) 
	       
	       Create an instance $\mathcal{A}_j$ of adaptive Hedge (\pref{alg:AdaHedge}) with nodes in $\mathcal{C}_j$, $\forall j\in [\kappa]$.
	       
			\While{$t\le T$}{
				Execute \pref{line:receive_A_j} to \pref{line:increasing_eta} of \pref{alg:kappalstaralg}.
			
				\If{$\frac{\kappa+1}{\eta}\le \eta\sum_{\tau=T_{\lambda}+1}^{t}\inner{p_{t},\hat{\ell}_{t}}$}{ 
			$\eta \leftarrow \frac{\eta}{2}$, $t \leftarrow t + 1$.\\
			    \textbf{Break}
			}
				$t \leftarrow t + 1$.
			}
	}	
\end{algorithm}

\begin{theorem}\label{thm:DoublingTrickKappalstarthm}
    \pref{alg:DoublingTrickKappalstaralg} with $c=64\beta$ and $\eta = \frac{1}{2000(\ln T)\ln^2(KT)+80\kappa\ln T}$ guarantees 
    \begin{align*}
        \Reg=\mathcal{\tilde{O}}\left(\sqrt{(\kappa+1) L_\star}+\beta^2\right).
    \end{align*}
\end{theorem}

\begin{proof}
We call the time steps between two resets an epoch (indexed by $\lambda$) and let $\eta_\lambda$ be the value of $\eta$ during epoch $\lambda$ so that $\eta_\lambda = 2^{1-\lambda}\eta_1$.
Also let $\lambda^\star$ be the index of the last epoch. For notational convenience, define 
     \begin{align*}
         \widehat{\Reg} \triangleq 
         \begin{cases}
         \sum_{t=1}^T\inner{p_t, \hat{\ell}_t} - \sum_{t=1}^T \widetilde{\ell}_{t,i^\star}, &\mbox{$i^\star\in S$}.\\
         \sum_{t=1}^T\inner{p_t-e_{i^\star}, \hat{\ell}_t},&\mbox{$i^\star\in\bar{S}$}.
         \end{cases}
     \end{align*}
     Note that $\Reg = \E[\widehat{\Reg}]$.
     We will first prove the following
     \begin{equation}\label{eq:reg_hat_bound}
     \widehat{\Reg} \le\sum_{\lambda=1}^{\lambda^\star}\mathcal{\tilde{O}}\left(\frac{\kappa+ 1}{\eta_{\lambda}}+\beta^2\right)+\frac{1}{T}\sum_{t=1}^{T}\sum_{i\in [\beta]}\hat{\ell}_{t,i}.
     \end{equation}
     To show this, consider the regret in each epoch $\lambda$. 
     When $i^\star\in \bar{S}$, we have:
     \begin{align*}
     &\sum_{t=T_{\lambda}+1}^{T_{\lambda+1}}\inner{p_t-e_{i^\star}, \hat{\ell}_t}\\&\le\mathcal{O}\left( c\beta\ln T+\frac{\kappa\ln T+\ln K}{\eta_\lambda}\right)+40\eta_{\lambda}\sum_{t=T_{\lambda}+1}^{T_{\lambda+1}}\min_z\|\hat{\ell}_{t}-z\cdot\one\|_{\nabla^{-2}\psi_t(p_{t})}^2 +\frac{1}{T}\sum_{t=T_{\lambda}+1}^{T_{\lambda+1}}\sum_{i\in [\beta]}\hat{\ell}_{t,i}\\
         &\le \mathcal{\tilde{O}}\left(\beta^2+\frac{\kappa+1}{\eta_\lambda}\right)+80\eta_\lambda\sum_{t=T_{\lambda}+1}^{T_{\lambda+1}-1}\inner{p_t, \hat{\ell}_t} +\frac{5}{8}+\frac{1}{T}\sum_{t=T_{\lambda}+1}^{T_{\lambda+1}}\sum_{i\in [\beta]}\hat{\ell}_{t,i} \\
         &\le \mathcal{\tilde{O}}\left(\beta^2+\frac{\kappa+1}{\eta_{\lambda}}\right)+\frac{1}{T}\sum_{t=T_{\lambda}+1}^{T_{\lambda+1}}\sum_{i\in [\beta]}\hat{\ell}_{t,i}.
     \end{align*}
     Here, the first inequality is according to the analysis of \pref{lem:CorralLemma}. 
     In the second inequality, we bound $\min_z\|\hat{\ell}_{t}-z\cdot\one\|_{\nabla^{-2}\psi(p_{t})}^2$ by $2\inner{p_t, \hat{\ell}_t}$ for $t = T_\lambda+1, \ldots, T_{\lambda+1}-1$ by the same analysis of \pref{thm:thmbetalstar}, and bound the same term for $t = T_{\lambda+1}$  by $\frac{1}{64}$ by \pref{lem:SpecialStability}.
     The final inequality is because the reset condition does not hold for $t = T_{\lambda+1}-1$. 
     Summing over the epochs proves \pref{eq:reg_hat_bound} for the first case.
     
     When $i^\star\in \mathcal{C}_j$ for some $j\in[\kappa]$, similar to the previous analysis and the derivation of \pref{eq:key_step}, we have:
     \begin{align*}
     &\sum_{t=T_{\lambda}+1}^{T_{\lambda+1}}\inner{p_t, \hat{\ell}_t}-\sum_{t=T_{\lambda}+1}^{T_{\lambda+1}}\widetilde{\ell}_{t,i^\star}\\
    &\le \mathcal{O}\left(c\beta\ln T+\frac{\kappa\ln T+\ln K}{\eta_\lambda}\right)+40\eta_{\lambda}\sum_{t=T_{\lambda}+1}^{T_{\lambda+1}}\min_z\|\hat{\ell}_{t}-z\cdot\one\|_{\nabla^{-2}\psi(p_{t})}^2+\frac{1}{T}\sum_{t=T_{\lambda}+1}^{T_{\lambda+1}}\sum_{i\in[\beta]}\hat{\ell}_{t,i}\\
     &\quad-\frac{\rho_{T_{\lambda+1},j}}{80\eta_{\lambda}\ln T}-\frac{\rho_{T_{\lambda+1},j}}{80\eta_\lambda\ln T}+25\rho_{T_{\lambda+1},j}\ln^2 (KT)-\frac{\rho_{T_{\lambda+1},j}}{40\eta_{\lambda}\ln T}+10\ln (KT)\sqrt{\rho_{T_{\lambda+1},j}\sum_{t=T_{\lambda}+1}^{T_{\lambda+1}}\widetilde{\ell}_{t,i^\star}}\\
     &\le  \mathcal{\tilde{O}}\left(\beta^2+\frac{\kappa+1}{\eta_{\lambda}}\right)-\frac{\rho_{T_{\lambda+1},j}}{80\eta_\lambda\ln T} + 1000\eta_{\lambda}(\ln T)\ln^2 (KT)\sum_{t=T_{\lambda}+1}^{T_{\lambda+1}}\widetilde{\ell}_{t,i^\star}+\frac{1}{T}\sum_{t=T_{\lambda}+1}^{T_{\lambda+1}}\sum_{i\in[\beta]}\hat{\ell}_{t,i},
     \end{align*}
     where the second inequality uses the fact $\eta_{\lambda}\le \eta_1\le \frac{1}{2000\ln T\ln^2 (KT)}$ and the AM-GM inequality. By rearranging terms, we have 
     \begin{align*}
     &\left(1+1000\eta(\ln T)\ln^2(KT)\right)\left(\sum_{t=T_{\lambda}+1}^{T_{\lambda+1}}\inner{p_t, \hat{\ell}_t}-\sum_{t=T_{\lambda}+1}^{T_{\lambda+1}}\widetilde{\ell}_{t,i^\star}\right)\\
     &\le \mathcal{\tilde{O}}\left(\beta^2+\frac{\kappa+1}{\eta_{\lambda}}\right)-\frac{\rho_{T_{\lambda+1},j}}{80\eta_\lambda\ln T} + 1000\eta_{\lambda}(\ln T)\ln^2 (KT)\sum_{t=T_{\lambda}+1}^{T_{\lambda+1}}\inner{p_t, \hat{\ell}_t}+\frac{1}{T}\sum_{t=T_{\lambda}+1}^{T_{\lambda+1}}\sum_{i\in[\beta]}\hat{\ell}_{t,i}\\ 
     &\le \mathcal{\tilde{O}}\left(\beta^2+\frac{\kappa+1}{\eta_{\lambda}}\right) +\frac{1}{2}\inner{p_{T_{\lambda+1}}, \hat{\ell}_{T_{
     \lambda+1}}}-\frac{\rho_{T_{\lambda+1},j}}{80\eta_\lambda\ln T} +\frac{1}{T}\sum_{t=T_{\lambda}+1}^{T_{\lambda+1}}\sum_{i\in [\beta]}\hat{\ell}_{t,i},
     \end{align*}
     where the second inequality is again because the condition $\frac{\kappa +1}{\eta_{\lambda}}\le \eta_{\lambda}\sum_{t=T_{\lambda}+1}^{T_{\lambda+1}-1}\inner{p_t, \hat{\ell}_t}$ does not hold and $\eta_{\lambda}\le \eta_1\le \frac{1}{2000(\ln T)\ln^2(KT)}$. Now if for all $i\in \bar{S}$, $p_{T_{\lambda+1},i}<\frac{1}{2}$, then we have \[\inner{p_{T_{\lambda+1}},\hat{\ell}_{T_{\lambda+1} }} \le \sum_{j'\in [\kappa]}p_{T_{\lambda+1},j'}\frac{1}{p_{T_{\lambda+1},j'}}+\sum_{i\in \bar{S}}\frac{p_{T_{\lambda+1},i}}{1-p_{T_{\lambda+1},i}}\le \beta.\]
     Otherwise, we have exactly one $i_0\in \bar{S}$ such that $p_{T_{\lambda+1}, i_0} > \frac{1}{2}$ and then we have
     \[\inner{p_{T_{\lambda+1}},\hat{\ell}_{T_{\lambda+1} }} \le \sum_{j'\in [\kappa]}p_{T_{\lambda+1},j'}\frac{1}{p_{T_{\lambda+1},j'}}+\sum_{i\in \bar{S}}\frac{p_{T_{\lambda+1},i}}{1-p_{T_{\lambda+1},i}}\le \kappa + \frac{1}{1-p_{T_{\lambda+1},i_0}}\le \kappa+\rho_{T_{\lambda+1},j}.\]
     The last inequality is because $1-p_{T_{\lambda+1},i_0}\ge p_{T_{\lambda+1}, j}$. Therefore, as $\eta_{\lambda}\le \eta_1 = \frac{1}{2000\ln T\ln^2 (KT)+80\kappa\ln T}$, we always have
     $$\inner{p_{T_{\lambda+1}}, \hat{\ell}_{T_{\lambda+1}}}\le \frac{\rho_{T_{\lambda+1},j}}{80\eta_{\lambda}\ln T}.$$
     We have thus shown
     \begin{align*}
         \sum_{t=T_{\lambda}+1}^{T_{\lambda+1}}\inner{p_t, \hat{\ell}_t}-\sum_{t=T_{\lambda}+1}^{T_{\lambda+1}}\widetilde{\ell}_{t,i^\star}\le \mathcal{\tilde{O}}\left(\beta^2+\frac{\kappa+1}{\eta_{\lambda}}\right)+\frac{1}{T}\sum_{t=T_{\lambda}+1}^{T_{\lambda+1}}\sum_{i\in [\beta]}\hat{\ell}_{t,i}.
     \end{align*}
    Summing up the regret from epoch $1$ to $\lambda^\star$ gives \pref{eq:reg_hat_bound}. 
    
    Next, using the definition of $\eta_\lambda$, we further have
     \begin{align}
         \widehat{\Reg}&\le\sum_{\lambda=1}^{\lambda^\star}\mathcal{\tilde{O}}\left(\frac{\kappa+ 1}{\eta_{\lambda}}+\beta^2\right)+\frac{1}{T}\sum_{t=1}^{T}\sum_{i\in [\beta]}\hat{\ell}_{t,i}\le \mathcal{\tilde{O}}\left(\frac{\kappa +1}{\eta_{\lambda^\star}}+\beta^2\lambda^\star\right)+\frac{1}{T}\sum_{t=1}^{T}\sum_{i\in [\beta]}\hat{\ell}_{t,i}.\label{eq:reg_hat_bound2}
     \end{align}
     When $\lambda^\star = 1$, direct calculation gives $\wh{\Reg}\le \mathcal{\tilde{O}}(\beta^2)+\frac{1}{T}\sum_{t=1}^{T}\sum_{i\in [\beta]}\hat{\ell}_{t,i}.\label{eq:lambda1}$
     On the other hand, if $\lambda^\star\ge 2$, consider the time step at the end of epoch $\lambda^\star-1$. Using the reset condition, we have:
       \begin{align*}
            (\kappa+1)\left(\frac{2^{\lambda^\star-2}}{\eta_1}\right)^2=(\kappa+1)\frac{1}{\eta_{\lambda^\star-1}^2}  \le \sum_{t=T_{\lambda^\star-1}+1}^{T_{\lambda^\star}}\inner{p_t, \hat{\ell}_t} \le \sum_{t=1}^T\inner{p_t, \hat{\ell}_t}\le T^2.
       \end{align*}
     So $\lambda^\star=\mathcal{O}(\ln T)$, $\left(\frac{\kappa +1}{\eta_{\lambda^\star}}\right)^2=\mathcal{\tilde{O}}\left({(\kappa+1)\sum_{t=1}^T\inner{p_t, \hat{\ell}_t}}\right)$. Plugging these into \pref{eq:reg_hat_bound2}, we have
      \begin{align*}
           \widehat{\Reg}&\le \mathcal{\tilde{O}}\left(\sqrt{(\kappa+1)\left(\sum_{t=1}^T\inner{p_t, \hat{\ell}_t}\right)}+\beta^2\right)+\frac{1}{T}\sum_{t=1}^{T}\sum_{i\in [\beta]}\hat{\ell}_{t,i},
       \end{align*}
     which also holds for the case $\lambda^\star=1$. Finally, taking expectation on both sides gives:
     \begin{align*}
         \Reg&=\mathbb{E}\left[\widehat{\Reg}\right]\le \mathcal{\tilde{O}}\left(\mathbb{E}\left[\sqrt{(\kappa+1)\left(\sum_{t=1}^T\inner{p_t, \hat{\ell}_t}\right)}+\beta^2\right]\right)\\
         &\le\mathcal{\tilde{O}}\left(\sqrt{\left(\kappa+ 1\right)\left(\mathbb{E}\left[\sum_{t=1}^T\inner{p_t, \hat{\ell}_t}\right]\right)}+\beta^2\right)\le \mathcal{\tilde{O}}\left(\sqrt{\left(\kappa+1\right)\left(\mathbb{E}\left[\sum_{t=1}^T\inner{p_t, \ell_t}\right]\right)}+\beta^2\right).
     \end{align*}
     Solving the quadratic inequality, we obtain the regret bound $\Reg\le\mathcal{\tilde{O}}\left(\sqrt{(\kappa+1) L_\star}+\beta^2\right)$.
\end{proof}

\section{Omitted details for \pref{sec: minAlphaKappa}\label{app:minAlphaKappa}}

\setcounter{AlgoLine}{0}
\begin{algorithm}[h]\caption{An Algorithm with Regret $\mathcal{\tilde{O}}(\min\{\sqrt{\alpha T}, \sqrt{\kappa L_\star}\})$ for Self-aware Graphs}\label{alg:ClippedEXP3.G}
	\textbf{Input:} A clique partition $\{\mathcal{C}_1,\dots,\mathcal{C}_\kappa\}$ of $G_S$, parameter $\eta_{\text{init}}$ and $\epsilon$.\\
	\textbf{Define:} $\Omega=\{p\in \Delta(K): p_i\ge \frac{1}{T},\forall i\in [K]\}$.\\
	\nl\label{line: meta-epoch} \For{$\mtep=1,2,\dots,\log_2 T$}{
	\nl \label{line: init eta} $\eta = \eta_{\text{init}}$.\\
	\nl\label{line: epoch}\For{$\lambda=1,2,\dots$}{
		\nl $p_t = \frac{1}{K}\cdot \one$, $T_{\lambda}=t-1$. \label{line: reset}\\
		\nl \While{$t\le T$}{
		    \nl Pull arm $i_t\sim p_t$ and receive feedback $\ell_{t,i}$ for all $i$ such that $i_t\in\Nin(i)$.\\
		    \nl Construct estimator $\hat{\ell}_t\in\mathbb{R}^K$ such that $\hat{\ell}_{t,i}=
		        \begin{cases} \frac{\ell_{t,i}\cdot \mathbbm{1}\{i_t\in\Nin(i)\}}{\sum_{j\in\Nin(i)}p_{t,j}},&\mbox{if  $p_{t,i}>0$.}\\
		        0,&\mbox{if $p_{t,i}=0$.}
		        \end{cases}$\\
			\nl Compute $\hat{p}_{t+1}=\argmin_{p\in\Delta_K}\left\{ \left<p,\hat{\ell}_t\right>+D_{\psi}(p,\hat{p}_t)\right\}$, where 	$$\psi(p)=\frac{1}{\eta}\sum_{i\in [K]}p_i\ln p_i.$$
			\nl $p_t=\hat{p}_t$.\\	
            \nl\For{$j=1,2,\dots,\kappa$}{	\nl\label{line: Clipping threshold}\If{$\sum_{i\in \mathcal{C}_j}{p}_{t+1,i}\le \epsilon$}{
			\nl \For {$i\in \mathcal{C}_j$}{
			\nl 			${p}_{t+1,i}=0$.\label{line: Clipping}
					}
				}
			}
			\nl Renormalize $p_{t+1}$ such that $p_{t+1}\in\Delta(K)$. \label{line: normalization}\\
			\nl\label{line: update rule}\If{$\frac{1}{\eta} \le 4\eta\kappa\min_{i\in [K]}\left\{\sum_{\tau=T_{\lambda}+1}^{t}\hat{\ell}_{\tau,i}\right\}$}{
			\nl	$\eta\leftarrow \frac{\eta}{2}$, $\epsilon=\max\{2\eta,\frac{1}{T}\},t \leftarrow t+1$.\label{line: reset para}\\ 
			\nl\label{line: eta is small}\If{$\eta\le \sqrt{\frac{1}{\alpha T}}$}{\label{line: start meta-epoch}
			\nl Jump to \pref{line: meta-epoch}.
			}
		        \nl \label{line: start epoch} Jump to \pref{line: epoch}.\\
                \nl \textbf{break}.
			}
	    }
    }
}
\nl \label{line: stage two}Run the algorithm from \pref{thm:alphaTEXP3} (from scratch) for the rest of the game.
\end{algorithm}

In this section, we discuss how to obtain $\mathcal{\tilde{O}}(\min\{\sqrt{\alpha T}, \sqrt{\kappa L_\star}\})$ regret for self-aware graphs. \pref{alg:ClippedEXP3.G} shows the complete pseudocode. It consists of two stages. The first stage is when $\mtep\le \log_2 T$ and runs yet another parameter-free algorithm with $\mathcal{\tilde{O}}(\sqrt{\kappa L_\star})$ regret for self-aware graphs (for technical reasons we are not able to use \pref{alg:kappalstaralg} directly here). 
The second stage exactly runs the algorithm mentioned in \pref{sec: remove uniform exploration} (\pref{thm:alphaTEXP3}), which achieves $\mathcal{\tilde{O}}(\sqrt{\alpha T})$ regret. 

The first stage mainly follows the clipping idea of \citep{allenberg2006hannan} for OMD with entropy regularizer.  
At each round $t$, after updating the distribution with OMD, we clip the probability of a clique to zero if it has low probability to be chosen (\pref{line: Clipping threshold} and \pref{line: Clipping}), and normalize the distribution after clipping (\pref{line: normalization}). 
Then, a doubling trick is introduced: once the condition $\frac{1}{\eta} \le 4\eta\kappa\min_{i\in [K]}\{\sum_{\tau=T_{\lambda}+1}^{t}\hat{\ell}_{\tau,i}\}$  holds (\pref{line: update rule}), we halve the learning rate and reset (\pref{line: start epoch}). 
We point out that thanks to the clipping trick, the term $\min_{i\in [K]}\{\sum_{\tau=T_{\lambda}+1}^{t}\hat{\ell}_{\tau,i}\}$ is nicely bounded, which is crucial for the doubling trick analysis.
We say that we start a new \emph{epoch} in this case. 

Furthermore, once the learning rate is smaller than $\sqrt{\nicefrac{1}{\alpha T}}$ (\pref{line: eta is small}), we jump to \pref{line: meta-epoch} and reset it to the initial learning rate $\eta=\eta_{\text{init}}$ in \pref{line: init eta}. 
We say that we start a new \emph{meta-epoch} in this case. 
After having $\log_2 T$ meta-epochs, the algorithm is confident that $\sqrt{\alpha T}\le \mathcal{\tilde{O}}(\sqrt{\kappa L_\star})$, and thus switches to the second stage and runs the algorithm introduced in \pref{sec: remove uniform exploration} with regret $\mathcal{\tilde{O}}(\sqrt{\alpha T})$. 
We point out that in fact any  algorithm with $\mathcal{\tilde{O}}(\sqrt{\alpha T})$ regret is acceptable for this second stage. 


The guarantee of this algorithm is shown below.

\begin{theorem}\label{thm:Corralthm}
	\pref{alg:ClippedEXP3.G} with $\eta_{\text{init}}=\frac{1}{4\kappa}$ and $\epsilon=\frac{1}{2\kappa}$ guarantees 
	\[\Reg = \mathcal{\tilde{O}}\left(\min\left\{\sqrt{\alpha T}, \sqrt{\kappa L_\star}\right\}+K^2\right).\]
\end{theorem}
To prove the theorem, we first prove a bound on the regret within an epoch $\lambda$, where $\eta\le\frac{1}{4\kappa}$ and $\epsilon =\max\{2\eta, \frac{1}{T}\}$ are fixed.
\begin{lemma}\label{lem:epoch}
Fix a meta-epoch and consider an epoch $\lambda$ in the first stage. \pref{alg:ClippedEXP3.G} guarantees
    \begin{align*}
    &\sum_{t = {T}_\lambda+1}^{T_{\lambda+1}}\left<p_t, \hat{\ell}_t\right>-\min_{i\in [K]}\sum_{t = {T}_\lambda+1}^{T_{\lambda+1}}\hat{\ell}_{t,i}= \mathcal{\tilde{O}}\left(\frac{1}{\eta}+\eta\kappa\min_{i\in [K]}\sum_{t = {T}_\lambda+1}^{T_{\lambda+1}}\hat{\ell}_{t,i}+\kappa+\frac{K}{T}\min_{i\in [K]}\sum_{t = {T}_\lambda+1}^{T_{\lambda+1}}\hat{\ell}_{t,i}.\right).
    \end{align*}
\end{lemma}
\begin{proof}
According to the analysis of online mirror descent with entropy regularizer~\citep{cesa2006prediction}, we have
	\begin{align*}
	\sum_{t = {T}_\lambda+1}^{T_{\lambda+1}} \left<\hat{p}_t, \hat{\ell}_t\right>-\min_{i\in [K]}\sum_{t = {T}_\lambda+1}^{T_{\lambda+1}}\hat{\ell}_{t,i} &\le \frac{\ln K}{\eta}+\eta\sum_{t = {T}_\lambda+1}^{T_{\lambda+1}}\sum_{i=1}^K\hat{p}_{t,i}\hat{\ell}_{t,i}^2 \\
	&\le \frac{\ln K}{\eta}+\eta\sum_{t = {T}_\lambda+1}^{T_{\lambda+1}}\sum_{j\in [\kappa]}\sum_{i\in \mathcal{C}_j}\frac{\hat{p}_{t,i}}{\sum_{i'\in \mathcal{C}_j}p_{t,i'}}\hat{\ell}_{t,i}.
	\end{align*}
	As we do clipping for each clique, if $p_{t,i}>0$, then $1-\kappa\epsilon\le \frac{\hat{p}_{t,i}}{p_{t,i}}\le 1$ and otherwise $\hat{\ell}_{t,i}=0$. Therefore, we have
	\begin{align}
	\sum_{t = {T}_\lambda+1}^{T_{\lambda+1}} (1-\kappa\epsilon)\left<p_t, \hat{\ell}_t\right>-\min_{i\in [K]}\sum_{t = {T}_\lambda+1}^{T_{\lambda+1}}\hat{\ell}_{t,i} &\le \frac{\ln K}{\eta}+\eta\sum_{t = {T}_\lambda+1}^{T_{\lambda+1}}\sum_{j\in [\kappa]}\sum_{i\in \mathcal{C}_j}\frac{p_{t,i}}{\sum_{i'\in \mathcal{C}_j}p_{t,i'}}\hat{\ell}_{t,i} \label{eq:clip}
	\end{align}
	Now consider a fixed clique $\mathcal{C}_j$ that is not clipped at time $t$. We have $\frac{p_{t,i}}{\sum_{i'\in \mathcal{C}_j}p_{t,i'}} = \frac{\hat{p}_{t,i}}{\sum_{i'\in \mathcal{C}_j}\hat{p}_{t,i'}} = \frac{\exp(-\eta\sum_{\tau=1}^{t-1}\hat{\ell}_{\tau,i})}{\sum_{i'\in\mathcal{C}_j}\exp(-\eta\sum_{\tau=1}^{t-1}\hat{\ell}_{\tau,i'})}$, which can be considered as a probability distribution generated by online mirror descent with entropy regularizer inside the clique. Therefore for any clique $\mathcal{C}_j$, $j\in [\kappa]$, we have 
	\begin{align*}
	\sum_{t = {T}_\lambda+1}^{T_{\lambda+1}}\sum_{i\in \mathcal{C}_j}\frac{p_{t,i}}{\sum_{i'\in \mathcal{C}_j}p_{t,i'}}\hat{\ell}_{t,i}-\min_{i\in \mathcal{C}_j}\sum_{t = {T}_\lambda+1}^{T_{\lambda+1}}\hat{\ell}_{t,i} &\le \frac{\ln K}{\eta}+\eta\sum_{t = {T}_\lambda+1}^{T_{\lambda+1}}\sum_{i\in \mathcal{C}_j}\frac{p_{t,i}}{\sum_{i'\in \mathcal{C}_j} p_{t,i'}}\hat{\ell}_{t,i}^2 \\
	&\le \frac{\ln K}{\eta}+\frac{\eta}{\epsilon}\sum_{t = {T}_\lambda+1}^{T_{\lambda+1}}\sum_{i\in \mathcal{C}_j}\frac{p_{t,i}}{\sum_{i'\in \mathcal{C}_j}p_{t,i'}}\hat{\ell}_{t,i} \\
	&\le \frac{\ln K}{\eta}+\frac{1}{2}\sum_{t = {T}_\lambda+1}^{T_{\lambda+1}}\sum_{i\in \mathcal{C}_j}\frac{p_{t,i}}{\sum_{i'\in \mathcal{C}_j}p_{t,i'}}\hat{\ell}_{t,i}.
	\end{align*}
	The second inequality is because $\hat{\ell}_{t,i}\le \frac{1}{\epsilon}$ for all $i\in\mathcal{C}_j$, $j\in[\kappa]$ and the third inequality is because $\epsilon\ge 2\eta$. Rearranging the terms, we have
	\begin{align*}
	\sum_{t = {T}_\lambda+1}^{T_{\lambda+1}}\sum_{i\in \mathcal{C}_j}\frac{p_{t,i}}{\sum_{i'\in \mathcal{C}_j}p_{t,i'}}\hat{\ell}_{t,i}\le \frac{2\ln K}{\eta}+2\min_{i\in \mathcal{C}_j}\sum_{t = {T}_\lambda+1}^{T_{\lambda+1}}\hat{\ell}_{t,i}.
	\end{align*}
	Therefore, by combining with \pref{eq:clip}, we have
	\begin{align}
	(1-\kappa\epsilon)\sum_{t = {T}_\lambda+1}^{T_{\lambda+1}}\left<p_t, \hat{\ell}_t\right>-\min_{i\in [K]}\sum_{t = {T}_\lambda+1}^{T_{\lambda+1}}\hat{\ell}_{t,i}&\le \frac{\ln K}{\eta}+\eta\sum_{j\in [\kappa]}\left(\frac{2\ln K}{\eta}+2\min_{i\in \mathcal{C}_j}\sum_{t = {T}_\lambda+1}^{T_{\lambda+1}}\hat{\ell}_{t,i}\right)\nonumber\\
	&\le \frac{\ln K}{\eta} + 2\kappa\ln K+2\eta\sum_{j\in [\kappa]}\min_{i\in \mathcal{C}_j}\sum_{t = {T}_\lambda+1}^{T_{\lambda+1}}\hat{\ell}_{t,i}.\label{eq:clipped-inter}
	\end{align}
	Furthermore, let $T_{i_0}$ be the last round such that $\hat{\ell}_{t,i_0} > 0$ for some $i_0\in \mathcal{C}_j$. Note that $T_{i_0}$ is also the last round such that $\ell_{t,i'}>0$ for all $i'\in \mathcal{C}_j$. Then we have for any $i'\in [K]$:
	\begin{align*}
	\epsilon \le \sum_{i\in \mathcal{C}_j}\hat{p}_{T_{i_0},i} &= \frac{\sum_{i\in \mathcal{C}_j}\exp(-\eta\sum_{\tau=T^{(s)}+1}^{T_{i_0}-1}\hat{\ell}_{\tau,i})}{\sum_{i=1}^K\exp(-\eta\sum_{\tau=T^{(s)}+1}^{T_{i_0}-1}\hat{\ell}_{\tau, i})} \\
	&\le \frac{\sum_{i\in \mathcal{C}_j}\exp(-\eta\sum_{t = {T}_\lambda+1}^{T_{\lambda+1}}\hat{\ell}_{t,i}+\frac{\eta}{\epsilon})}{\sum_{i=1}^K\exp(-\eta\sum_{t = {T}_\lambda+1}^{T_{\lambda+1}}\hat{\ell}_{t, i})} \\
	&\le \frac{|\mathcal{C}_j|\exp(-\eta\min_{i\in\mathcal{C}_j}\sum_{t = {T}_\lambda+1}^{T_{\lambda+1}}\hat{\ell}_{t,i}+\frac{\eta}{\epsilon})}{\exp(-\eta\sum_{t = {T}_\lambda+1}^{T_{\lambda+1}}\hat{\ell}_{t,i'})},
	\end{align*}
	where the second inequality is because $\hat{\ell}_{T_{i_0},i}\le\frac{1}{\epsilon}$ for all $i\in \mathcal{C}_j$ and the fact $\hat{\ell}_{t,i}\ge 0$ for all $t\in [T]$ and $i\in [K]$. Therefore, by rearranging terms, for any $i'\in[K]$ and any $j\in [\kappa]$, we have
	\begin{align*}
	\min_{i\in \mathcal{C}_j}\sum_{t = {T}_\lambda+1}^{T_{\lambda+1}}\hat{\ell}_{t,i}\le \sum_{t = {T}_\lambda+1}^{T_{\lambda+1}}\hat{\ell}_{t,i'}+\frac{1}{\epsilon}+\frac{1}{\eta}\ln \frac{K}{\epsilon}.
	\end{align*}
	Combining \pref{eq:clipped-inter} and choosing $i'=\argmin_{i\in [K]}\sum_{t = {T}_\lambda+1}^{T_{\lambda+1}}\hat{\ell}_{t,i}$ further show
	\begin{align*}
	(1-\kappa\epsilon)\sum_{t = {T}_\lambda+1}^{T_{\lambda+1}}\left<p_t, \hat{\ell}_t\right>-\min_{i\in [K]}\sum_{t = {T}_\lambda+1}^{T_{\lambda+1}}\hat{\ell}_{t,i}&\le\frac{\ln K}{\eta}+2\kappa\ln K+2\eta\kappa\left(\min_{i\in [K]}\sum_{t = {T}_\lambda+1}^{T_{\lambda+1}}\hat{\ell}_{t,i}+\frac{1}{\epsilon}+\frac{1}{\eta}\ln \frac{K}{\epsilon}\right) \\
	&\le \frac{\ln K}{\eta}+2\eta\kappa\min_{i\in [K]}\sum_{t = {T}_\lambda+1}^{T_{\lambda+1}}\hat{\ell}_{t,i}+3\kappa\ln K + 2\kappa\ln (KT)\\
	&\le \frac{\ln K}{\eta}+2\eta\kappa\min_{i\in [K]}\sum_{t = {T}_\lambda+1}^{T_{\lambda+1}}\hat{\ell}_{t,i}+5\kappa\ln (KT),
	\end{align*}
	where the second inequality is because $\epsilon=\max\{2\eta, \frac{1}{T}\}$. Finally, rearranging terms again shows
	\begin{align*}
	&\sum_{t = {T}_\lambda+1}^{T_{\lambda+1}}\left<p_t, \hat{\ell}_t\right>-\min_{i\in [K]}\sum_{t = {T}_\lambda+1}^{T_{\lambda+1}}\hat{\ell}_{t,i}\\
	&\le \frac{1}{1-\kappa\epsilon}\left(\frac{\ln K}{\eta}+\kappa(2\eta+\epsilon)\min_{i\in [K]}\sum_{t = {T}_\lambda+1}^{T_{\lambda+1}}\hat{\ell}_{t,i}+5\kappa\ln (KT)\right) \\
	&\le \frac{2\ln K}{\eta}+8\eta\kappa\min_{i\in [K]}\sum_{t = {T}_\lambda+1}^{T_{\lambda+1}}\hat{\ell}_{t,i}+10\kappa\ln (KT)+\frac{2K}{T}\min_{i\in [K]}\sum_{t = {T}_\lambda+1}^{T_{\lambda+1}}\hat{\ell}_{t,i}.
	\end{align*}
	The second inequality is because $\eta\le \frac{1}{4\kappa}$, which means $\kappa\epsilon\le \max\{\frac{\kappa}{T},2\eta\kappa\}\le \frac{1}{2}$, as $T\ge 2K$.
\end{proof}

Next we bound the regret within a meta-epoch. In the remaining of the section, we use $\mathcal{T}_\mtep$ to denote the set of rounds in meta-epoch $\mtep$.
\begin{lemma}\label{lem:ClippedEXP3.Gthm}
    For any meta-epoch $\mtep$ in the first stage, \pref{alg:ClippedEXP3.G} guarantees
    \begin{align*}
    &\sum_{t\in\mathcal{T}_\mtep}\left<p_t, \hat{\ell}_t\right>-\min_{i\in [K]}\sum_{t\in\mathcal{T}_\mtep}\hat{\ell}_{t,i}\le  \mathcal{\tilde{O}}\left(\min\left\{\sqrt{\alpha T},\sqrt{\kappa \min_{i\in [K]}\sum_{t\in\mathcal{T}_\mtep}\hat{\ell}_{t,i}}\right\}+\kappa+\frac{K}{T}\min_{i\in [K]}\sum_{t\in\mathcal{T}_\mtep}\hat{\ell}_{t,i}\right).
    \end{align*}
\end{lemma}
\begin{proof}
	 Let $\eta_{\lambda}=2^{1-\lambda}\eta_{1}$ and $\epsilon_\lambda=\max\{2\eta_{\lambda}, \nicefrac{1}{T}\}$ be the value of $\eta$ and $\epsilon$ during epoch $\lambda$ and let $\lambda^\star$ be the index of the last epoch. Consider the regret in epoch $\lambda$. Using \pref{lem:epoch}, we know that
	\begin{align*}
	&\sum_{T_{\lambda}+1}^{T_{\lambda+1}}\left<p_t, \hat{\ell}_t\right>-\min_{i\in [K]}\sum_{t=T_{\lambda}+1}^{T_{\lambda+1}}\hat{\ell}_{t,i}\\
    &= \mathcal{\tilde{O}}\left(\frac{1}{\eta_{\lambda}}+\eta_{\lambda}\kappa\min_{i\in [K]}\sum_{t= {T}_\lambda+1}^{T_{\lambda+1}}\hat{\ell}_{t,i}+\kappa+\frac{K}{T}\min_{i\in [K]}\sum_{t= {T}_\lambda+1}^{T_{\lambda+1}}\hat{\ell}_{t,i}.\right)\\
	&=\mathcal{\tilde{O}}\left( \frac{1}{\eta_{\lambda}}+\frac{\eta_{\lambda}\kappa}{\epsilon_{\lambda}} + \kappa+\frac{K}{T}\min_{i\in [K]}\sum_{t=T_{\lambda}+1}^{T_{\lambda+1}}\hat{\ell}_{t,i}\right)\tag{the update rule and $\hat{\ell}_{T_{\lambda+1},i}\le \frac{1}{\epsilon_{\lambda}}$ for all $i\in [K]$} \\
    &=\mathcal{\tilde{O}}\left( \frac{1}{\eta_{\lambda}}+ \kappa+\frac{K}{T}\min_{i\in [K]}\sum_{t=T_{\lambda}+1}^{T_{\lambda+1}}\hat{\ell}_{t,i}\right) \tag{$\epsilon_{\lambda}=\max\{\frac{1}{T},2\eta_{\lambda}\}$}\\
	&=\mathcal{\tilde{O}}\left( \frac{1}{\eta_{\lambda}}+\frac{K}{T}\min_{i\in [K]}\sum_{t=T_{\lambda}+1}^{T_{\lambda+1}}\hat{\ell}_{t,i}\right).\tag{$\eta_{\lambda}\le \eta_1 = \frac{1}{4\kappa}$}
	\end{align*}
	Taking a summation from $\lambda=1,2,\dots,\lambda^\star$, we have
	\begin{align}\label{eq:etalambdastar}
		\sum_{t\in\mathcal{T}_\mtep}\inner{p_t, \hat{\ell}_t}-\min_{i\in [K]}\sum_{t\in \mathcal{T}_\mtep}\hat{\ell}_{t,i}&\le \sum_{\lambda=1}^{\lambda^\star}\mathcal{\tilde{O}}\left( \frac{1}{\eta_{\lambda}}+\frac{K}{T}\min_{i\in [K]}\sum_{t=T_{\lambda}+1}^{T_{\lambda+1}}\hat{\ell}_{t,i}\right)\nonumber\\
        &\le\mathcal{\tilde{O}}\left( \frac{1}{\eta_{\lambda^\star}}+\frac{K}{T}\min_{i\in [K]}\sum_{t\in\mathcal{T}_\mtep}\hat{\ell}_{t,i}\right).
	\end{align}
	Now we come to bound $\frac{1}{\eta_{\lambda^\star}}$. If $\lambda^\star=1$, we have $\sum_{t\in\mathcal{T}_\mtep}\inner{p_t, \hat{\ell}_t}-\min_{i\in [K]}\sum_{t\in \mathcal{T}_\mtep}\hat{\ell}_{t,i}\le \mathcal{\tilde{O}}(K)+\frac{2K}{T}\min_{i\in [K]}\sum_{t\in \mathcal{T}_\mtep}\hat{\ell}_{t,i}$. If $\lambda^\star\ge 2$,  consider the last round of epoch $\lambda^\star-1$. According to the update rule of $\eta$, we have
	\begin{align*}
	\frac{1}{\eta_{\lambda^\star-1}^2}\le 4\kappa\min_{i\in [K]}\sum_{t=T_{\lambda^\star-1}+1}^{T_{\lambda^\star}}\hat{\ell}_{t,i}\le 4\kappa\min_{i\in [K]}\sum_{t\in\mathcal{T}_\mtep}\hat{\ell}_{t,i}.
	\end{align*}
	Therefore, we have $\frac{1}{\eta_{\lambda^\star}}\le\sqrt{16\kappa\min_{i\in [K]}\sum_{t\in\mathcal{T}_\mtep}\hat{\ell}_{t,i}}$. In addition, note that $\eta_{\lambda^\star}\ge \sqrt{\frac{1}{\alpha T}}$ by \pref{line: start meta-epoch} of the algorithm. So $\frac{1}{\eta_{\lambda^\star}}=\mathcal{\tilde{O}}\left(\min\left\{\sqrt{\alpha T},\sqrt{\kappa\min_{i\in [K]}\sum_{t\in\mathcal{T}_\mtep}\hat{\ell}_{t,i}}\right\} + \kappa\right)$. Plugging this into \pref{eq:etalambdastar} completes the proof.
\end{proof}

Now we are ready to prove \pref{thm:Corralthm}.
\begin{proof}{\textbf{of \pref{thm:Corralthm}}}
    We first show that with high probability, the algorithm does not enter the second stage if $\alpha T\ge 64\kappa L_\star$. Let $\mtep^\star\le\log_2 T$ be the number of meta-epochs executed in the first stage. Define random variables as follows:
    \begin{align*}
        R_\mtep\triangleq\min_{i\in [K]}\sum_{t\in\mathcal{T}_\mtep}\hat{\ell}_{t,i},\quad \mtep\in [\mtep^\star].
    \end{align*}
    As we reset all the parameters for each meta-epoch, we have
    \begin{align*}
        \mathbb{E}\left[R_\mtep \,\middle\vert\, R_1,\dots,R_{\mtep-1}\right]\le \mathbb{E}\left[\min_{i\in [K]}\sum_{t=1}^T\hat{\ell}_{t,i}\right]\le \min_{i\in [K]}\mathbb{E}\left[\sum_{t=1}^T\hat{\ell}_{t,i}\right]=L_\star.
    \end{align*}
    On the other hand, according to \pref{line: eta is small}, the learning rate $\eta$ used at the last round of each meta-epoch $\mtep$ is at most $\sqrt{\frac{4}{\alpha T}}$ before being halved. Combining with \pref{line: update rule}, we have $\alpha T\le \frac{4}{\eta^2}\le 16\kappa R_\mtep$. Now suppose $\alpha T\ge 64\kappa L_\star$. Using Markov inequality, we have
    \begin{align*}
        \text{Prob}\left[R_\mtep\ge \frac{\alpha T}{16\kappa} \,\middle\vert\, R_1,\dots,R_{\mtep-1}\right]\le \frac{16\kappa L_\star}{\alpha T}\le \frac{1}{4},~\forall \mtep \in[m^\star].
    \end{align*}
    Therefore, the probability that the algorithm reaches the second stage when $\alpha T\ge 64\kappa L_\star$ is upper bounded as follows
    \begin{align*}
        \text{Prob}\left[\text{\pref{alg:ClippedEXP3.G} reaches the second stage}\right]&\le \prod_{\mtep=1}^{\log_2 T} \text{Prob}\left[R_\mtep\ge \frac{\alpha T}{16\kappa}\,\middle\vert\,R_1,\dots,R_{\mtep-1}\right]\\
        &\le \left(\frac{1}{4}\right)^{\log_2 T}=\frac{1}{T^2}.
    \end{align*}
    This shows that it is unlikely to reach the second stage when $\alpha T\ge 64\kappa L_\star$. Now consider the expected regret when $\alpha T\ge 64\kappa L_\star$. With probability $1-\frac{1}{T^2}$, the regret only comes from the first stage. According to \pref{lem:ClippedEXP3.Gthm}, we have
    \begin{align*}
        \Reg &\le \mathbb{E}\left[\sum_{\mtep=1}^{\mtep^\star}\left(\sum_{t\in \mathcal{T}_\mtep}\inner{p_t, \hat{\ell}_t}-\min_{i\in [K]}\sum_{t\in\mathcal{T}_\mtep}\hat{\ell}_{t,i}\right)\right]+\frac{1}{T^2}\cdot T\\
        &\le \mathcal{\tilde{O}}\left(\mathbb{E}\left[\min\left\{\sqrt{\alpha T}, \sqrt{\kappa\min_{i\in [K]}\sum_{t=1}^T\hat{\ell}_{t,i}}\right\}\right]+K\right)\\
        &\le \mathcal{\tilde{O}}\left(\min\left\{\sqrt{\alpha T}, \sqrt{\kappa L_\star}\right\}+K\right).
    \end{align*}
    On the other hand, when $\alpha T\le 64\kappa L_\star$, the regret is bounded by the sum of the worst-case regret from both stages. According to \pref{thm:alphaTEXP3} and \pref{lem:ClippedEXP3.Gthm}, we arrive at
    \begin{align*}
        \Reg &\le \mathbb{E}\left[\sum_{\mtep=1}^{\mtep^\star}\left(\sum_{t\in\mathcal{T}_\mtep}\inner{p_t, \hat{\ell}_t}-\min_{i\in [K]}\sum_{t\in \mathcal{T}_\mtep}\hat{\ell}_{t,i}\right)\right] +\mathcal{\tilde{O}}\left(\sqrt{\alpha T}+K^2\right)\\
        &\le \mathcal{\tilde{O}}\left(\sqrt{\alpha T}+K^2\right) = \mathcal{\tilde{O}}\left(\min\left\{\sqrt{\alpha T}, \sqrt{\kappa L_\star}\right\}+K^2\right).
    \end{align*}
    Combining the two cases completes the proof.
\end{proof}

\section{Omitted Details for \pref{sec:weakly}}
\label{app:LowerBound}\label{app:generalweak}
In this section, we provide omitted details for \pref{sec:weakly}, including the proof of \pref{thm:lowerbound} (\pref{app:lowerbound}), the proof of \pref{thm:adahomethm} (\pref{app:bipartite}), an adaptive version of \pref{alg:adahome} for \bigraph and its analysis (\pref{app:adaptivebipartite}), the proof of \pref{thm:adahome-general-thm} (\pref{app:GeneralCase}), and an adaptive version of \pref{alg:adahome} for general weakly observable graphs and its analysis (\pref{app:DoublingTrickGeneralCase}). For notational convenience, we use $\one_U$ to denote a vector in $\mathbb{R}^K$ whose $i$-th coordinate is $1$ if $i\in U$ and $0$ otherwise. We also define $h(x)\triangleq x-1-\ln x$ so that for a hybrid regularizer of the form $\psi(p)=\frac{1}{\eta}\sum_{{i}\in S}\ln \frac{1}{p_i}+\frac{1}{\bar{\eta}}\sum_{i\in \bar{S}}p_i\ln p_i$, its associated Bregman divergence is $D_\psi(p,q)=\frac{1}{\eta}\sum_{{i}\in S}h(\nicefrac{p_i}{q_i})+\frac{1}{\bar{\eta}}\sum_{i\in \bar{S}}p_i\ln (\nicefrac{p_i}{q_i})$.
\subsection{Proof of \pref{thm:lowerbound} }\label{app:lowerbound}
\begin{proof}
	  Fix any $\epsilon \in (0,1/3)$.
     Since the feedback graph is weakly observable, there must exist two nodes $u$ and $v$, such that $u$ does not have a self-loop and $v$ cannot observe $u$.
     Consider the following environment: $\forall t\in[T]$, $\ell_{t,u}=0,~\ell_{t,v}=T^{-b}$ for some $b$ such that $\epsilon < b < 1-2\epsilon$, and $\ell_{t,w}=1, \forall w\ne u,v$.
     We call nodes other than $u$ and $v$ ``bad arms''.
     Note that the loss of each node is constant over time and $L_\star=0$.
     
     Since $\mathcal{A}$ achieves $\otil(1)$ regret when $L_\star=0$, the expected number of times that bad arms are selected by $\mathcal{A}$ (denoted by $\mathcal{N}_{\text{bad}}$) is at most ${\mathcal{O}}(T^{\epsilon})$. 
     Similarly, the expected number of times that $v$ is selected by $\mathcal{A}$ (denoted by $\mathcal{N}_v$) is at most ${\mathcal{O}}( T^{b+\epsilon})$.
     The condition $\mathcal{N}_{\text{bad}}={\mathcal{O}}(T^{\epsilon})$ implies that we can find a interval of length $\frac{T}{2\mathcal{N}_{\text{bad}}+1}=\Omega(T^{1-\epsilon})$ such that $\mathcal{A}$ selects bad arms less than $\frac{1}{2}$ times in expectation in this interval.
     Therefore, by Markov inequality, we have with probability $\frac{1}{2}$, $\mathcal{A}$ does not select bad arms at all in this interval.
     
     Now consider creating another environment by switching the loss of $u$ to $1$ only in this interval. Note that with probability $\frac{1}{2}$, $\mathcal{A}$ cannot notice the change because $u$'s loss is not revealed  if none of the bad arms is selected.
     Since $\mathcal{N}_v ={\mathcal{O}}( T^{b+\epsilon})$, we conclude that $\mathcal{A}$ suffers expected loss ${{\Omega}}(T^{1-\epsilon}-\mathcal{N}_v)=\Omega(T^{1-\epsilon})$ in this interval using the condition $b < 1-2\epsilon$.
     Moreover, $v$ becomes the best arm in this new environment and $L_\star=T^{1-b}$.
     Therefore, $\mathcal{A}$ suffers expected regret ${{\Omega}}(T^{1-\epsilon} - T^{1-b}) = {{\Omega}}(T^{1-\epsilon})$ since $\epsilon<b$, which completes the proof.
\end{proof}

\subsection{Proofs for \pref{thm:adahomethm}}\label{app:bipartite}
We first prove \pref{lem:lemadahome}, which will be useful for the proofs of \pref{thm:adahomethm} and \pref{thm:adahome-general-thm}.
\begin{proof}{\textbf{of \pref{lem:lemadahome}}}
Let $\widetilde{p}_{t+1} = \argmin_{p\in\mathbb{R}_+^K}\left\{\inner{p, \hat{\ell}_t+a_t}+D_{\psi}(p, p_t)\right\}$. One can verify that $p_{t+1}=\argmin_{p\in \Omega}\left\{D_{\psi}(p,\widetilde{p}_{t+1})\right\}$ and for any $u\in \Omega$, we have
\begin{align*}
	&\inner{p_t-u, \hat{\ell}_t+a_t}\\
	&= D_{\psi}(u,p_t)-D_{\psi}(u, \widetilde{p}_{t+1})+D_{\psi}(p_t, \widetilde{p}_{t+1})\\
    &\le D_{\psi}(u,p_t)-D_{\psi}(u, p_{t+1})+D_{\psi}(p_t, \widetilde{p}_{t+1})\\
	&= D_{\psi}(u,p_t)-D_{\psi}(u, p_{t+1}) \\
	&\quad + \frac{1}{\eta}\sum_{i\in S}\left(\frac{p_{t,i}}{\widetilde{p}_{t+1,i}}-1-\ln\left(\frac{p_{t,i}}{\widetilde{p}_{t+1,i}}\right)\right)+ \frac{1}{\bar{\eta}}\sum_{i\in \bar{S}}\left(p_{t,i}\ln\frac{p_{t,i}}{\widetilde{p}_{t+1,i}}-\left(p_{t,i}-\widetilde{p}_{t+1,i}\right)\right),
\end{align*}
where the second inequality is by the generalized Pythagorean theorem. According to the choice of $\psi$, one can also obtain the close-form of $\widetilde{p}_{t+1}$, which satisfies
\begin{align*}
 &\frac{p_{t,i}}{\widetilde{p}_{t+1,i}}=1+\eta p_{t,i}(\hat{\ell}_{t,i}+a_{t,i}),\mbox{ for }i\in S;~\ln\left(\frac{p_{t,i}}{\widetilde{p}_{t+1,i}}\right)=\bar{\eta}(\hat{\ell}_{t,i}+a_{t,i}),\mbox{ for }i\in \bar{S}.
\end{align*}
Plugging this into the previous inequality shows
\begin{align*}
	&\inner{p_t-u, \hat{\ell}_t+a_t}\\
	&\le D_{\psi}(u,p_t)-D_{\psi}(u,p_{t+1})+ \frac{1}{\eta}\sum_{i\in S}\left(\eta p_{t,i}\left(\hat{\ell}_{t,i}+a_{t,i}\right)-\ln\left(1+\eta p_{t,i}\left(\hat{\ell}_{t,i}+a_{t,i}\right)\right)\right)\\
	&\quad+ \frac{1}{\bar{\eta}}\sum_{i\in \bar{S}}\left(\bar{\eta}p_{t,i}\left(\hat{\ell}_{t,i}+a_{t,i}\right)-p_{t,i}+p_{t,i}\exp\left(-\bar{\eta}\left(\hat{\ell}_{t,i}+a_{t,i}\right)\right)\right)
\end{align*}
Using the facts $\ln(1+x)\ge x-x^2$ and $\exp(-x)\le 1-x+x^2$ for any $x\ge 0$, and realizing $\hat{\ell}_{t,i}+a_{t,i}\ge 0$ holds, we further have
\begin{equation}
    \left<p_t-u, \hat{\ell}_t+a_t\right> 
    \le D_{\psi}(u, p_t)-D_{\psi}(u, p_{t+1})+\sum_{i\in S}\eta p_{t,i}^2\left(\hat{\ell}_{t,i}+a_{t,i}\right)^2+\sum_{i\in \bar{S}}\bar{\eta}p_{t,i}\left(\hat{\ell}_{t,i}+a_{t,i}\right)^2.\label{eq:lm6}
\end{equation}
Next we use the conditions for $\eta$ and $\bar{\eta}$ and the concrete form of $a_t$ to bound the last two terms as
\begin{align*}
        &\sum_{i\in S}\eta p_{t,i}^2\left(\hat{\ell}_{t,i}+a_{t,i}\right)^2+\sum_{i\in \bar{S}}\bar{\eta} p_{t,i}\left(\hat{\ell}_{t,i}+a_{t,i}\right)^2 \\
        &= \sum_{i\in S}\eta p_{t,i}^2\left(\hat{\ell}_{t,i}+2\eta p_{t,i}\hat{\ell}_{t,i}^2\right)^2+\sum_{i\in \bar{S}}\bar{\eta}p_{t,i}\left(\hat{\ell}_{t,i}+2\bar{\eta}\hat{\ell}_{t,i}^2\right)^2 \\
        &= \sum_{i\in S}\eta p_{t,i}^2\hat{\ell}_{t,i}^2\left(1+2\eta p_{t,i}\hat{\ell}_{t,i}\right)^2 + \sum_{i\in \bar{S}}\bar{\eta}p_{t,i}\hat{\ell}_{t,i}^2\left(1+2\bar{\eta}\hat{\ell}_{t,i}\right)^2\\
        &\le \sum_{i\in S}\eta p_{t,i}^2\hat{\ell}_{t,i}^2\left(1+2\eta\right)^2 + \sum_{i\in \bar{S}}\bar{\eta}p_{t,i}\hat{\ell}_{t,i}^2\left(1+\frac{2}{5}\right)^2 \tag{$p_{t,i}\hat{\ell}_{t,i}\le 1,i\in S;\bar{\eta}\hat{\ell}_{t,i}\le\frac{1}{5},i\in \bar{S}$}\\
        &\le \left<p_t,a_t\right>\tag{$\eta\le\frac{1}{5}$}.
    \end{align*}
    The proof is completed by plugging the inequality above into \pref{eq:lm6} and rearranging terms. 
\end{proof}
\begin{proof}{\textbf{of \pref{thm:adahomethm}}}
      The condition of \pref{lem:lemadahome} holds since according to the definition of $\bar{\eta}$ and $\Omega$, we have $\frac{\bar{\eta}}{W_{t,i}}\le\frac{\bar{\eta}}{\sqrt{\bar{\eta}}} =\sqrt{\bar{\eta}}\le\frac{1}{5},~\forall t\in[T]$ and $i\in\bar{S}$. Thus, by \pref{lem:lemadahome}, we have for any $u\in \Omega$
     \begin{align*}
         \left<p_t-u,\hat{\ell}_t\right>&\le D_{\psi}(u,p_t)-D_{\psi}(u,p_{t+1})+\left<u, a_t\right>.
     \end{align*}
     Summing over $t\in[T]$, we have:
    \begin{align*}
        \sum_{t=1}^T\left<p_t-u, \hat{\ell}_t\right>&\le D_{\psi}(u, p_1)+\sum_{t=1}^T\left<u, a_t\right>= \frac{1}{\eta}\sum_{i\in S}h\left(\frac{u_i}{p_{1,i}}\right)+\frac{1}{\bar{\eta}}\sum_{i\in \bar{S}}u_i\ln \frac{u_i}{p_{1,i}} + \sum_{t=1}^T\left<u, a_t\right>.
    \end{align*}
    When comparing to a node $i\in S$, we set $u=\frac{1}{T}\cdot \one_S+\left(1-\frac{\slpn }{T}\right)\cdot e_{i} \in \Omega$. Using the definition of $p_1$, we have
    \begin{align*}
    D_{\psi}(u, p_1)&=\frac{1}{\eta}\sum_{i\in S}h\left(\frac{u_i}{p_{1,i}}\right)+\frac{1}{\bar{\eta}}\sum_{i\in \bar{S}}u_i\ln\frac{u_i}{p_{1,i}}\\
    &=\frac{1}{\eta}\sum_{j\ne i, j\in S}h\left(\frac{2\slpn}{T}\right)+\frac{1}{\eta}h\left(2s\left(1-\frac{s-1}{T}\right)\right)\\
    &\le \frac{s-1}{\eta}\left(\frac{2s}{T}-1-\ln\frac{2s}{T}\right)+\frac{1}{\eta}h(2s) \\
    &= \frac{s-1}{\eta}\left(\frac{2s}{T}-1-\ln\frac{2s}{T}\right)+\frac{1}{\eta}\left(2s-1-\ln 2s\right) \\
    &\le \frac{(s-1)\ln T}{\eta}+\frac{2s-1-s\ln 2s}{\eta}\le \frac{s\ln T}{\eta},
     \end{align*}
    The first inequality is because $h(y)$ is increasing when $y\ge 1$ and $T\ge 2K\ge 2s$. The second inequality is also because $T\ge 2s$ and the last one is because $2s-1-s\ln 2s\le 1\le \ln T$ for all $s> 0$. Therefore, we have
    \begin{equation}
        \begin{split}
        \sum_{t=1}^T\left<p_t-e_{i}, \hat{\ell}_t\right> &\le \frac{\slpn \ln T}{\eta}+2\eta\sum_{t=1}^Tp_{t,i}\hat{\ell}_{t,i}^2+\frac{2\eta}{T}\sum_{t=1}^T\sum_{j\in S}p_{t,j}\hat{\ell}_{t,j}^2+\frac{1}{T}\sum_{t=1}^T\sum_{j\in S}\hat{\ell}_{t,j} \\
        &\le \frac{\slpn \ln T}{\eta}+2\eta\sum_{t=1}^T\hat{\ell}_{t,i}+\frac{2}{T}\sum_{t=1}^T\sum_{j\in S}\hat{\ell}_{t,j}, 
        \label{eq:special case S}
        \end{split}
    \end{equation}
    where the second inequality is because $p_{t,i}\hat{\ell}_{t,i}\le 1$ for all $i\in S$ and $\eta\le\frac{1}{5}\le\frac{1}{2}$. When comparing with node $i\in \bar{S}$, let $\hat{i}_S^\star=\argmin_{i\in S}\sum_{t=1}^T\hat{\ell}_{t,i}$ and we choose $u=\sqrt{\bar{\eta}}\cdot e_{\hat{i}_S^\star}+\frac{1}{T}\cdot \one_S + \left(1-\frac{\slpn }{T}-\sqrt{\bar{\eta}}\right)e_{i} \in \Omega$. According to the choice of $p_1$, we bound the Bregman divergence term as
     \begin{align*}
         &D_{\psi}(u, p_1)\\
         &=\frac{1}{\eta}\sum_{i\in S}h\left(\frac{u_i}{p_{1,i}}\right)+\frac{1}{\bar{\eta}}\sum_{i\in \bar{S}}u_i\ln\frac{u_i}{p_{1,i}}\\
         &=\frac{1}{\eta}\sum_{j\ne \hat{i}_S^\star, j\in S}h\left(\frac{2\slpn}{T}\right)+\frac{1}{\eta}h\left(2s\left(\frac{1}{T}+\sqrt{\bar{\eta}}\right)\right)+\frac{1}{\bar{\eta}}\left(1-\frac{\slpn }{T}-\sqrt{\bar{\eta}}\right)\ln\left(2\bar{s}\left(1-\frac{\slpn }{T}-\sqrt{\bar{\eta}}\right)\right)\\
         &\le \frac{s}{\eta}\left(\frac{2\slpn}{T}-1-\ln\left(\frac{2\slpn}{T}\right)\right)+\frac{1}{\eta}\left(\frac{2s}{T}+2s\sqrt{\bar{\eta}}-1-\ln\left(\frac{2s}{T}+2s\sqrt{\bar{\eta}}\right)\right)+\frac{1}{\bar{\eta}}\left(1-\frac{\slpn }{T}-\sqrt{\bar{\eta}}\right)\ln\left(2\bar{s}\right)\\
         &\le \frac{s\ln T}{\eta}+\frac{1}{\eta}\left(2s\sqrt{\bar{\eta}}+\ln T\right)+\frac{1}{\bar{\eta}}\ln\left(2\bar{s}\right)\\
         &\le \frac{\slpn\ln T+\slpn+\ln T}{\eta}+\frac{\ln (2\woslpn)}{\bar{\eta}}\le\frac{2s\ln T}{\eta}+\frac{\ln (2\woslpn)}{\bar{\eta}}.
    \end{align*}
    Here, the second inequality is because $T\ge 2K\ge 2\slpn$; the third inequality is because $\sqrt{\bar{\eta}}\le \frac{1}{5}\le \frac{1}{2}$; and the last inequality is because $\ln T\ge 1$ and $\slpn \ge 1$.
    Therefore, we have
    \begin{align}
        &\sum_{t=1}^T\left<p_t-e_{i},\hat{\ell}_t\right>\nonumber\\&\le\frac{2\slpn \ln T}{\eta}+\frac{\ln (2\woslpn)}{\bar{\eta}}+2\bar{\eta}\sum_{t=1}^T\hat{\ell}_{t,i}^2+\frac{2\eta}{T}\sum_{t=1}^T\sum_{j\in S}p_{t,j}\hat{\ell}_{t,j}^2+2\eta\sqrt{\bar{\eta}}\sum_{t=1}^Tp_{t,\hat{i}_{S}^\star}\hat{\ell}_{t,\hat{i}_{S}^\star}^2\nonumber\\
        &\quad +\sqrt{\bar{\eta}}\min_{i\in S}\sum_{t=1}^T\hat{\ell}_{t,i}+\frac{1}{T}\sum_{t=1}^T\sum_{j\in S}\hat{\ell}_{t,j} \nonumber\\
        &\le\frac{2\slpn \ln T}{\eta}+\frac{\ln (2\woslpn)}{\bar{\eta}}+2\sqrt{\bar{\eta}}\sum_{t=1}^T\hat{\ell}_{t,i}+2\sqrt{\bar{\eta}}\min_{i\in S}\sum_{t=1}^T\hat{\ell}_{t,i} + \frac{2}{T}\sum_{t=1}^T\sum_{j\in S}\hat{\ell}_{t,j},\label{eq:special case barS}
    \end{align}
    where the second inequality is because $\hat{\ell}_{t,i}\le\frac{1}{\sqrt{\bar{\eta}}}$ for nodes $i\in \bar{S}$, $p_{t,i}\hat{\ell}_{t,i}\le 1$ for all $i\in S$, $t\in [T]$ and $\eta\le \frac{1}{5}\le\frac{1}{2}$.

    Now we take expectation over both sides. If $i\in S$, then
    \begin{align*}
        \Reg_i = \mathbb{E}\left[\sum_{t=1}^T\left<p_t-e_{i}, \hat{\ell}_t\right>\right]\le\frac{\slpn \ln T}{\eta}+2\eta L_i + 2\slpn.
    \end{align*}
    If $i\in \bar{S}$, then
    \begin{align*}
        \Reg_i = \mathbb{E}\left[\sum_{t=1}^T\left<p_t-e_{i}, \hat{\ell}_t\right>\right] &\le\frac{2\slpn \ln T}{\eta}+\frac{2\ln K}{\bar{\eta}}+2\sqrt{\bar{\eta}}L_i+2\sqrt{\bar{\eta}}\mathbb{E}\left[\min_{i\in S}\sum_{t=1}^T\hat{\ell}_{t,i}\right] + 2\slpn \\
        &\le\frac{2\slpn \ln T}{\eta}+\frac{2\ln K}{\bar{\eta}}+2\sqrt{\bar{\eta}}L_i+2\sqrt{\bar{\eta}}\min_{i\in S}\mathbb{E}\left[\sum_{t=1}^T\hat{\ell}_{t,i}\right] + 2\slpn\\
        &=\frac{2\slpn \ln T}{\eta}+\frac{2\ln K}{\bar{\eta}}+2\sqrt{\bar{\eta}}L_i+2\sqrt{\bar{\eta}}L_{\isstar}+ 2\slpn.
    \end{align*}
    The first inequality is because $\ln (2\woslpn)\le \ln (2K)\le 2\ln K$ as $K\ge 2$, and the second inequality is because of Jensen's inequality. Finally, choosing $\eta=\min\{\sqrt{\nicefrac{\slpn }{L_{\isstar}}}, \nicefrac{1}{5}\}$, $\bar{\eta}=\min\{L_{\isstar}^{-\nicefrac{2}{3}}, \nicefrac{1}{25}\}$, we have for $i\in S$
    \begin{align*}
    	\Reg_i\le \Reg_{\isstar}\le \mathcal{O}\left(s\ln T+\sqrt{sL_{\isstar}}\right)\le\mathcal{\tilde{O}}\left(\sqrt{sL_i}+s\right).
    \end{align*}
    For $i\in \bar{S}$, if $L_i\le L_{\isstar}$, we have
    \begin{align*}
    	\Reg_i \le \frac{2s\ln T}{\eta}+\frac{2\ln K}{\bar{\eta}}+4\sqrt{\bar{\eta}}L_{\isstar}+2s\le \mathcal{\tilde{O}}\left(L_{\isstar}^{\nicefrac{2}{3}}+\sqrt{sL_{\isstar}}+s\right).
    \end{align*}
    Otherwise, we have
    \begin{align*}
    	\Reg_i &\le \Reg_{\isstar} \le \mathcal{\tilde{O}}\left(\sqrt{sL_{\isstar}}+s\right).
    \end{align*}
    Combining the above results finishes the proof.
\end{proof}

\subsection{Adaptive Version of \pref{alg:adahome} for \BiGraph}\label{app:adaptivebipartite}

\setcounter{AlgoLine}{0}
 \begin{algorithm}[t]\caption{Adaptive Version of \pref{alg:adahome} for Directed Complete Bipartite Graphs}\label{alg:CHROME}
 \textbf{Input:} Feedback graph $G$ and parameter $\eta\le \frac{1}{5}$.\\
\textbf{Initialize:} \initp\\
\For{$\lambda=1,2,\dots$}{

\nl $p_t=p_1$, $\bar{\eta}=\slpn ^{-\frac{2}{3}}\eta^{\frac{4}{3}}$, $T_{\lambda}=t-1$.

\nl Define decision set $\Omega =\left\{p\in\Delta(K): \sum_{i\in S}p_i\ge \sqrt{\bar{\eta} }\right\}$\label{line: Doubling Adahome init_1}.

\nl Define hybrid regularizer $\psi (p)=\frac{1}{\eta }\sum_{i\in S}\ln \frac{1}{p_i}+\frac{1}{\bar{\eta} }\sum_{i\in \bar{S}}p_i\ln p_i$.\label{line: Doubling Adahome init_2}

\nl \While{$t\le T$}{
	\nl \label{line: Doubling Adahome rec loss}
	Play arm $i_t\sim {p_t}$ and receive feedback $\ell_{t,i}$ for all $i$ such that $i_t\in \Nin(i)$.
	
	\nl \label{line: Doubling Adahome construct loss}
	Construct estimator $\hat{\ell}_t$ such that $\hat{\ell}_{t,i}=\begin{cases}
	0, &\mbox{$p_{t,i}=0$},\\
	\frac{\ell_{t,i}}{W_{t,i}}\cdot\mathbbm{1}\{i_t\in\Nin(i)\}, &\mbox{$p_{t,i}>0$,}
	\end{cases}
	$,\\ where $W_{t,i}=\sum_{j\in \Nin (i)}p_{t,j}$. 
	
	\nl \label{line:Doubling Adahome a}
	Construct correction term $a_t$ such that 
	$a_{t,i} =\begin{cases}
	2\eta  p_{t,i}\hat{\ell}_{t,i}^2, &\mbox{for $i\in S$},\\
	2\bar{\eta} \hat{\ell}_{t,i}^2,  &\mbox{for $i\in \bar{S}$}.
	\end{cases}
	$
	
	\nl \label{line: Doubling Adahome OMD-Variant}
	Compute $\hat{p}_{t+1}=\argmin_{p\in\Omega }\left\{\inner{p,\hat{\ell}_{t}+a_{t}}+{\brgmd}_{\psi }(p,\hat{p}_{t})\right\}$.
	
	\nl \label{line: Doubling Adahome Clipping} Construct $p_{t+1}$ as follows, where $\mu =\frac{\eta \sqrt{\bar{\eta} }}{\slpn}$ : 
	$$
	p_{t+1,i} =\begin{cases}
	\frac{\hat{p}_{t+1,i}\mathbbm{1}\{\hat{p}_{t+1,i}\ge\mu \}}{\sum_{i'\in S}\hat{p}_{t+1,i'}\cdot\mathbbm{1}\{\hat{p}_{t+1,i'}\ge\mu \}}\cdot\sum_{i'\in S}\hat{p}_{t+1,i'},&\mbox{if $i\in S$},\\
\hat{p}_{t+1,i},&\mbox{if $i\in \bar{S}$}.
\end{cases}
$$
\nl \label{line: Doubling Adahome Update Rule} \If{$\frac{\slpn}{\eta }\le \eta \min_{i\in S}\sum_{\tau = T_{\lambda}+1}^t\hat{\ell}_{t,i}$}{\nl \label{line: Doubling Adahome Reset} $\eta \leftarrow \eta/2$, $t\leftarrow t+1$.\\ \textbf{Break}.}
$t\leftarrow t+1$.
}}
\end{algorithm}

In order to make \pref{alg:adahome} parameter-free, one may consider directly applying doubling trick.
However, one technical issue comes from analyzing the last round before each restart where the loss estimator might be too large.
To address this issue, we combine \pref{alg:adahome} and the clipping technique, together with a double trick. 

The full algorithm is described in \pref{alg:CHROME}. 
Similar to previous doubling trick algorithms, we start from some large $\eta$ and $\bar{\eta}$, run the procedures of \pref{alg:adahome} (\pref{line: Doubling Adahome rec loss} to \pref{line: Doubling Adahome OMD-Variant}) and reduce the learning rate when the accumulated estimated loss is too large (\pref{line: Doubling Adahome Update Rule} and \pref{line: Doubling Adahome Reset}).
{The key difference is that we again follow the clipping idea of \citep{allenberg2006hannan}:
after computing $\hat{p}_{t+1}$ through OMD, we do clipping with threshold $\mu$ and renormalization for nodes in $S$ (\pref{line: Doubling Adahome Clipping}).
In this way, $\hat{\ell}_{t,i}$ defined in \pref{line: Doubling Adahome construct loss} is well upper bounded for all $i\in [K]$, $t\in [T]$, which is crucial for the doubling trick analysis. Formally, we prove the following theorem.
\begin{theorem}
   \pref{alg:CHROME} with $\eta=\min\{\frac{1}{5}, \frac{1}{\slpn}\}$ guarantees for any directed complete bipartite graph:
    \begin{align*}
\Reg = \begin{cases} \mathcal{\tilde{O}}\left(\sqrt{\slpn L_{\isstar}}+\slpn ^2\right), &\mbox{if } i^\star\in S.\\ 
\mathcal{\tilde{O}}\left(L_{\isstar}^{\nicefrac{2}{3}}+\sqrt{\slpn L_{\isstar}}+\slpn ^2\right), & \mbox{if } i^\star\in \bar{S}. \end{cases}
    \end{align*}
\end{theorem}
\begin{proof}
    We call the time steps between two resets an epoch (indexed by $\lambda$) and let $\eta_\lambda$, $\bar{\eta}_\lambda$, and $\mu_\lambda$ be the value of ${\eta}$, $\bar{\eta}$, and $\mu$ during epoch $\lambda$ so that $\eta_\lambda = 2^{1-\lambda}\eta_1$, $\bar{\eta}_\lambda={\slpn}^{-\nicefrac{2}{3}}\eta_{\lambda}^{\nicefrac{4}{3}}$, and $\mu_\lambda =\frac{\eta_\lambda \sqrt{\bar{\eta}_\lambda }}{\slpn}$.
        Also let $\lambda^\star$ be the index of the last epoch. As we only do clipping restricted on the nodes in $S$, all nodes in $\bar{S}$ can still be observed with probability greater than zero. Therefore, $\hat{\ell}_{t,i}$ is still unbiased for any node $i\in \bar{S}$ and we have $\mathbb{E}\left[\ell_{t,i_t}\right]=\mathbb{E}[\langle p_t,\hat{\ell}_t\rangle]$. In addition, in each epoch $\lambda$, as the clipping threshold is $\mu_\lambda\le\frac{\sqrt{\bar{\eta}_{\lambda}}}{\slpn }$, at least one node in $S$ will survive and we have $1\le \frac{p_{t,i_t}}{\hat{p}_{t,i_t}}\le\frac{1}{1-\frac{\slpn\mu_{\lambda}}{\sqrt{\bar{\eta}_{\lambda}}}}=\frac{1}{1-\eta_{\lambda}}$.
    
    Now we consider the regret in epoch $\lambda$. We will prove that
    \begin{equation}
        \frac{1}{1-\eta_{\lambda}}\left(\sum_{t=T_{\lambda}+1}^{T_{\lambda+1}}\left<\hat{p}_t,\hat{\ell}_t\right>\right) - \sum_{t=T_{\lambda}+1}^{T_{\lambda+1}}\hat{\ell}_{t,i^\star} \le 
        \begin{cases}
            \mathcal{\tilde{O}}\left(\frac{\slpn}{\eta_{\lambda}}+\frac{1}{T}\sum_{t=T_{\lambda}+1}^{T_{\lambda+1}}\sum_{i\in S}\hat{\ell}_{t,i}\right), &\mbox{if $i^\star\in S$},\\
            \mathcal{\tilde{O}}\left(\frac{\slpn}{\eta_{\lambda}}+\frac{1}{\bar{\eta}_{\lambda}}+\frac{1}{T}\sum_{t=T_{\lambda}+1}^{T_{\lambda+1}}\sum_{i\in S}\hat{\ell}_{t,i}\right), &\mbox{if $i^\star\in \bar{S}$}.\label{eq:doubling trick eq1}\\
        \end{cases}
    \end{equation}
    When $i^\star\in S$, we have
    \begin{align*}
    & \frac{1}{1-\eta_{\lambda}}\left(\sum_{t=T_{\lambda}+1}^{T_{\lambda+1}}\left<\hat{p}_t,\hat{\ell}_t\right>\right) - \sum_{t=T_{\lambda}+1}^{T_{\lambda+1}}\hat{\ell}_{t,i^\star} \\
    &\le \frac{1}{1-\eta_{\lambda}}\left(\sum_{t=T_{\lambda}+1}^{T_{\lambda+1}}\left<\hat{p}_t,\hat{\ell}_t\right>\right) - \min_{i\in S}\sum_{t=T_{\lambda}+1}^{T_{\lambda+1}}\hat{\ell}_{t,i} \\
    &\le\mathcal{\tilde{O}}\left(\frac{\slpn}{\eta_{\lambda}}+\eta_{\lambda}\min_{i\in S}\sum_{t=T_{\lambda}+1}^{T_{\lambda+1}}\hat{\ell}_{t,i}+\frac{1}{T}\sum_{t=T_{\lambda}+1}^{T_{\lambda+1}}\sum_{i\in S}\hat{\ell}_{t,i}\right)\\
    &\le\mathcal{\tilde{O}}\left(\frac{\slpn }{\eta_{\lambda}}+\eta_{\lambda}\min_{i\in S}\sum_{t=T_{\lambda}+1}^{T_{\lambda+1}-1}\hat{\ell}_{t,i}+\eta_{\lambda}\max_{i\in S}\hat{\ell}_{T_{\lambda+1},i}+\frac{1}{T}\sum_{t=T_{\lambda}+1}^{T_{\lambda+1}}\sum_{i\in S}\hat{\ell}_{t,i}\right) \\
    &\le\mathcal{\tilde{O}}\left(\frac{\slpn}{\eta_{\lambda}}+\slpn ^{\frac{4}{3}}\eta_{\lambda}^{-\frac{2}{3}}+\frac{1}{T}\sum_{t=T_{\lambda}+1}^{T_{\lambda+1}}\sum_{i\in S}\hat{\ell}_{t,i}\right)\\
    &\le\mathcal{\tilde{O}}\left(\frac{\slpn }{\eta_{\lambda}}+\frac{1}{T}\sum_{t=T_{\lambda}+1}^{T_{\lambda+1}}\sum_{i\in S}\hat{\ell}_{t,i}\right).
\end{align*}
   The second inequality is derived by rearranging terms in \pref{eq:special case S}. The fourth inequality is because $\frac{\slpn}{\eta_{\lambda} }\le \eta_{\lambda} \min_{i\in S}\sum_{t = T_{\lambda}+1}^{T_{\lambda+1}-1}\hat{\ell}_{t,i}$ does not hold and $\hat{\ell}_{T_{\lambda+1},i}\le \frac{1}{\mu_{\lambda}}$ holds for all $i\in [K]$. The last inequality is because $\eta_{\lambda}\le \eta_1\le\frac{1}{\slpn}$. On the other hand, if $i^\star\in \bar{S}$, we have
    \begin{align*}
        & \frac{1}{1-\eta_{\lambda}}\left(\sum_{t=T_\lambda+1}^{T_{\lambda+1}}\left<\hat{p}_t, \hat{\ell}_t\right>\right)-\sum_{t=T_\lambda+1}^{T_{\lambda+1}}\hat{\ell}_{t,i^\star} \\
        &\le \mathcal{\tilde{O}}\left(\frac{\slpn }{\eta_{\lambda}}+\frac{1}{\bar{\eta}_{\lambda}}+\left(\sqrt{\bar{\eta}_{\lambda}}+\eta_{\lambda}\right)\sum_{t=T_\lambda+1}^{T_{\lambda+1}}\hat{\ell}_{t,i^\star}+\sqrt{\bar{\eta}_{\lambda}}\min_{i\in S}\sum_{t=T_{\lambda}+1}^{T_{\lambda+1}}\hat{\ell}_{t,i}+\frac{1}{T}\sum_{t=T_{\lambda}+1}^{T_{\lambda+1}}\sum_{i\in S}\hat{\ell}_{t,i}\right).
    \end{align*}
    which is also derived by rearranging terms in \pref{eq:special case barS}. Then we consider the following two cases. If $\sum_{t=T_{\lambda}+1}^{T_{\lambda+1}}\hat{\ell}_{t,i^\star}\le \min_{i\in S}\sum_{t=T_{\lambda}+1}^{T_{\lambda+1}}\hat{\ell}_{t,i}$, we have
    \begin{align*}
    & \frac{1}{1-\eta_{\lambda}}\left(\sum_{t=T_\lambda+1}^{T_{\lambda+1}}\left<\hat{p}_t, \hat{\ell}_t\right>\right)-\sum_{t=T_\lambda+1}^{T_{\lambda+1}}\hat{\ell}_{t,i^\star} \\
    &\le\mathcal{\tilde{O}}\left(\frac{\slpn }{\eta_{\lambda}}+\frac{1}{\bar{\eta}_{\lambda}}+\left(\sqrt{\bar{\eta}_{\lambda}}+\eta_{\lambda}\right)\min_{i\in S}\sum_{t=T_\lambda+1}^{T_{\lambda+1}}\hat{\ell}_{t,i}+\frac{1}{T}\sum_{t=T_\lambda+1}^{T_{\lambda+1}}\sum_{i\in S}\hat{\ell}_{t,i}\right)\\
    &\le  \mathcal{\tilde{O}}\left(\frac{\slpn}{\eta_{\lambda}}+\frac{1}{\bar{\eta}_{\lambda}}+\left(\sqrt{\bar{\eta}_{\lambda}}+\eta_{\lambda}\right)\min_{i\in S}\sum_{t=T_{\lambda}+1}^{T_{\lambda+1}-1}\hat{\ell}_{t,i}+\left(\sqrt{\bar{\eta}_{\lambda}}+\eta_{\lambda}\right)\max_{i\in S}\hat{\ell}_{T_{\lambda+1}, i}+\frac{1}{T}\sum_{t=T_{\lambda}+1}^{T_{\lambda+1}}\sum_{i\in S}\hat{\ell}_{t,i}\right)\\
    &\le  \mathcal{\tilde{O}}\left(\frac{\slpn}{\eta_{\lambda}}+\frac{1}{\bar{\eta}_{\lambda}}+\left(\sqrt{\bar{\eta}_{\lambda}}+\eta_{\lambda}\right)\max_{i\in S}\hat{\ell}_{T_{\lambda+1},i}+\frac{1}{T}\sum_{t=T_{\lambda}+1}^{T_{\lambda+1}}\sum_{i\in S}\hat{\ell}_{t,i}\right)\\
    &\le\mathcal{\tilde{O}}\left(\frac{\slpn}{\eta_{\lambda}}+\frac{1}{\bar{\eta}_{\lambda}}+\frac{\slpn }{\eta_{\lambda}}+\slpn ^{\frac{4}{3}}\eta_{\lambda}^{-\frac{2}{3}}+\frac{1}{T}\sum_{t=T_{\lambda}+1}^{T_{\lambda+1}}\sum_{i\in S}\hat{\ell}_{t,i}\right)\\
   &=  \mathcal{\tilde{O}}\left(\frac{\slpn }{\eta_{\lambda}}+\frac{1}{\bar{\eta}_{\lambda}}+\frac{1}{T}\sum_{t=T_{\lambda}+1}^{T_{\lambda+1}}\sum_{i\in S}\hat{\ell}_{t,i}\right).
\end{align*}
Here, the third inequality is because $\frac{\slpn}{\eta_{\lambda} }\le \eta_{\lambda} \min_{i\in S}\sum_{t = T_{\lambda}+1}^{T_{\lambda+1}-1}\hat{\ell}_{t,i}$ does not hold, which also implies that $\frac{1}{\eta_{\lambda} }\le \eta_{\lambda} \min_{i\in S}\sum_{t = T_{\lambda}+1}^{T_{\lambda+1}-1}\hat{\ell}_{t,i}$ does not hold; the fourth inequality is also $\hat{\ell}_{T_{\lambda+1},i}\le\frac{1}{\mu_{\lambda}}$ for all $i\in S$; and the last inequality is because $\eta_{\lambda}\le \eta_1\le \frac{1}{\slpn}$.

On the other hand, if $\sum_{t=T_{\lambda}+1}^{T_{\lambda+1}}\hat{\ell}_{t,i^\star}\ge \min_{i\in S}\sum_{t=T_{\lambda}+1}^{T_{\lambda+1}}\hat{\ell}_{t,i}$, then based on previous results, we have
\begin{align*}
     \frac{1}{1-\eta_{\lambda}}\left(\sum_{t=T_{\lambda}+1}^{T_{\lambda+1}}\left<\hat{p}_t,\hat{\ell}_t\right>\right) - \sum_{t=T_{\lambda}+1}^{T_{\lambda+1}}\hat{\ell}_{t,i^\star} &\le \frac{1}{1-\eta_{\lambda}}\left(\sum_{t=T_{\lambda}+1}^{T_{\lambda+1}}\left<\hat{p}_t,\hat{\ell}_t\right>\right) -\min_{i\in S}\sum_{t=T_{\lambda}+1}^{T_{\lambda+1}}\hat{\ell}_{t,i} \\
    &\le \mathcal{\tilde{O}}\left(\frac{\slpn }{\eta_{\lambda}}+\frac{1}{T}\sum_{t=T_{\lambda}+1}^{T_{\lambda+1}}\sum_{i\in S}\hat{\ell}_{t,i}\right).
\end{align*}
Combining the two cases, we finish proving \pref{eq:doubling trick eq1}. Now we sum up the regret over all epochs $\lambda=1,2,\dots,\lambda^\star$. For $i^\star\in S$, we have
\begin{align}
     \sum_{\lambda=1}^{\lambda^\star}\left(\frac{1}{1-\eta_{\lambda}}\sum_{t=T_{\lambda}+1}^{T_{\lambda+1}}\left<\hat{p}_t,\hat{\ell}_t\right>-\sum_{t=T_{\lambda}+1}^{T_{\lambda+1}}\hat{\ell}_{t,i^\star}\right) &\le  \sum_{\lambda=1}^{\lambda^\star}\mathcal{\tilde{O}}\left(\frac{\slpn }{\eta_{\lambda}}+\frac{1}{T}\sum_{t=T_{\lambda}+1}^{T_{\lambda+1}}\sum_{i\in S}\hat{\ell}_{t,i}\right)\nonumber\\
    &\le  \mathcal{\tilde{O}}\left(\frac{\slpn}{\eta_{\lambda^\star}}+\frac{1}{T}\sum_{t=1}^T\sum_{i\in S}\hat{\ell}_{t,i}\right)\nonumber\\
    &= \mathcal{\tilde{O}}\left(\frac{2^{\lambda^\star}\slpn}{\eta_1}+\frac{1}{T}\sum_{t=1}^T\sum_{i\in S}\hat{\ell}_{t,i}\right)\label{eq:generalS}
\end{align}
For $i^\star\in \bar{S}$, we have
\begin{align}
     \sum_{\lambda=1}^{\lambda^\star}\left(\frac{1}{1-\eta_{\lambda}}\sum_{t=T_{\lambda}+1}^{T_{\lambda+1}}\left<\hat{p}_t,\hat{\ell}_t\right>-\sum_{t=T_{\lambda}+1}^{T_{\lambda+1}}\hat{\ell}_{t,i^\star}\right) &\le  \sum_{\lambda=1}^{\lambda^\star}\mathcal{\tilde{O}}\left(\frac{\slpn }{\eta_{\lambda}}+\frac{1}{\bar{\eta}_{\lambda}}+\frac{1}{T}\sum_{t=T_{\lambda}+1}^{T_{\lambda+1}}\sum_{i\in S}\hat{\ell}_{t,i}\right)\nonumber\\
    &\le  \mathcal{\tilde{O}}\left(\frac{\slpn}{\eta_{\lambda^\star}}+\frac{1}{\bar{\eta}_{\lambda^\star}}+\frac{1}{T}\sum_{t=1}^T\sum_{i\in S}\hat{\ell}_{t,i}\right)\nonumber\\
    &= \mathcal{\tilde{O}}\left(\frac{2^{\lambda^\star}\slpn}{\eta_1}+\frac{2^{\frac{4}{3}\lambda^\star}}{\bar{\eta}_1}+\frac{1}{T}\sum_{t=1}^T\sum_{i\in S}\hat{\ell}_{t,i}\right).\label{eq:generalSbar}
\end{align}
Below we show that $\lambda^\star$ is well upper bounded. When $\lambda^\star\ge 2$, consider the last time step of epoch $\lambda^\star-1$, we have
\begin{align*}
    \frac{2^{2\lambda^\star-2}\slpn}{\eta_1^2}=\frac{\slpn}{\eta_{\lambda^\star-1}^2}\le \min_{i\in S}\sum_{t=T_{\lambda^\star-1}+1}^{T_{\lambda^\star}}\hat{\ell}_{t,i}\le \min_{i\in S}\sum_{t=1}^T\hat{\ell}_{t,i}.
\end{align*}
Therefore, we know that $2^{\lambda^\star}\le 2\eta_1\sqrt{\frac{\min_{i\in S}\sum_{t=1}^T\hat{\ell}_{t,i}}{\slpn}}$. Plugging this into \pref{eq:generalS}, we have for $i^\star\in S$,
\begin{align*}
     \sum_{\lambda=1}^{\lambda^\star}\left(\frac{1}{1-\eta_{\lambda}}\sum_{t=T_{\lambda}+1}^{T_{\lambda+1}}\left<\hat{p}_t,\hat{\ell}_t\right>-\sum_{t=T_{\lambda}+1}^{T_{\lambda+1}}\hat{\ell}_{t,i^\star}\right) &\le  \mathcal{\tilde{O}}\left(\frac{2^{\lambda^\star}\slpn}{\eta_1}+\frac{1}{T}\sum_{t=1}^T\sum_{i\in S}\hat{\ell}_{t,i}\right)\\
    &\le  \mathcal{\tilde{O}}\left(\sqrt{\slpn \min_{i\in S}\sum_{t=1}^T\hat{\ell}_{t,i}}+\frac{1}{T}\sum_{t=1}^T\sum_{i\in S}\hat{\ell}_{t,i}\right).
\end{align*}
On the other hand, for $i^\star\in \bar{S}$, plugging $2^{\lambda^\star}\le 2\eta_1\sqrt{\frac{\min_{i\in S}\sum_{t=1}^T\hat{\ell}_{t,i}}{\slpn}}$ into \pref{eq:generalSbar} gives
\begin{align*}
    &\sum_{\lambda=1}^{\lambda^\star}\left(\frac{1}{1-\eta_{\lambda}}\sum_{t=T_{\lambda}+1}^{T_{\lambda+1}}\left<\hat{p}_t,\hat{\ell}_t\right>-\sum_{t=T_{\lambda}+1}^{T_{\lambda+1}}\hat{\ell}_{t,i^\star}\right)\\ &\le\mathcal{\tilde{O}}\left(\frac{2^{\lambda^\star}\slpn}{\eta_1}+\frac{2^{\frac{4}{3}\lambda^\star}}{\bar{\eta}_1}+\frac{1}{T}\sum_{t=1}^T\sum_{i\in S}\hat{\ell}_{t,i}\right)\\
    &\le \mathcal{\tilde{O}}\left(\sqrt{\slpn \min_{i\in S}\sum_{t=1}^T\hat{\ell}_{t,i}}+\left(\min_{i\in S}\sum_{t=1}^T\hat{\ell}_{t,i}\right)^{\frac{2}{3}}+\frac{1}{T}\sum_{t=1}^T\sum_{i\in S}\hat{\ell}_{t,i}\right).
\end{align*}
Combining with the case $\lambda^\star=1$, we have the following result
\begin{align*}
        &\frac{1}{1-\eta_{\lambda}}\left(\sum_{t=T_{\lambda}+1}^{T_{\lambda+1}}\left<\hat{p}_t,\hat{\ell}_t\right>\right) - \sum_{t=T_{\lambda}+1}^{T_{\lambda+1}}\hat{\ell}_{t,i^\star} \\
        &\le 
        \begin{cases}
            \mathcal{\tilde{O}}\left(\sqrt{\slpn \min_{i\in S}\sum_{t=1}^T\hat{\ell}_{t,i}}+\slpn^2+\frac{1}{T}\sum_{t=1}^T\sum_{i\in S}\hat{\ell}_{t,i}\right), &\mbox{if $i^\star\in S$},\\
            \mathcal{\tilde{O}}\left(\sqrt{\slpn \min_{i\in S}\sum_{t=1}^T\hat{\ell}_{t,i}}+\left(\min_{i\in S}\sum_{t=1}^T\hat{\ell}_{t,i}\right)^{\frac{2}{3}}+\slpn^2+\frac{1}{T}\sum_{t=1}^T\sum_{i\in S}\hat{\ell}_{t,i}\right), &\mbox{if $i^\star\in \bar{S}$}.\\
        \end{cases}
    \end{align*}
Now we take the expectation over both sides. First, for the left hand side, we have

\begin{align*}
     & \mathbb{E}\left[\sum_{\lambda=1}^{\lambda^\star}\left(\frac{1}{1-\eta_{\lambda}}\sum_{t=T_{\lambda}+1}^{T_{\lambda+1}}\left<\hat{p}_t,\hat{\ell}_t\right>-\sum_{t=T_{\lambda}+1}^{T_{\lambda+1}}\hat{\ell}_{t,i^\star}\right)\right] \\
    &= \mathbb{E}\left[\sum_{\lambda=1}^{\lambda^\star}\left(\frac{1}{1-\eta_{\lambda}}\sum_{t=T_{\lambda}+1}^{T_{\lambda+1}}\left<\hat{p}_t,\hat{\ell}_t\right>-\sum_{t=T_{\lambda}+1}^{T_{\lambda+1}}\ell_{t,i_t}\right)\right] + \mathbb{E}\left[\sum_{t=1}^T\ell_{t,i_t}-\hat{\ell}_{t,i^\star}\right] \\
    &= \mathbb{E}\left[\sum_{t=1}^T\frac{1}{1-\eta_{\lambda_t}}\left<\hat{p}_t,  \hat{\ell}_t\right>-\ell_{t,i_t}\right]+\mathbb{E}\left[\sum_{t=1}^T\ell_{t,i_t}-\ell_{t,i^\star}\right]\\
    &= \sum_{t=1}^T\mathbb{E}_{\lambda_t}\left[\mathbb{E}_{i_t|\lambda_t}\left[\frac{1}{1-\eta_{\lambda_t}}\left<\hat{p}_t,\hat{\ell}_t\right>-\ell_{t,i_t}\right]\right]+\mathbb{E}\left[\sum_{t=1}^T\ell_{t,i_t}-\ell_{t,i^\star}\right]\\
    &\ge \mathbb{E}\left[\sum_{t=1}^T\ell_{t,i_t}-\ell_{t,i^\star}\right].
\end{align*}
Here, $\lambda_t$ represents the epoch that time $t$ belongs to. The last inequality is because of the fact that
whether $t$ is in epoch $\lambda$ or not is independent of what action is realized in time $t$ and $1\le\frac{p_{t,i_t}}{\hat{p}_{t,i_t}}\le \frac{1}{1-\eta_{\lambda}}$. 

Next we consider the right hand side. For $i^\star\in S$, we have
\begin{align*}
    \mathbb{E}\left[\mathcal{\tilde{O}}\left(\sqrt{\slpn \min_{i\in S}\sum_{t=1}^T\hat{\ell}_{t,i}}+\slpn^2+\frac{1}{T}\sum_{t=1}^T\sum_{i\in S}\hat{\ell}_{t,i}\right)\right]&\le \mathcal{\tilde{O}}\left(\left[\sqrt{\slpn \min_{i\in S}\mathbb{E}\left[\sum_{t=1}^T\hat{\ell}_{t,i}\right]}\right]+\slpn^2 \right)\\
    &= \mathcal{\tilde{O}}\left(\sqrt{\slpn L_{\isstar}}+\slpn^2 \right).
\end{align*}
where we use Jensen's inequality. For $i^\star\in \bar{S}$, we have
\begin{align*}
    & \mathbb{E}\left[\mathcal{\tilde{O}}\left(\sqrt{\slpn \min_{i\in S}\sum_{t=1}^T\hat{\ell}_{t,i}}+\left(\min_{i\in S}\sum_{t=1}^T\hat{\ell}_{t,i}\right)^{\frac{2}{3}}+\slpn^2+\frac{1}{T}\sum_{t=1}^T\sum_{i\in S}\hat{\ell}_{t,i}\right)\right]\\
    &\le \mathcal{\tilde{O}}\left(\sqrt{\slpn \min_{i\in S}\mathbb{E}\left[\sum_{t=1}^T\hat{\ell}_{t,i}\right]}+\left(\min_{i\in S}\mathbb{E}\left[\sum_{t=1}^T\hat{\ell}_{t,i}\right]\right)^{\frac{2}{3}}+\slpn^2 \right)\\
    &= \mathcal{\tilde{O}}\left(\sqrt{\slpn L_{\isstar}}+L_{\isstar}^{\nicefrac{2}{3}}+\slpn^2\right).
\end{align*}
Finally combining the results above proves the theorem statement:
\begin{align*}
\Reg = \begin{cases} \mathcal{\tilde{O}}\left(\sqrt{\slpn L_{\isstar}}+\slpn ^2\right), &\mbox{if $i^\star\in S$},\\ 
\mathcal{\tilde{O}}\left(L_{\isstar}^{\nicefrac{2}{3}}+\sqrt{\slpn L_{\isstar}}+\slpn ^2\right), & \mbox{if $i^\star\in \bar{S}$}. 
\end{cases}
\end{align*}
\end{proof}


\subsection{Proof of \pref{thm:adahome-general-thm}}\label{app:GeneralCase}
\begin{proof}
    The condition of \pref{lem:lemadahome} holds since according to the choice of $\bar{\eta}$ and $\delta$, we have $\frac{\bar{\eta}}{W_{t,i}}\le\frac{\bar{\eta}}{\delta}\le \delta^{\nicefrac{1}{3}} \le\frac{1}{5},~\forall t\in[T],~i\in\bar{S}$. Therefore, we know that for any $u\in \Omega$,
    \begin{align*}
        \sum_{t=1}^T\inner{p_t-u, \hat{\ell}_t}\le \frac{1}{\eta}\sum_{i\in S}h\left(\frac{u_i}{p_{1,i}}\right)+\frac{1}{\bar{\eta}}\sum_{i\in \bar{S}}\left(u_i\ln \frac{u_i}{p_{1,i}}\right)+\sum_{t=1}^T\inner{u, a_t}.
    \end{align*}
    Set $u = \frac{1}{T}\cdot \one_{S\backslash\wklyds}+\delta\cdot\one_{\wklyds}+\left(1-\wklydn\delta-\frac{|S\backslash\wklyds|}{T}\right)\cdot e_{i}$.  When comparing with $i\in S$, we have
    
    \begin{align*}
    &\sum_{j\in S}h\left(\frac{u_j}{p_{1,j}}\right)+\sum_{j\in \bar{S}}\left(u_j\ln \frac{u_j}{p_{1,j}}\right)\\ &= \sum_{j\ne i, j\in S\backslash\wklyds}\left(\frac{2\slpn}{T}-1-\ln\frac{2\slpn}{T}\right)+\sum_{j\ne i, j\in S\cap \wklyds}\left(2\slpn\delta-1-\ln 2\slpn\delta\right) \\
    &\quad + \left(2s\left(1-\wklydn\delta-\frac{|S\backslash\wklyds|}{T}\right)-1-\ln 2s\left(1-\wklydn\delta-\frac{|S\backslash\wklyds|}{T}\right) \right)+|\bar{S}\cap \mathcal{D}|\left(\delta\ln(2\woslpn \delta)\right)\\
    &\le (\slpn-1)\ln\frac{T}{2\slpn}+2\slpn-1-\ln\frac{\slpn}{2} \tag{$T\ge 2K\ge 2\slpn$ and $\frac{1}{T}\le\delta\le \min\left\{\frac{1}{4\wklydn}, \frac{1}{4\slpn}\right\}$}\\
    &= (\slpn-1)\ln T-\slpn\ln 2\slpn+2\slpn+\ln 4-1\\
    &\le 2\slpn \ln T.\tag{$-\slpn\ln 2\slpn +2\slpn +\ln 4-1\le 2\le (\slpn+1)\ln T$}
    \end{align*}
    Therefore,
    \begin{align*}
        \sum_{t=1}^T\inner{p_t-e_{i}, \hat{\ell}_t} &\le  \frac{2\slpn\ln T}{\eta}+2\eta\sum_{t=1}^Tp_{t,i}\hat{\ell}_{t,i}^2+\delta\sum_{t=1}^T\left(2\eta\sum_{j\in S\cap\wklyds}p_{t,j}\hat{\ell}_{t,j}^2+2\bar{\eta}\sum_{j\in \bar{S}\cap\wklyds}\hat{\ell}_{t,j}^2\right)\\
        &\quad+\frac{2\eta}{T}\sum_{t=1}^T\sum_{j\in S\backslash\wklyds}p_{t,j}\hat{\ell}_{t,j}^2+\delta\sum_{t=1}^T\sum_{j\in \wklyds}\hat{\ell}_{t,j}+\frac{1}{T}\sum_{t=1}^T\sum_{j\in S\backslash\wklyds}\hat{\ell}_{t,j} \\
        &\le \frac{2\slpn\ln T}{\eta}+2\eta\sum_{t=1}^T\hat{\ell}_{t,i}+2\delta\sum_{t=1}^T\sum_{j\in \wklyds}\hat{\ell}_{t,j}+\frac{2}{T}\sum_{t=1}^T\sum_{j\in S}\hat{\ell}_{t,j}.
    \end{align*}
    The second inequality is because $p_{t,j}\hat{\ell}_{t,j}\le 1$ for $j\in S$, $\hat{\ell}_{t,j}\le \frac{1}{\delta}$ for $j\in \bar{S}$, $2\bar{\eta}\le \delta$ and $2\eta\le 1$. The reason that $\hat{\ell}_{t,j}\le \frac{1}{\delta}$ holds for $j\in \bar{S}$ is that if $i$ is weakly observable, this trivially holds according to the definition of $\Omega$. Otherwise, $j$ can be observed by all the other nodes, which include at least one weakly observable node. This shows that $\hat{\ell}_{t,j}\le \frac{1}{\delta}$. 
    
    On the other hand, when comparing with $i\in \bar{S}$, we have
         \begin{align*}
         	&\sum_{j\in S}h\left(\frac{u_j}{p_{1,j}}\right)+\sum_{j\in \bar{S}}\left(u_j\ln \frac{u_j}{p_{1,j}}\right)\\ &= \sum_{j\in S\backslash\wklyds}\left(\frac{2\slpn}{T}-1-\ln\frac{2\slpn}{T}\right)+\sum_{j\in S\cap \wklyds}\left(2\slpn\delta-1-\ln 2\slpn\delta\right)+\left(|\bar{S}\cap \mathcal{D}|-1\right)\left(\delta\ln(2\woslpn\delta)\right)+u_i\ln(2\woslpn u_i)\\
         	&\le\slpn \ln\frac{T}{2\slpn}+\ln (2\woslpn) \le \slpn\ln T+\ln(2\woslpn)
         \end{align*}
    because $T\ge 2K\ge 2\slpn$, $\frac{1}{T}\le \delta\le \frac{1}{4\slpn}$ and $u_i\le 1$. Therefore,
    \begin{align*}
        \sum_{t=1}^T\inner{p_t-e_{i}, \hat{\ell}_t} &\le  \frac{{\slpn}\ln T}{\eta}+\frac{\ln(2\woslpn)}{\bar{\eta}}+2\eta\sum_{t=1}^T\hat{\ell}_{t,i}^2+\delta\sum_{t=1}^T\left(2\eta\sum_{j\in S\cap\wklyds}p_{t,j}\hat{\ell}_{t,j}^2+2\bar{\eta}\sum_{j\in \bar{S}\cap\wklyds}\hat{\ell}_{t,j}^2\right)\\
        &\quad+\frac{2\eta}{T}\sum_{t=1}^T\sum_{j\in S\backslash\wklyds}p_{t,j}\hat{\ell}_{t,j}^2+\delta\sum_{t=1}^T\sum_{j\in \wklyds}\hat{\ell}_{t,j}+\frac{1}{T}\sum_{t=1}^T\sum_{j\in S\backslash\wklyds}\hat{\ell}_{t,j} \\
        &\le \frac{\slpn\ln T}{\eta}+\frac{\ln(2\woslpn)}{\bar{\eta}}+\frac{2\eta}{\delta}\sum_{t=1}^T\hat{\ell}_{t,i}+2\delta\sum_{t=1}^T\sum_{j\in \wklyds}\hat{\ell}_{t,j}+\frac{2}{T}\sum_{t=1}^T\sum_{j\in S}\hat{\ell}_{t,j},
    \end{align*}
    where the second inequality holds by the same reasons for the case $i\in S$.
    Taking expectation over both sides, we have for $i\in S$
    \begin{align*}
        \Reg_i&\le\frac{2\slpn\ln T}{\eta}+2\eta\sum_{t=1}^T\mathbb{E}\left[\hat{\ell}_{t,i}\right]+2\delta\sum_{t=1}^T\sum_{j\in \wklyds}\mathbb{E}\left[\hat{\ell}_{t,j}\right] + \frac{2}{T}\sum_{t=1}^T\sum_{j\in S}\mathbb{E}\left[\hat{\ell}_{t,j}\right]\\
        &\le \frac{2\slpn\ln T}{\eta}+2\eta L_{i} + 2\delta\wklydn\LC + 2\slpn.
    \end{align*}
    Similarly for $i\in \bar{S}$, we have
    \begin{align*}
        \Reg_i&\le \frac{\slpn\ln T}{\eta}+\frac{\ln(2 \woslpn)}{\bar{\eta}}+\frac{2\bar{\eta}}{\delta}\mathbb{E}\left[\sum_{t=1}^T\hat{\ell}_{t,i}\right]+2\delta\sum_{t=1}^T\sum_{j\in \wklyds}\mathbb{E}\left[\hat{\ell}_{t,j}\right] + \frac{2}{T}\sum_{t=1}^T\sum_{j\in S}\mathbb{E}\left[\hat{\ell}_{t,j}\right]\\
        &\le \frac{\slpn\ln T}{\eta}+\frac{\ln(2 \woslpn)}{\bar{\eta}}+\frac{2\bar{\eta}L_{i}}{\delta} + 2\delta\wklydn\LC + 2\slpn.
    \end{align*}
    This finishes the proof.
\end{proof}

\subsection{Adaptive Version of \pref{alg:adahome} for General Weakly Observable Graphs}\label{app:DoublingTrickGeneralCase}

\setcounter{AlgoLine}{0}
\begin{algorithm}[t]
\caption{Adaptive Version for \pref{alg:adahome} for General Weakly Observable Graphs}\label{alg:adahome-general-doubling-trick}
\textbf{Input:} Feedback graph $G$ and parameter $\delta$.\\
\textbf{Initialize:} \initp\\
\nl
\For{$\lambda = 1,2,\dots$}{
    \nl $p_t=p_1$,~$\eta=\delta^{\frac{1}{2\gamma}}$,~$ \bar{\eta}=\delta^{\frac{1}{2}+\frac{1}{2\gamma}}$, $T_{\lambda}=t-1$.

	\nl Define decision set $\Omega = \{p\in \Delta(K):p_i\ge \delta,i\in\wklyds\}$.
	
	\nl Define hybrid regularizer $\psi(p)=\frac{1}{\eta}\sum_{i\in S}\ln \frac{1}{p_i}+\frac{1}{\bar{\eta}} \sum_{i\in \bar{S}}p_i\ln p_i$.
	
\nl \While{$t\le T$}{
\nl Execute \pref{line:alg2bg} to \pref{line:OMD_variant} of \pref{alg:adahome}.

\nl \If{$\delta^{-\frac{1}{\gamma}} \le \sum_{\tau=T_{\lambda}+1}^{t}\sum_{i\in \wklyds}\hat{\ell}_{t,i} $}{
\nl $\delta\leftarrow\frac{\delta}{2},~t \leftarrow t+1$.

\nl \textbf{Break}.}
\nl $t \leftarrow t + 1$.
}
}
\end{algorithm}

In this section, we introduce the adaptive version of \pref{alg:adahome} for general weakly observable graphs. For conciseness, we ignore the dependence on $\slpn$ and $\wklydn$ and prove that we can achieve the same result of \pref{thm:adahome-general-thm}. We also assume that the weakly dominating set $\wklyds$ contains at least one node in $S$ if $S$ is not empty. Otherwise, we can add an arbitrary node in $S$ to $\wklyds$ and the new set is still a weakly dominant set (just with size increased by 1). \pref{alg:adahome-general-doubling-trick} shows the full pseudocode. Unlike the \bigraph case, here we do not need the clipping technique. Formally, we have the following theorem. 

\begin{theorem}
    \pref{alg:adahome-general-doubling-trick} with $\delta=\min\{\frac{1}{125}, \frac{1}{4\wklydn}, \frac{1}{4\slpn}\}$ guarantees:
    \begin{align*}
        \Reg= \begin{cases} \mathcal{\tilde{O}}\left(\LC^{1-\gamma}\right), &\mbox{if } i^\star\in S.\\ 
        \mathcal{\tilde{O}}\left(\LC^{(1+\gamma)/2}\right), & \mbox{if } i^\star\in \bar{S}. \end{cases}
    \end{align*}
\end{theorem}
\begin{proof}
    We call the time steps between two resets an epoch (indexed by $\lambda$) and let $\delta_\lambda$, $\eta_\lambda$, and $\bar{\eta}_\lambda$ be the value of $\delta$, $\eta$, and $\bar{\eta}$ during epoch $\lambda$ so that $\delta_\lambda = 2^{1-\lambda}\delta_1$, $\eta_\lambda=\delta_\lambda^{\nicefrac{1}{2\gamma}}$, and $ \bar{\eta}_\lambda=\delta_\lambda^{\nicefrac{(1+\gamma)}{2\gamma}}$.
    Also let $\lambda^\star$ be the index of the last epoch. According to the analysis in \pref{thm:adahome-general-thm}, for fixed $\eta$, $\bar{\eta}$ and $\delta$, we have 
    \begin{align*}
        \sum_{t=1}^T\inner{p_t-e_{i^\star}, \hat{\ell}_t} \le \begin{cases} \mathcal{\tilde{O}}\left(\frac{1}{\eta}+\eta\sum_{t=1}^T\hat{\ell}_{t,i^\star}+\delta\sum_{t=1}^T\sum_{i\in \wklyds}\hat{\ell}_{t,i}+\frac{1}{T}\sum_{t=1}^T\sum_{i\in S}\hat{\ell}_{t,i}\right), &\mbox{if } i^\star\in S.\\ 
        \mathcal{\tilde{O}}\left(\frac{1}{\eta}+\frac{1}{\bar{\eta}}+\frac{\bar{\eta}}{\delta}\sum_{t=1}^T\hat{\ell}_{t,i^\star}+\delta\sum_{t=1}^T\sum_{i\in \wklyds}\hat{\ell}_{t,i}+\frac{1}{T}\sum_{t=1}^T\sum_{i\in S}\hat{\ell}_{t,i}\right), &\mbox{if } i^\star\in S.\\ \end{cases}
    \end{align*}
    Here, $\mathcal{\tilde{O}}(\cdot)$ suppresses all the constant factors, log factors and the dependence on $\slpn$ and $\wklydn$. Now we consider the regret of epoch $\lambda$. We will prove that
    \begin{equation}
        \begin{split}
            \sum_{t=T_{\lambda}+1}^{T_{\lambda+1}}\inner{p_t, \hat{\ell}_t}-\sum_{t=T_{\lambda}+1}^{T_{\lambda+1}}\hat{\ell}_{t,i^\star}\le
            \begin{cases}
                \mathcal{\tilde{O}}\left(\delta_{\lambda}^{1-\frac{1}{\gamma}}+\frac{1}{T}\sum_{t=T_{\lambda}+1}^{T_{\lambda+1}}\sum_{i\in S}\hat{\ell}_{t,i}\right), &\mbox{if $i^\star\in S$},\\
                \mathcal{\tilde{O}}\left(\delta_{\lambda}^{-\frac{1}{2}-\frac{1}{2\gamma}}+\frac{1}{T}\sum_{t=T_{\lambda}+1}^{T_{\lambda+1}}\sum_{i\in S}\hat{\ell}_{t,i}\right), &\mbox{if $i^\star\in \bar{S}$}.
            \end{cases}
        \end{split}\label{eq:general doubling trick}
    \end{equation}
    If $i^\star\in S$, we have from the proof of \pref{thm:adahome-general-thm}
    \begin{align*}
        \sum_{t=T_{\lambda}+1}^{T_{\lambda+1}}\inner{p_t, \hat{\ell}_t}-\sum_{t=T_{\lambda}+1}^{T_{\lambda+1}}\hat{\ell}_{t,i^\star}
        &\le \mathcal{\tilde{O}}\left(\frac{1}{\eta_{\lambda}}+\eta_{\lambda}\sum_{t=T_{\lambda}+1}^{T_{\lambda+1}}\hat{\ell}_{t,i^\star}+\delta_{\lambda}\sum_{t=T_{\lambda}+1}^{T_{\lambda+1}}\sum_{i\in \wklyds}\hat{\ell}_{t,i}+\frac{1}{T}\sum_{t=T_{\lambda}+1}^{T_{\lambda+1}}\sum_{i\in S}\hat{\ell}_{t,i}\right).
        \end{align*}
        Compare the terms $\sum_{t=T_{\lambda}+1}^{T_{\lambda+1}}\hat{\ell}_{t,i^\star}$ and $\sum_{t=T_{\lambda}+1}^{T_{\lambda+1}}\sum_{i\in \wklyds}\hat{\ell}_{t,i}$. If $\sum_{t=T_{\lambda}+1}^{T_{\lambda+1}}\hat{\ell}_{t,i^\star}\le \sum_{t=T_{\lambda}+1}^{T_{\lambda+1}}\sum_{i\in \wklyds}\hat{\ell}_{t,i}$, then we have
        \begin{align*}
        	&\sum_{t=T_{\lambda}+1}^{T_{\lambda+1}}\inner{p_t, \hat{\ell}_t}-\sum_{t=T_{\lambda}+1}^{T_{\lambda+1}}\hat{\ell}_{t,i^\star}\\
        	&\le \mathcal{\tilde{O}}\left(\frac{1}{\eta_{\lambda}}+\left(\delta_{\lambda}^{\frac{1}{2\gamma}}+\delta_{\lambda}\right)\sum_{t=T_{\lambda}+1}^{T_{\lambda+1}}\sum_{i\in \wklyds}\hat{\ell}_{t,i}+\frac{1}{T}\sum_{t=T_{\lambda}+1}^{T_{\lambda+1}}\sum_{i\in S}\hat{\ell}_{t,i}\right) \tag{$\eta_{\lambda}=\delta_{\lambda}^{\frac{1}{2\gamma}}$}\\
        & \le \mathcal{\tilde{O}}\left(\delta_\lambda^{-\frac{1}{2\gamma}}+\delta_{\lambda}\sum_{t=T_{\lambda}+1}^{T_{\lambda+1}}\sum_{i\in \wklyds}\hat{\ell}_{t,i}+\frac{1}{T}\sum_{t=T_{\lambda}+1}^{T_{\lambda+1}}\sum_{i\in S}\hat{\ell}_{t,i}\right) \tag{$\delta_{\lambda}\le 1$ and $\gamma\in [\nicefrac{1}{3}, \nicefrac{1}{2}]$}\\
        & \le \mathcal{\tilde{O}}\left(\delta_{\lambda}^{-\frac{1}{2\gamma}}+\delta_{\lambda}^{1-\frac{1}{\gamma}}+\delta_{\lambda}\sum_{i\in\wklyds}\hat{\ell}_{T_{\lambda+1},i}+\frac{1}{T}\sum_{t=T_{\lambda}+1}^{T_{\lambda+1}}\sum_{i\in S}\hat{\ell}_{t,i}\right) \tag{$\delta_{\lambda}^{-\nicefrac{1}{\gamma}}> \sum_{\tau=T_{\lambda}+1}^{T_{\lambda+1}-1}\sum_{i\in \wklyds}\hat{\ell}_{t,i}$}\\
        &\le \mathcal{\tilde{O}}\left(\delta_{\lambda}^{1-\frac{1}{\gamma}}+\frac{1}{T}\sum_{t=T_{\lambda}+1}^{T_{\lambda+1}}\sum_{i\in S}\hat{\ell}_{t,i}\right).
    \end{align*}
   	The last inequality is because for $i\in\mathcal{D}$, if $i$ has a self-loop, then $\hat{\ell}_{t,i}\le \frac{1}{p_{t,i}}\le \frac{1}{\delta_{\lambda}}$; if $i$ does not have a self-loop, it can be observed either by all the other nodes or at least one node in $\wklyds$, which also means that $\hat{\ell}_{t,i}\le\frac{1}{\delta_{\lambda}}$. Therefore, $\delta_{\lambda}\sum_{i\in \wklyds}\hat{\ell}_{T_{\lambda+1},i}\le \wklydn$. We also use $\delta_{\lambda}^{-\nicefrac{1}{2\gamma}}\le \delta_{\lambda}^{1-\nicefrac{1}{\gamma}}$ as $\delta_{\lambda}\le 1$ and $\gamma\in [\nicefrac{1}{3},\nicefrac{1}{2}]$.
   	
   	If $\sum_{t=T_{\lambda}+1}^{T_{\lambda+1}}\hat{\ell}_{t,i^\star}\ge \sum_{t=T_{\lambda}+1}^{T_{\lambda+1}}\sum_{i\in \wklyds}\hat{\ell}_{t,i}$, then let $i_0\in \wklyds$ be the node with a self-loop in $\wklyds$ and we have $\sum_{t=T_{\lambda}+1}^{T_{\lambda+1}}\hat{\ell}_{t,i^\star}\ge \sum_{t=T_{\lambda}+1}^{T_{\lambda+1}}\hat{\ell}_{t,i_0}$. Therefore,
   	\begin{align}
   		\sum_{t=T_{\lambda}+1}^{T_{\lambda+1}}\inner{p_t, \hat{\ell}_t}-\sum_{t=T_{\lambda}+1}^{T_{\lambda+1}}\hat{\ell}_{t,i^\star} &\le \sum_{t=T_{\lambda}+1}^{T_{\lambda+1}}\inner{p_t, \hat{\ell}_t}-\sum_{t=T_{\lambda}+1}^{T_{\lambda+1}}\hat{\ell}_{t,i_0} \nonumber\\
   		&\le \mathcal{\tilde{O}}\left(\frac{1}{\eta_{\lambda}}+\eta_{\lambda}\sum_{t=T_{\lambda}+1}^{T_{\lambda+1}}\hat{\ell}_{t,i_0}+\delta_{\lambda}\sum_{t=T_{\lambda}+1}^{T_{\lambda+1}}\sum_{i\in \wklyds}\hat{\ell}_{t,i}+\frac{1}{T}\sum_{t=T_{\lambda}+1}^{T_{\lambda+1}}\sum_{i\in S}\hat{\ell}_{t,i}\right)\nonumber\\
   		&\le \mathcal{\tilde{O}}\left(\delta_\lambda^{-\frac{1}{2\gamma}}+\delta_{\lambda}\sum_{t=T_{\lambda}+1}^{T_{\lambda+1}}\sum_{i\in \wklyds}\hat{\ell}_{t,i}+\frac{1}{T}\sum_{t=T_{\lambda}+1}^{T_{\lambda+1}}\sum_{i\in S}\hat{\ell}_{t,i}\right)\nonumber \\
   		&\le \mathcal{\tilde{O}}\left(\delta_{\lambda}^{1-\frac{1}{\gamma}}+\frac{1}{T}\sum_{t=T_{\lambda}+1}^{T_{\lambda+1}}\sum_{i\in S}\hat{\ell}_{t,i}\right).\label{eq:doubling weakly node S}
   	\end{align}
   	The third inequality is because $i_0\in \wklyds$ and also $\eta_{\lambda}\le \delta_{\lambda}$, and the last inequality is by the same reason as in the previous case.
   	
   	Next we consider the case $i^\star\in \bar{S}$. We have again from the proof of \pref{thm:adahome-general-thm}:
    \begin{align*}
        &\sum_{t=T_{\lambda}+1}^{T_{\lambda+1}}\inner{p_t, \hat{\ell}_t}-\sum_{t=T_{\lambda}+1}^{T_{\lambda+1}}\hat{\ell}_{t,i^\star} \\
        &\le \mathcal{\tilde{O}}\left(\frac{1}{\eta_{\lambda}}+\frac{1}{\bar{\eta}_{\lambda}}+\frac{\bar{\eta}_{\lambda}}{\delta_{\lambda}}\sum_{t=T_{\lambda}+1}^{T_{\lambda+1}}\hat{\ell}_{t,i^\star}+\delta_{\lambda}\sum_{t=T_{\lambda}+1}^{T_{\lambda+1}}\sum_{i\in \wklyds}\hat{\ell}_{t,i}+\frac{1}{T}\sum_{t=T_{\lambda}+1}^{T_{\lambda+1}}\sum_{i\in S}\hat{\ell}_{t,i}\right)\\
        &= \mathcal{\tilde{O}}\left(\delta_{\lambda}^{-\frac{1}{2}-\frac{1}{2\gamma}}+\delta_{\lambda}^{-\frac{1}{2}+\frac{1}{2\gamma}}\sum_{t=T_{\lambda}+1}^{T_{\lambda+1}}\hat{\ell}_{t,i^\star}+\delta_{\lambda}\sum_{t=T_{\lambda}+1}^{T_{\lambda+1}}\sum_{i\in \wklyds}\hat{\ell}_{t,i}+\frac{1}{T}\sum_{t=T_{\lambda}+1}^{T_{\lambda+1}}\sum_{i\in S}\hat{\ell}_{t,i}\right),
    \end{align*}
    where the last step is because $\eta_{\lambda}=\delta_{\lambda}^{\nicefrac{1}{2\gamma}}\ge \delta_{\lambda}^{\nicefrac{(\gamma+1)}{2\gamma}}=\bar{\eta}_{\lambda}$. Consider the following two cases. If $S=\emptyset$, then $\wklyds$ contains at least one node in $\bar{S}$, which means that $\min_{i\in \bar{S}}\sum_{t=T_{\lambda}+1}^{T_{\lambda+1}}\hat{\ell}_{t,i}\le \sum_{t=T_{\lambda}+1}^{T_{\lambda+1}}\sum_{i\in \wklyds}\hat{\ell}_{t,i}$. Therefore, following similar steps as in previous cases, we have
    \begin{align*}
    	&\sum_{t=T_{\lambda}+1}^{T_{\lambda+1}}\inner{p_t, \hat{\ell}_t}-\sum_{t=T_{\lambda}+1}^{T_{\lambda+1}}\hat{\ell}_{t,i^\star} \le \sum_{t=T_{\lambda}+1}^{T_{\lambda+1}}\inner{p_t, \hat{\ell}_t}-\min_{i\in \bar{S}}\sum_{t=T_{\lambda}+1}^{T_{\lambda+1}}\hat{\ell}_{t,i}\\
        &\le \mathcal{\tilde{O}}\left(\delta_{\lambda}^{-\frac{1}{2}-\frac{1}{2\gamma}}+\delta_{\lambda}^{-\frac{1}{2}+\frac{1}{2\gamma}}\min_{i\in \bar{S}}\sum_{t=T_{\lambda}+1}^{T_{\lambda+1}}\hat{\ell}_{t,i}+\delta_{\lambda}\sum_{t=T_{\lambda}+1}^{T_{\lambda+1}}\sum_{i\in \wklyds}\hat{\ell}_{t,i}+\frac{1}{T}\sum_{t=T_{\lambda}+1}^{T_{\lambda+1}}\sum_{i\in S}\hat{\ell}_{t,i}\right)\\
        &\le \mathcal{\tilde{O}}\left(\delta_{\lambda}^{-\frac{1}{2}-\frac{1}{2\gamma}}+\delta_{\lambda}^{-\frac{1}{2}+\frac{1}{2\gamma}}\sum_{t=T_{\lambda}+1}^{T_{\lambda+1}}\sum_{i\in \wklyds}\hat{\ell}_{t,i}+\frac{1}{T}\sum_{t=T_{\lambda}+1}^{T_{\lambda+1}}\sum_{i\in S}\hat{\ell}_{t,i}\right)\\
        & \le \mathcal{\tilde{O}}\left(\delta_{\lambda}^{-\frac{1}{2}-\frac{1}{2\gamma}}+\delta_{\lambda}^{-\frac{1}{2}+\frac{1}{2\gamma}}\sum_{i\in \wklyds}\hat{\ell}_{T_{\lambda+1},i}+\frac{1}{T}\sum_{t=T_{\lambda}+1}^{T_{\lambda+1}}\sum_{i\in S}\hat{\ell}_{t,i}\right) \\
        & \le \mathcal{\tilde{O}}' \left(\delta_{\lambda}^{-\frac{1}{2}-\frac{1}{2\gamma}}+\frac{1}{T}\sum_{t=T_{\lambda}+1}^{T_{\lambda+1}}\sum_{i\in S}\hat{\ell}_{t,i}\right).
    \end{align*}
    If $S\ne \emptyset$, then $\mathcal{D}$ contains at least one node in $S$. Let this node be $i_0$. Now we compare $\sum_{t=T_{\lambda}+1}^{T_{\lambda+1}}\hat{\ell}_{t,i^\star}$ with $\sum_{t=T_{\lambda}+1}^{T_{\lambda+1}}\hat{\ell}_{t,i_0}$. If $\sum_{t=T_{\lambda}+1}^{T_{\lambda+1}}\hat{\ell}_{t,i^\star}\le\sum_{t=T_{\lambda}+1}^{T_{\lambda+1}}\hat{\ell}_{t,i_0}$, then we have
    \begin{align*}
    &\sum_{t=T_{\lambda}+1}^{T_{\lambda+1}}\inner{p_t, \hat{\ell}_t}-\sum_{t=T_{\lambda}+1}^{T_{\lambda+1}}\hat{\ell}_{t,i^\star}\\
    &\le \mathcal{\tilde{O}}\left(\delta_{\lambda}^{-\frac{1}{2}-\frac{1}{2\gamma}}+\delta_{\lambda}^{-\frac{1}{2}+\frac{1}{2\gamma}}\sum_{t=T_{\lambda}+1}^{T_{\lambda+1}}\hat{\ell}_{t,i_0}+\delta_{\lambda}\sum_{t=T_{\lambda}+1}^{T_{\lambda+1}}\sum_{i\in \wklyds}\hat{\ell}_{t,i}+\frac{1}{T}\sum_{t=T_{\lambda}+1}^{T_{\lambda+1}}\sum_{i\in S}\hat{\ell}_{t,i}\right)\\
    &\le \mathcal{\tilde{O}}\left(\delta_{\lambda}^{-\frac{1}{2}-\frac{1}{2\gamma}}+\delta_{\lambda}^{-\frac{1}{2}+\frac{1}{2\gamma}}\sum_{t=T_{\lambda}+1}^{T_{\lambda+1}}\sum_{i\in \wklyds}\hat{\ell}_{t,i}+\frac{1}{T}\sum_{t=T_{\lambda}+1}^{T_{\lambda+1}}\sum_{i\in S}\hat{\ell}_{t,i}\right)\\
    & \le \mathcal{\tilde{O}} \left(\delta_{\lambda}^{-\frac{1}{2}-\frac{1}{2\gamma}}+\frac{1}{T}\sum_{t=T_{\lambda}+1}^{T_{\lambda+1}}\sum_{i\in S}\hat{\ell}_{t,i}\right),
    \end{align*}
    where the second inequality is because $\sum_{t=T_{\lambda}+1}^{T_{\lambda+1}}\hat{\ell}_{t,i_0}\le \sum_{t=T_{\lambda}+1}^{T_{\lambda+1}}\sum_{i\in \wklyds}\hat{\ell}_{t,i}$. Otherwise, according to \pref{eq:doubling weakly node S}, we have
    \begin{align*}
    \sum_{t=T_{\lambda}+1}^{T_{\lambda+1}}\inner{p_t, \hat{\ell}_t}-\sum_{t=T_{\lambda}+1}^{T_{\lambda+1}}\hat{\ell}_{t,i^\star} \le \sum_{t=T_{\lambda}+1}^{T_{\lambda+1}}\inner{p_t, \hat{\ell}_t}-\sum_{t=T_{\lambda}+1}^{T_{\lambda+1}}\hat{\ell}_{t,i_0}\le \mathcal{\tilde{O}}\left(\delta_{\lambda}^{1-\frac{1}{\gamma}}+\frac{1}{T}\sum_{t=T_{\lambda}+1}^{T_{\lambda+1}}\sum_{i\in S}\hat{\ell}_{t,i}\right).
    \end{align*}
    By combining all the above cases, we finish the proof of \pref{eq:general doubling trick}. Summing up the regret from $\lambda=1,2,\dots,\lambda^\star$, we have for $i^\star\in S$:
    \begin{align*}
        \sum_{t=1}^T\inner{p_t, \hat{\ell}_t}-\sum_{t=1}^T\hat{\ell}_{t,i^\star}&\le \sum_{\lambda=1}^{\lambda^\star}\mathcal{\tilde{O}}\left(\delta_{\lambda}^{1-\frac{1}{\gamma}}+\frac{1}{T}\sum_{t=T_{\lambda}+1}^{T_{\lambda+1}}\sum_{i\in S}\hat{\ell}_{t,i}\right)\\
        &\le \mathcal{\tilde{O}}\left(\delta_{\lambda^\star}^{1-\frac{1}{\gamma}}+\frac{1}{T}\sum_{t=1}^T\sum_{i\in S}\hat{\ell}_{t,i}\right)\\
        &\le \mathcal{\tilde{O}}\left(\delta_1^{1-\frac{1}{\gamma}}\cdot 2^{\lambda^\star\left(-1+\frac{1}{\gamma}\right)}+\frac{1}{T}\sum_{t=1}^T\sum_{i\in S}\hat{\ell}_{t,i}\right).
    \end{align*}
    and for $i^\star\in \bar{S}$:
    \begin{align*}
        \sum_{t=1}^T\inner{p_t, \hat{\ell}_t}-\sum_{t=1}^T\hat{\ell}_{t,i^\star}&\le \sum_{\lambda=1}^{\lambda^\star}\mathcal{\tilde{O}}\left(\delta_{\lambda}^{-\frac{1}{2}-\frac{1}{2\gamma}}+\frac{1}{T}\sum_{t=T_{\lambda}+1}^{T_{\lambda+1}}\sum_{i\in S}\hat{\ell}_{t,i}\right)\\
        &\le \mathcal{\tilde{O}}\left(\delta_{\lambda^\star}^{-\frac{1}{2}-\frac{1}{2\gamma}}+\frac{1}{T}\sum_{t=1}^T\sum_{i\in S}\hat{\ell}_{t,i}\right)\\
        &\le \mathcal{\tilde{O}}\left(\delta_1^{-\frac{1}{2}-\frac{1}{2\gamma}}\cdot 2^{\lambda^\star\left(\frac{1}{2}+\frac{1}{2\gamma}\right)}+\frac{1}{T}\sum_{t=1}^T\sum_{i\in S}\hat{\ell}_{t,i}\right).
    \end{align*}
    Below we upper bound $\lambda^\star$. If $\lambda^\star\ge 2$, consider the last round of epoch $\lambda^\star-1$. According to the update rule, we have
    \begin{align*}
        \left(\delta_1\cdot 2^{-\lambda^\star+1}\right)^{-\frac{1}{\gamma}} = \delta_{\lambda^\star-1}^{-\frac{1}{\gamma}}\le \sum_{t=T_{\lambda^\star-1}+1}^{T_{\lambda^\star}}\sum_{i\in \wklyds}\hat{\ell}_{t,i}\le \sum_{t=1}^T\sum_{i\in \wklyds}\hat{\ell}_{t,i}.
    \end{align*}
    Plugging the above result to both cases, we have
    \begin{align*}
        \sum_{t=1}^T\inner{p_t-e_{i^\star}, \hat{\ell}_t} \le
        \begin{cases}
        \mathcal{\tilde{O}}\left(\frac{1}{T}\sum_{t=1}^T\sum_{i\in S}\hat{\ell}_{t,i}+\left(\sum_{t=1}^T\sum_{i\in \wklyds}\hat{\ell}_{t,i}\right)^{1-\gamma}\right),&\mbox{if $i^\star\in S$},\\
        \mathcal{\tilde{O}}\left(\frac{1}{T}\sum_{t=1}^T\sum_{i\in S}\hat{\ell}_{t,i}+\left(\sum_{t=1}^T\sum_{i\in \wklyds}\hat{\ell}_{t,i}\right)^{\frac{\left(1+\gamma\right)}{2}}\right),&\mbox{if $i^\star\in \bar{S}$}.
        \end{cases}
    \end{align*}
    Combining the case $\lambda^\star=1$, taking expectation over both sides and using Jensen's inequality, we have:
    \begin{align*}
        \Reg \le
        \begin{cases}
        \mathcal{\tilde{O}}\left(\mathbb{E}\left[\left(\sum_{t=1}^T\sum_{i\in \wklyds}\hat{\ell}_{t,i}\right)^{1-\gamma}\right]\right)\le \mathcal{\tilde{O}}\left(L_{\wklyds}^{1-\gamma}\right),&\mbox{if $i^\star\in S$},\\
        \mathcal{\tilde{O}}\left(\mathbb{E}\left[\left(\sum_{t=1}^T\sum_{i\in \wklyds}\hat{\ell}_{t,i}\right)^{\frac{\left(1+\gamma\right)}{2}}\right]\right)\le \mathcal{\tilde{O}}\left(L_{\wklyds}^{\frac{1+\gamma}{2}}\right),&\mbox{if $i^\star\in \bar{S}$}.
        \end{cases}
    \end{align*}
    This completes the proof.
\end{proof}

\end{document}